\RequirePackage{fix-cm}
\documentclass[smallextended]{svjour3}       
\smartqed  
\usepackage{amsmath,amsfonts,amssymb}
\usepackage{graphicx}
\usepackage{setspace}
\usepackage[subfigure]{tocloft}
\usepackage{amssymb}
\usepackage{mathrsfs}
\usepackage{subfigure}
\usepackage{epstopdf}
\usepackage{tocloft}
\usepackage{hyperref}
\usepackage[figure]{hypcap}
\usepackage{algorithmic}
\usepackage{algorithm}
\usepackage{comment}
\usepackage{caption}
\usepackage{bbding}
\usepackage{rotating}
\usepackage[sort&compress,numbers]{natbib}
\hypersetup{hypertex=true,
            colorlinks=true,
            linkcolor=blue,
            anchorcolor=blue,
            citecolor=blue}

\begin{document}

\title{A Theoretically Guaranteed Quaternion Weighted Schatten $p$-norm Minimization Method for Color Image Restoration 
}

\titlerunning{QWSNM for Color Image Restoration}        

\author{ Qinghua Zhang         \and
        Liangtian He    \and
       Yilun Wang    \and
       Liang-Jian Deng \and
        Jun Liu
}


\institute{Qinghua Zhang \at
              Key Laboratory of Intelligent Computing and Signal Processing of Ministry of Education, School of Mathematical Sciences,  Anhui University,  Hefei 230601, People's Republic of China. \\
              \email{qinghuaz@stu.ahu.edu.cn}           
           \and
           Liangtian He \at
           Key Laboratory of Intelligent Computing and Signal Processing of Ministry of Education, School of Mathematical Sciences, Anhui University,  Hefei 230601, People's Republic of China. \\
           \email{helt@ahu.edu.cn}
           \and
           Yilun Wang \at
           Machinify Inc, Palo Alto, California 94301, United States. \\
           \email{yilun.wang@gmail.com}
           \and
           Liang-Jian Deng \at
           School of Mathematical Sciences, University of Electronic Science and Technology of China, Chengdu 611731,  People's Republic of China. \\
           \email{liangjian.deng@uestc.edu.cn}
           \and
           Jun Liu \at
           Key Laboratory for Applied Statistics of MOE, School of Mathematics and Statistics, Northeast Normal University, Changchun 130024, People's Republic of China. \\
           \email{liuj292@nenu.edu.cn}
}

\date{Received: date / Accepted: date}

\maketitle

\begin{abstract}
Inspired by the fact that the matrix formulated by nonlocal similar patches in a natural image is of low rank, the rank approximation issue have been extensively investigated over the past decades,  among which weighted nuclear norm minimization (WNNM) and weighted Schatten $p$-norm minimization (WSNM) are two prevailing methods have shown great superiority in various image restoration (IR) problems.
Due to the physical characteristic of color images, color image restoration (CIR) is often a much more difficult task than its grayscale image counterpart.
However, when applied to CIR, the traditional WNNM/WSNM method only processes three color channels individually and fails to consider their cross-channel correlations.
Very recently, a quaternion-based WNNM approach (QWNNM) has been developed to mitigate this issue, which is capable of representing the color image as a whole in the quaternion domain and preserving the inherent correlation among the three color channels.
Despite its empirical success, unfortunately, the convergence behavior of QWNNM has not been strictly studied yet.
In this paper, on the one side, we extend the WSNM into quaternion domain and correspondingly propose a novel quaternion-based WSNM model (QWSNM) for tackling the CIR problems. Extensive experiments on two representative CIR tasks, including color image denoising and deblurring,  demonstrate that the proposed QWSNM method performs favorably against many state-of-the-art alternatives, in both quantitative and qualitative evaluations.
On the other side, more importantly, we preliminarily provide a theoretical convergence analysis, that is, by modifying the quaternion alternating direction method of multipliers (QADMM) through a simple continuation strategy, we theoretically prove that both the solution sequences generated by the QWNNM and QWSNM have fixed-point convergence guarantees.
The source code of our algorithm can be downloaded at the website: \url{https://github.com/qiuxuanzhizi/QWSNM}.

\end{abstract}

\section{Introduction}
\label{intro}
Due to the imperfects of imaging systems and inevitable interference of external factors, the acquired images are often degraded in the process of image acquisition, transmission, and storage.
Restoring the clean high-quality image from the corrupted one has been a long-standing research topic for its highly practical value.
It is often acted as a crucial preprocessing step for the success of subsequent various high-level applications, such as detection, classification, recognition,  segmentation, and so on.
Let an $m\times n$ image with $q$ channels be lexicographically stacked by  $\mathbf{x}=[\mathbf{x}^{(1)},\mathbf{x}^{(2)},\cdots,\mathbf{x}^{(q)}]\in \mathbb{R}^{qmn}$, where $\mathbf{x}^{(i)}\in \mathbb{R}^{mn}$ represents the $i$-th channel for $i=1,2,\cdots,q$. Mathematically, the degradation process is generally formulated as:
\begin{equation} \label{1}
\mathbf{y}=\mathbf{Ax}+\mathbf{n},
\end{equation}
where $\mathbf{A}\in \mathbb{R}^{qmn\times qmn}$ is a linear degraded operator\footnote{In this paper, we mainly focus on the color image denoising and deblurring, note that the extension to other color image restoration tasks would be straightforward.}, typically a convolution operator in image deblurring and an identity matrix in image denoising, $\mathbf{x}\in \mathbb{R}^{qmn}$ and $\mathbf{y}\in \mathbb{R}^{qmn}$ are the underlying clean image and corrupted observation, respectively. $\mathbf{n}$ is normally assumed to be the additive white Gaussian noise (AWGN) with variance $\sigma^{2}$.
Since a significant amount of information has been lost during the degradation process, image restoration (IR) is typically an ill-posed problem, making the recovery of clean image from its degraded one rather challenging.
To compensate for the information loss, as a common used strategy, appropriate prior information of the interested image  (also known as regularization) often needs to be introduced to alleviate the ill-posedness by constraining the solution space.
During the last several decades, there have been enduring efforts to the prior exploration, such as total variation (TV) regularized methods \cite{rudin1992nonlinear,wang2008new,bredies2010total,chang2014domain}, sparse representation methods \cite{mairal2007sparse,zhang2014group,zha2021nonconvex}, nonlocal self-similarity (NSS) prior methods \cite{buades2005non,danielyan2011bm3d}, low-rank regularized methods \cite{peng2014reweighted,dong2012nonlocal,lai2018manifold}, denoiser-guided methods \cite{zhang2021plug,he2023sln}, etc.

As an important branch, low rank matrix approximation (LRMA) has been widely applied in numerous IR tasks, which intents on recovering the underlying low rank matrix from its degraded observation. Generally, the LRMA can be divided into two categories: low rank matrix factorization (LRMF) methods \cite{he2015total,chi2019nonconvex,shi2019low,ke2023quasi} and low rank minimization (LRM) methods \cite{candes2011robust,jin2016alternating,yuan2023rank}. LRMF assumes the latent matrix is a product of two low-rank matrices. However, a significant shortcoming limiting its practical applicability is that the rank of underlying matrix often can not be easily predetermined. By contrast, LRM usually solves a minimization problem with different rank approximation regularizers, such as the widely used nuclear norm. Cand\`{e}s  et al. \cite{candes2012exact} has proven that nuclear norm is the tightest convex relaxation of the NP-hard rank function. Despite a good theoretical guarantee by the singular value thresholding (SVT) operation, many studies have demonstrated that the nuclear norm minimization (NNM) tends to yield sub-optimal solutions of the original rank minimization. The main reason behind this phenomenon is that NNM treats each singular value equally, which conflicts with the fact that the larger singular values often contain more information and should be penalized slighter.
In order to enhance the stability and effectiveness of the nuclear norm regularization,
many nonconvex surrogates have been introduced instead. Among them, one of the most representative work is the nonlocal patch-based weighted NNM model (WNNM) proposed by \cite{gu2014weighted} for grayscale image denoising, which considers the physical significance of the singular values and assigns different weights to them.
In subsequent research, Gu et al. \cite{gu2017weighted} further extended WNNM to other low-level vision tasks, including background subtraction,  color image denoising and inpainting. It is worth emphasizing that though WNNM often delivers promising results, it still tends to over-shrink the dominant rank components, which constraints its capacity of recovering the images with rich textures and details.
To overcome this shortcoming, based on the Schatten $p$-norm minimization (SNM) \cite{nie2012low}, Xie et al. \cite{xie2016weighted} proposed a more effective weighted Schatten $p$-norm minimization (WSNM) model, which avoids shrinking too much the low rank components and is capable of obtaining more accurate recovery results compared to the WNNM counterpart. Similarly, Zha et al. \cite{zha2018non} proposed a nonconvex weighted $\ell_{p}$ nuclear norm minimization (NCW-NNM) model under the ADMM framework for IR problems.

Given the rapidly growing role of color images in our daily life, the color image restoration (CIR) has received increasingly attention in this research field. It should be pointed out that most of the existing IR algorithms are originally developed for grayscale images (i.e., single channel). For more complicated CIR problems, the easiest and rudest way is to implement these IR approaches on three color channels (i.e., red, green, and blue in the RGB color space) individually. Unfortunately, this straightforward extension through a channel-wise manner commonly yields unsatisfactory results since they totally ignores the inter-channel correlation between RGB components.
Alternatively, two other practicable strategies have been widely adopted to CIR \cite{luisier2008sure}.
The first strategy is to convert the color images from standard RGB space into a less correlated color space, such as YCbCr space, and then restore each color channel in the transformed space independently. For example, the color block-matching and 3D filtering (CBM3D) \cite{dabov2007image} first converts the sRGB image into a luminance-chrominance space and then applies BM3D to each channel separately. The second strategy is to concatenate the three RGB channels to make use of the channel correlation. For instance, Xu et al. \cite{xu2017multi} proposed a multi-channel WNNM (MC-WNNM) model for real color image denoising, where a weight matrix is introduced to exploit the correlated information and noise difference among channels.  In the same spirit of MC-WNNM \cite{xu2017multi}, Huang et al. \cite{huang2021multichannel} also employed the multi-channel strategy and proposed a multi-channel WSNM (MC-WSNM) model for real color image denoising.
However, from the perspective of tensor unfolding, the concatenation CIR methods just utilize one unfolding matrix and ignore the other two unfolding matrices.
According to the analyses above, a key challenge for CIR lies in preserving the channel correlations, and the above-mentioned CIR strategies fail to make full use of the correlation information among three color channels, thus the recovered results are often unsatisfactory in terms of reducing the color distortions and color artifacts.

As a new color image representation tool, quaternion encodes the color image pixel into three imaginary parts, which perfectly fits the color image structure and is capable of fully preserving the inter-relationship between the color channels. In very recent years, quaternion-based approaches have been attracting much attention attributing to its rapid development in both theory and applications.
Typically, Chen et al. \cite{chen2019low} extended LRMA into the quaternion domain and proposed the low rank quaternion approximation (LRQA) for color image denoising and inpainting.
Liu et al. \cite{liu2018infrared} combined the quaternion-based TV and sparse dictionary learning for color image super-resolution. Miao et al. \cite{miao2021color} developed an efficient low-rank quaternion matrix completion algorithm to recover missing data of a color image and both the LRMF and NNM techniques were combined in their quaternion matrix-based model.
Zou et al. \cite{zou2016quaternion} proposed a quaternion collaborative and sparse representation model with application to color face recognition.
Yu et al. \cite{yu2019quaternion} extended the nonlocal patch-based WNNM into quaternion domain and proposed the quaternion-based WNNM method (QWNNM) for color image denoising.
Thereafter, Huang et al. \cite{huang2022quaternion} proposed a nonlocal patch-based quaternion WNNM model\footnote{It is denoted as QWNNM* here to avoid confusion with the QWNNM for color image denoising application in \cite{yu2019quaternion}.}  (QWNNM*) with application to color image deblurring. It is worth mentioning that they are the first to leverage the quaternion representation to address color image deblurring problems. The authors extended the 2-dimensional blurring matrix to quaternion domain and correspondingly formulated the quaternion blurring operator.
For a comprehensive understanding of the quaternion-based color image processing methods, the interested readers are referred to \cite{huang2023review} for more details.

In this paper, on the one hand, as a significant extension, we merge the nonlocal patch-based WSNM into quaternion domain and propose a novel quaternion-based WSNM (QWSNM) model, which is capable of taking advantages of both the WSNM regularizer and quaternion representation for tacking CIR problems. Thanks to the quaternion calculation, the inner-relationship among the color channels can be well preserved. Extensive experiments demonstrate that our proposed QWSNM is superior to many state-of-the-art alternatives, in both the objective evaluation and perceptual observation.
To the best of our knowledge, the most similar to our work is the patch-based QWNNM* model proposed by Huang \cite{huang2022quaternion} et al., but they did not strictly discuss the convergence behavior of their algorithmic solver.
On the other hand, we introduce a unified algorithmic framework to the resulting nonconvex QWNNM* and QWSNM optimization problems. To be specific, by modifying the quaternion ADMM algorithm (QADMM) \cite{flamant2021general} with a continuation strategy, we theoretically prove that the variable sequences generated by QWNNM* and QWSNM converge to their corresponding stational points.
To facilitate the understanding to readers,  Table 1 summarizes the connections and differences between our work and several closely related nonlocal patch-based low rank methods.
\begin{table}[htbp]
\caption{ Comparison of the key attributes between our proposed method and several closely related nonlocal patch-based low rank approaches.}
\centering {
\resizebox{\textwidth}{12mm}{
\begin{tabular}{|c|c|c|c|c|}

\hline 

 Method & Regularizer term &  Task & Quaternion representation & Convergence analysis         \\
 \hline
WNNM \cite{gu2014weighted}    &   weighted nuclear norm    & grayscale image denoising    &   \XSolidBrush       &    \XSolidBrush      \\
  \hline
WNNM \cite{gu2017weighted}    &   weighted nuclear norm    & color image denoising/inpainting, background substraction    &   \XSolidBrush       &    \Checkmark      \\
  \hline
WSNM \cite{xie2016weighted}    &   weighted Schatten $p$-norm    & grayscale image denoising, background substraction     &   \XSolidBrush       &    \Checkmark      \\
  \hline
NCW-NNM \cite{zha2018non}   &   weighted Schatten $p$-norm   & color  image deblurring/inpainting/compressive sensing     &    \XSolidBrush       &    \XSolidBrush      \\
  \hline
QWNNM  \cite{yu2019quaternion} &    weighted nuclear norm   &  color image denoising     &  \Checkmark         &    \XSolidBrush      \\
  \hline
QWNNM* \cite{huang2022quaternion}  &  weighted nuclear norm     &  color image deblurring    &   \Checkmark      &    \XSolidBrush     \\
  \hline
QWSNM (Proposed) &    weighted Schatten $p$-norm   &  color image denoising/deblurring    &   \Checkmark       &   \Checkmark      \\
 \hline
\end{tabular}} }  \\

\label{tab: different priors}
\end{table}

The organization of this paper is as follows. Section 2 briefly reviews several necessary preliminaries, including the basic concepts of quaternion algebra, the QADMM algorithm, nonlocal patch-based QWNNM and WSNM recovery models.
Section 3 serves as the main part of this work, where a novel QWSNM model is introduced and the adapted QADMM with a continuation strategy is developed to solve the resulting nonconvex quaternion minimization problems. Moreover, the convergence analysis of our proposed algorithm is carried out by calculating the residual iteratively for each step of the modified QADMM. In Section 4, we report and discuss the recovery results on color image denoising and color image deblurring by the given method and compare it with several state-of-the-art CIR approaches. Some concluding remarks will be drawn in the Section 5.

\section{Preliminaries}
In this section, we first briefly review some mathematical notations and definitions of quaternion representation, quaternion algebra and quaternion constrained convex optimization. Then we introduce the WSNM regularization  and  QWNNM model for CIR. Throughout this paper, scalars, vectors and matrices are denoted as lowercase letters, boldface lowercase letters and boldface capital letters, respectively, e.g., $a$, $\mathbf{a}$, $\mathbf{A}$. Following \cite{yu2019quaternion}, the variables with a dot above (e.g., $\dot{a}$, $\dot{\mathbf{a}}$ and $\dot{\mathbf{A}}$) are used to denote the corresponding quaternion variables in the quaternion domain.

\subsection{Quaternion Representation and Algebra}
As a generalization of the real space $\mathbb{R}$ and complex space $\mathbb{C}$, the quaternion space is defined as
\begin{equation}\label{2}
\begin{aligned}
\mathbb{H}=\{a_{0}+ a_{1}\mathbf{i}+a_{2}\mathbf{j}+a_{3}\mathbf{k} | a_{0},a_{1},a_{2},a_{3}\in \mathbb{R}\},
\end{aligned}
\end{equation}
where $\{1,\mathbf{i},\mathbf{j},\mathbf{k}\}$ is the basis of $\mathbb{H}$, and the imaginary units $\mathbf{i},\mathbf{j},\mathbf{k}$ obey the quaternion rules $\mathbf{i}^{2}=\mathbf{j}^{2}=\mathbf{k}^{2}=\mathbf{ijk}=-1$, which implies $\mathbf{ij}=\mathbf{k}=-\mathbf{ji}$, $\mathbf{jk}=\mathbf{i}=-\mathbf{kj}$, $\mathbf{ki}=\mathbf{j}=-\mathbf{ik}$.

Let $\dot{a}=a_{0}+a_{1}\mathbf{i}+a_{2}\mathbf{j}+a_{3}\mathbf{k}\in\mathbb{H}$, $\dot{b}=b_{0}+b_{1}\mathbf{i}+b_{2}\mathbf{j}+b_{3}\mathbf{k}\in\mathbb{H}$, and $\lambda\in\mathbb{R}$, then we have
\begin{equation}\label{3}
\begin{aligned}
\dot{a}+\dot{b}&=(a_{0}+b_{0})+(a_{1}+b_{1})\mathbf{i}+(a_{2}+b_{2})\mathbf{j}+(a_{3}+b_{3})\mathbf{k},\\
\lambda\dot{a}&=(\lambda a_{0})+(\lambda a_{1})\mathbf{i}+(\lambda a_{2})\mathbf{j}+(\lambda a_{3})\mathbf{k},
\end{aligned}
\end{equation}
and
\begin{equation}\label{4}
\begin{aligned}
\dot{a}\dot{b}&=(a_{0}b_{0}-a_{1}b_{1}-a_{2}b_{2}-a_{3}b_{3})+(a_{0}b_{1}+a_{1}b_{0}+a_{2}b_{3}-a_{3}b_{2})\mathbf{i}\\
&+(a_{0}b_{2}-a_{1}b_{3}+a_{2}b_{0}+a_{3}b_{1})\mathbf{j}+(a_{0}b_{3}+a_{1}b_{2}-a_{2}b_{1}+a_{3}b_{0})\mathbf{k}.
\end{aligned}
\end{equation}
The conjugate and modulus of $\dot{a}$ are defined by
\begin{equation}\label{5}
\begin{aligned}
\dot{a}^{\ast}&=a_{0}-a_{1}\mathbf{i}-a_{2}\mathbf{j}-a_{3}\mathbf{k},\\
|\dot{a}|&=\sqrt{a_{0}^{2}+a_{1}^{2}+a_{2}^{2}+a_{3}^{2}}.
\end{aligned}
\end{equation}
The transformation of the quaternion $\dot{a}$ with pure unit quaternions $\mathbf{i}, \mathbf{j}, \mathbf{k}$ is defined as
\begin{equation}\label{6}
\begin{aligned}
\dot{a}^{\mathbf{i}}&=-\mathbf{i}\dot{a}\mathbf{i}=a_{0}+a_{1}\mathbf{i}-a_{2}\mathbf{j}-a_{3}\mathbf{k},\\
\dot{a}^{\mathbf{j}}&=-\mathbf{j}\dot{a}\mathbf{j}=a_{0}-a_{1}\mathbf{i}+a_{2}\mathbf{j}-a_{3}\mathbf{k},\\
\dot{a}^{\mathbf{k}}&=-\mathbf{k}\dot{a}\mathbf{k}=a_{0}-a_{1}\mathbf{i}-a_{2}\mathbf{j}+a_{3}\mathbf{k}.
\end{aligned}
\end{equation}
For quaternion matrix $\mathbf{\dot{X}}=(\dot{x}_{ij})\in \mathbb{H}^{m\times n}$, where $\mathbf{\dot{X}}=\mathbf{X}_{0}+\mathbf{X}_{1}\mathbf{i}+\mathbf{X}_{2}\mathbf{j}+\mathbf{X}_{3}\mathbf{k}$, and $\mathbf{X}_{l}\in \mathbb{R}^{m\times n}(l=0,1,2,3)$ are real matrices.
When $\mathbf{X}_{0}=\mathbf{0}$, $\mathbf{\dot{X}}$ reduces to be a pure quaternion matrix. Since color image has three channels $(R,G,B)$, we can represent a color image with a pure quaternion matrix, and the RGB channels of a color image pixel $\dot{x}_{ij}$ can be encoded as three imaginary parts of the quaternion, i.e.,
\begin{equation}\label{7}
\begin{aligned}
\dot{x}_{ij}=\dot{x}_{ij}^{r}\mathbf{i}+\dot{x}_{ij}^{g}\mathbf{j}+\dot{x}_{ij}^{b}\mathbf{k},
\end{aligned}
\end{equation}
where $i=1,\ldots,m,j=1,\ldots,n$, and $\dot{x}_{ij}^{r}, \dot{x}_{ij}^{g}, \dot{x}_{ij}^{b} \in \mathbb{R}$ are the red, green, blue components at position $(i,j)$ in the color image, respectively.

The quaternion identity matrix $\dot{\mathbf{I}}$ is similar to real-valued identity matrix. For example, if $\dot{\mathbf{I}}_2 \in \mathbb{H}^{2\times 2}$, it has the form
\begin{equation}\label{8}
\begin{aligned}
\dot{\mathbf{I}}_2=
    \begin{pmatrix}
       1+0\mathbf{i}+0\mathbf{j}+0\mathbf{k} & \quad 0+0\mathbf{i}+0\mathbf{j}+0\mathbf{k} \\
       0+0\mathbf{i}+0\mathbf{j}+0\mathbf{k} & \quad 1+0\mathbf{i}+0\mathbf{j}+0\mathbf{k}
    \end{pmatrix}
\end{aligned},
\end{equation}
The conjugate operator $\mathbf{\dot{X}}^{\ast}$, the transpose operator $\mathbf{\dot{X}}^{T}$ and the conjugate transpose operator $\mathbf{\dot{X}}^{\vartriangleleft}$ are defined as: $\mathbf{\dot{X}}^{\ast}=(\dot{x}_{ij}^{\ast}), \mathbf{\dot{X}}^{T}=(\dot{x}_{ji})$ and $\mathbf{\dot{X}}^{\vartriangleleft}=(\dot{x}_{ji}^{\ast})$.

The norms of quaternion vectors and matrices are defined as follows \cite{jia2021structure}.
\begin{definition}
The $\ell_{2}$-norm of quaternion vector $\mathbf{\dot{a}}=\alpha_{0}+\alpha_{1}\mathbf{i}+\alpha_{2}\mathbf{j}+\alpha_{3}\mathbf{k}\in \mathbb{H}^{n}$ is $\|\mathbf{\dot{a}}\|_{2}:=\sqrt{\sum_{i}|\alpha_{i}|^{2}}$; the $\ell_{2}$-norm of quaternion matrix $\mathbf{\dot{X}}=(\dot{x}_{ij})_{m\times n}$ is $\|\mathbf{\dot{X}}\|_{2} :=max(\sigma(\mathbf{\dot{X}}))$, where $\sigma(\mathbf{\dot{X}})$ is the set of singular values of $\mathbf{\dot{X}}$; the Frobenius norm of the quaternion matrix $\mathbf{\dot{X}}$ is $\|\mathbf{\dot{X}}\|_{F}:=\sqrt{\sum_{i,j}|\dot{x}_{ij}|^{2}}.$
\end{definition}

\begin{definition}[Unitary quaternion matrix]
$\mathbf{\dot{X}}\in \mathbb{H}^{m\times m}$ is called unitary quaternion matrix if and only if $\mathbf{\dot{X}}\mathbf{\dot{X}}^{\vartriangleleft}=\mathbf{\dot{X}}^{\vartriangleleft}\mathbf{\dot{X}}=\mathbf{\dot{I}}_{m}$, where $\mathbf{\dot{I}}_{m}$ is the quaternion identity matrix.
\end{definition}

The following theorem of quaternion singular value decomposition (QSVD) has been proved by Zhang \cite{zhang1997quaternions}. Similar to the SVD for real-valued matrices, all the singular values of quaternion matrices are nonnegative and also have the decreasing order property, and the larger singular values contain more color image information.
\begin{theopargself}
\begin{theorem}[(QSVD)]
 Given a quaternion matrix $\mathbf{\dot{Q}}\in \mathbb{H}^{m\times n}$ of rank $r$, there exist two unitary quaternion matrices $\mathbf{\dot{U}}\in \mathbb{H}^{m\times m}$ and $\mathbf{\dot{V}}\in \mathbb{H}^{n\times n}$ such that $\mathbf{\dot{Q}}=\mathbf{\dot{U}}
\begin{pmatrix}
\mathbf{\sum}_{r}&0\\
0&0\\
\end{pmatrix}
\mathbf{\dot{V}^{\ast}}$, where $\sum_{r}=diag(\sigma_{1},\cdots,\sigma_{r})\in \mathbb{R}^{r\times r}$
, and all singular values $\sigma_{i}>0, i=1,\cdots,r.$
\end{theorem}
\end{theopargself}

\begin{definition}[Quaternion rank]
The rank of a quaternion matrix is defined as the number of its nonzero singular values.
\end{definition}

\subsection{QADMM Algorithm}
The ADMM is an effective and flexible tool to address many convex and non-convex optimization problems. During the past decades, ADMM has been widely studied and applied in a wealth of applications \cite{mei2018cauchy,he2021support,hou2022truncated,tao2023study,wu2023framelet,wu2023lrtcfpan}. The basic idea of ADMM is to decompose the original problem into a set of sub-problems and solve them alternately.
The ADMM in quaternion domain (QADMM) was first proposed by Flamant et al. \cite{flamant2021general}. To the following quaternion constraint convex optimization problem
\begin{equation}\label{9}
\begin{aligned}
&\min_{\mathbf{\dot{p}}, \mathbf{\dot{q}}}f(\mathbf{\dot{p}})+g(\mathbf{\dot{q}})\\
&s.t. \quad \mathbf{\dot{A}}_{1}\mathbf{\dot{p}}+\mathbf{\dot{A}}_{2}\mathbf{\dot{p}}^{\mathbf{i}}+\mathbf{\dot{A}}_{3}\mathbf{\dot{p}}^{\mathbf{j}}+\mathbf{\dot{A}}_{4}\mathbf{\dot{p}}^{\mathbf{k}}\\
&+\mathbf{\dot{B}}_{1}\mathbf{\dot{q}}+\mathbf{\dot{B}}_{2}\mathbf{\dot{q}}^{\mathbf{i}}+\mathbf{\dot{B}}_{3}\mathbf{\dot{q}}^{\mathbf{j}}+\mathbf{\dot{B}_{4}}\mathbf{\dot{q}}^{\mathbf{k}}=\mathbf{\dot{c}},
\end{aligned}
\end{equation}
where $f$ and $g$ are two convex functions of quaternion variables $\mathbf{\dot{p}}\in \mathbb{H}^{n}$ and $\mathbf{\dot{q}}\in \mathbb{H}^{m}$, respectively. The two variables are linked through a widely affine constraint, defined by $\mathbf{\dot{A}}_{i}\in \mathbb{H}^{p \times n}, \mathbf{\dot{B}}_{i} \in \mathbb{H}^{p\times m}$ for $i=1,2,3,4$ and $\mathbf{\dot{c}}\in \mathbb{H}^{p}$.

For the sake of notation brevity, let us define the quaternion residual $r(\mathbf{\dot{p},\dot{q}})$ as
\begin{equation}\label{10}
\begin{aligned}
r(\mathbf{\dot{p},\dot{q}})&=\mathbf{\dot{A}}_{1}\mathbf{\dot{p}}+\mathbf{\dot{A}}_{2}\mathbf{\dot{p}}^{\mathbf{i}}+\mathbf{\dot{A}}_{3}\mathbf{\dot{p}}^{\mathbf{j}}+\mathbf{\dot{A}}_{4}\mathbf{\dot{p}}^{\mathbf{k}}\\
&+\mathbf{\dot{B}}_{1}\mathbf{\dot{q}}+\mathbf{\dot{B}}_{2}\mathbf{\dot{q}}^{\mathbf{i}}+\mathbf{\dot{B}}_{3}\mathbf{\dot{q}}^{\mathbf{j}}+\mathbf{\dot{B}_{4}}\mathbf{\dot{q}}^{\mathbf{k}}-\mathbf{\dot{c}}.
\end{aligned}
\end{equation}
Thereby, the augmented Lagrangian function of \eqref{9} has the form
\begin{equation}\label{11}
\begin{aligned}
\mathcal{L}_{\rho}(\mathbf{\dot{p}},\mathbf{\dot{q}},\mathbf{\dot{u}})=f(\mathbf{\dot{p}})+g(\mathbf{\dot{q}})+Re(\dot{\mathbf{u}}^{\ast}r(\mathbf{\dot{p}},\mathbf{\dot{q}}))+\frac{\rho}{2}\|r(\mathbf{\dot{p},\dot{q}})\|_{2}^{2},
\end{aligned}
\end{equation}
where $\mathbf{\dot{\mathbf{u}}}$ denotes the Lagrange multiplier, $\rho$ is the penalty parameter and $Re(\mathbf{\dot{x}})$ represents the real part of $\mathbf{\dot{x}}$.
Similar to the real-valued ADMM \cite{boyd2011distributed}, the QADMM algorithm consists of the iterations as follows.
\begin{equation}\label{12}
\left\{\begin{array}{lll}
\mathbf{\dot{p}}^{(k+1)}&=\mathop{\mathrm{arg}\min_{\mathbf{\dot{p}}}}\{f(\mathbf{\dot{p}})+\frac{\rho}{2}\|r(\mathbf{\dot{p},\dot{q}}^{(k)})+\mathbf{\dot{u}}^{(k)}\|_{2}^{2}\},\\
\mathbf{\dot{q}}^{(k+1)}&=\mathop{\mathrm{arg}\min_{\mathbf{\dot{q}}}}\{g(\mathbf{\dot{q}})+\frac{\rho}{2}\|r(\mathbf{\dot{p}}^{(k+1)},\mathbf{\dot{q}})+\mathbf{\dot{u}}^{(k)}\|_{2}^{2}\},\\
\mathbf{\dot{u}}^{(k+1)}&=\mathbf{\dot{u}}^{(k)}+r(\mathbf{\dot{p}}^{(k+1)},\mathbf{\dot{q}}^{(k+1)}).
\end{array}\right.
\end{equation}

\begin{theopargself}
    \begin{theorem}[\cite{flamant2021general}]
Under the assumptions that $f:\mathbb{H}^{n}\rightarrow\mathbb{R}$ and $g:\mathbb{H}^{n}\rightarrow\mathbb{R}$ are closed, proper and convex  functions, and there exists at least one $(\widetilde{\mathbf{\dot{p}}},\widetilde{\mathbf{\dot{q}}},\widetilde{\mathbf{\dot{u}}})$ such that $\mathcal{L}_{0}(\widetilde{\mathbf{\dot{p}}},\widetilde{\mathbf{\dot{q}}},\mathbf{\dot{u}})\leq \mathcal{L}_{0}(\widetilde{\mathbf{\dot{p}}},\widetilde{\mathbf{\dot{q}}},\widetilde{\mathbf{\dot{u}}})\leq \mathcal{L}_{0}(\mathbf{\dot{p}},\mathbf{\dot{q}},\widetilde{\mathbf{\dot{u}}})$ for all $\mathbf{\dot{p}},\mathbf{\dot{q}},\mathbf{\dot{u}}$, the QADMM algorithm satisfy:\\
$\bullet$ Convergence of the quaternion residual $r(\mathbf{\dot{p}}^{(k)},\mathbf{\dot{q}}^{(k)})\rightarrow 0$ as $k\rightarrow \infty$;\\
$\bullet$ Objective convergence: $f(\mathbf{\dot{p}}^{(k)})+g(\mathbf{\dot{q}}^{(k)})\rightarrow \tilde{v}$, where $\tilde{v}$ is the optimal value of \eqref{9};\\
$\bullet$ Dual variable convergence: $\mathbf{\dot{u}}^{(k)}\rightarrow\widetilde{\mathbf{\dot{u}}}$ as $k\rightarrow \infty$.
    \end{theorem}
\end{theopargself}

\subsection{QWNNM for Color Image Denoising/Deblurring}
Although WNNM \cite{gu2017weighted} has achieved excellent performance for grayscale image denoising, it is likely to generate color distortions and color artifacts in the scenarios of color image denoising, since it ignores the inter-relationship among the RGB channels. By leveraging the effective quaternion representation for color images, Yu et al. \cite{yu2019quaternion} merged the quaternion representation into WNNM and proposed an effective QWNNM model for color image denoising. Thereafter, Huang et al. \cite{huang2022quaternion} extended the QWNNM to color image deblurring.
\begin{equation}\label{13}
\begin{aligned}
\min_{\mathbf{\dot{X}}}\frac{\lambda}{2}\|\mathbf{\dot{A}}\mathbf{\dot{X}}-\mathbf{\dot{Y}}\|_{F}^{2}+\|\mathbf{\dot{X}}\|_{\mathbf{w},*},
\end{aligned}
\end{equation}
where $\mathbf{\dot{X}}\in\mathbb{H}^{m\times n}$ and $\mathbf{\dot{Y}}\in\mathbb{H}^{m\times n}$ are the encoded counterparts of clean image $\mathbf{x}$ and degraded image $\mathbf{y}$  in Eq. \eqref{1} with quaternion representation, respectively.  $\mathbf{\dot{A}} \in \mathbb{H}^{m\times m}$ is the convolution matrix with quaternion representation. To be specific,
\begin{equation}\label{80}
    \begin{aligned}
        \mathbf{\dot{A}}=
\begin{pmatrix}
\dot{a}_{11}&\dot{a}_{12}&\ldots&\dot{a}_{1m}\\
\dot{a}_{21}&\dot{a}_{22}&\ldots&\dot{a}_{2m}\\
\vdots      & \vdots     & \ddots&\vdots     \\
\dot{a}_{m1}&\dot{a}_{m2}&\ldots&\dot{a}_{mm}\\
\end{pmatrix},
    \end{aligned}
\end{equation}
where $\dot{a}_{ij}=\dot{a}_{ij}^{r}\mathbf{i}+\dot{a}_{ij}^{g}\mathbf{j}+\dot{a}_{ij}^{b}\mathbf{k}$ (and $\dot{a}_{ij}^{r}, \dot{a}_{ij}^{g}, \dot{a}_{ij}^{b} \in \mathbb{R}$) denotes a pixel component of 2-dimensional quaternion matrix $\mathbf{\dot{A}}\in \mathbb{H}^{m\times m}$($i=1,\cdots,m;j=1,\cdots,m$). It is easy to see that when $\mathbf{\dot{A}}$ is a quaternion identity matrix, it reduces to the color image denoising.
$\|\mathbf{\cdot}\|_{\mathbf{w},*}$ denotes the QWNNM regularizer, which is defined as
\begin{equation}\label{14}
\begin{aligned}
\|\mathbf{\dot{X}}\|_{\mathbf{w},*}=\sum_{i}|w_{i}\sigma_{i}(\mathbf{\dot{X}})|,
\end{aligned}
\end{equation}
where $\sigma_{i}(\mathbf{\dot{X}})$ is the $i$-th singular value of $\mathbf{\dot{X}}$ and the weight $w_{i}$ is set as:
\begin{equation}\label{15}
\begin{aligned}
w_{i}=\frac{c}{\sigma_{i}(\mathbf{\dot{X}})+\epsilon},
\end{aligned}
\end{equation}
where $c$ is a compromising constant, and $\epsilon$ is a small positive number to prevent denominator from being zero. By introducing an auxiliary quaternion variable $\mathbf{\dot{Z}}$ with the constraint $\mathbf{\dot{X}}=\mathbf{\dot{Z}}$ and Lagrange multiplier $\mathbf{\dot{\eta}}\in \mathbb{H}^{m\times n}$, the augmented Lagrangian function of \eqref{13} is formulated by
\begin{equation}\label{16}
\begin{aligned}
\mathcal{L}(\mathbf{\dot{X}},\mathbf{\dot{Z}},\mathbf{\dot{\eta}}) = \frac{\lambda}{2}\|\mathbf{\dot{A}}\mathbf{\dot{X}}-\mathbf{\dot{Y}}\|_{F}^{2}+\|\mathbf{\dot{Z}}\|_{\mathbf{w},*}+\frac{\beta}{2}\|\mathbf{\dot{X}}-\mathbf{\dot{Z}}\|_{F}^{2}+\left\langle{\mathbf{\dot{\eta}},\mathbf{\dot{X}-\dot{Z}}}\right\rangle,
\end{aligned}
\end{equation}
where $\beta>0$ is the penalty parameter. The QADMM algorithm can be applied to solving \eqref{16} and the iterative scheme is as follows.
\begin{equation}\label{17}
\left\{\begin{array}{llll}
\mathbf{\dot{X}}^{(k+1)}=\mathop{\mathrm{arg}\min_{\mathbf{\dot{X}}}}\mathcal{L}(\mathbf{\dot{X}},\mathbf{\dot{Z}}^{(k)},\mathbf{\dot{\eta}}^{(k)}),\\
\mathbf{\dot{Z}}^{(k+1)}=\mathop{\mathrm{arg}\min_{\mathbf{\dot{Z}}}}\mathcal{L}(\mathbf{\dot{X}}^{(k+1)},\mathbf{\dot{Z}},\mathbf{\dot{\eta}}^{(k)}),\\
\mathbf{\dot{\eta}}^{(k+1)}=\mathbf{\dot{\eta}}^{(k)}+(\mathbf{\dot{X}}^{(k+1)}-\mathbf{\dot{Z}}^{(k+1)}).\\
\end{array}\right.
\end{equation}
Concretely, for the $\mathbf{\dot{X}}$-subproblem of \eqref{17}, according to the first-order optimality condition, it is easy to obtain that
\begin{equation}\label{170}
\begin{aligned}
\left(\lambda \mathbf{\dot{A}}^{\ast}\mathbf{\dot{A}}  + \beta\mathbf{\dot{I}}\right)\mathbf{\dot{X}}^{(k+1)} = \lambda \mathbf{\dot{A}}^{\ast}\mathbf{\dot{Y}} + \beta \mathbf{\dot{Z}}^{k} - \mathbf{\dot{\eta}}^{(k)}.
\end{aligned}
\end{equation}
Under the quaternion periodic boundary condition for $\mathbf{\dot{X}}$,
the $\mathbf{\dot{X}}$-subproblem can be efficiently tackled by the quaternion fast Fourier transform \cite{sangwine1996fourier}.
\begin{equation}\label{18}
\begin{aligned}
\mathbf{\dot{X}}^{(k+1)}=\mathcal{F}^{-1}\bigg(\frac{\lambda\mathcal{F}(\mathbf{\dot{A}})^{\ast}\circ\mathcal{F}(\mathbf{\dot{Y}})+\beta\mathcal{F}(\mathbf{\dot{Z}}^{k})-\mathcal{F}(\mathbf{\dot{\eta}}^{(k)})}{\lambda\mathcal{F}(\mathbf{\dot{A}})^{\ast}\circ\mathcal{F}(\mathbf{\dot{A}})+\beta\mathbf{\dot{I}}}\bigg),
\end{aligned}
\end{equation}
where $\mathcal{F}$ denotes two-dimensional discrete quaternion Fourier transform, $\mathcal{F}^{-1}$ denotes two-dimensional discrete inverse quaternion Fourier transform, ``$\circ $'' and ``$- $'' represent the component-wise multiplication and division.

By excluding the constant, the $\mathbf{\dot{Z}}$-subproblem can be reformulated as:
\begin{equation}\label{19}
\begin{aligned}
\mathbf{\dot{Z}}^{(k+1)}=\mathop{\mathrm{arg}\min_{\mathbf{\dot{Z}}}}\|\mathbf{\dot{Z}}\|_{\mathbf{w},*}+\frac{\beta}{2}\|\mathbf{\dot{Z}}-\big(\mathbf{\dot{X}}^{(k+1)}+\frac{\mathbf{\dot{\eta}}^{(k)}}{\beta}\big)\|_{F}^{2}.
\end{aligned}
\end{equation}
Suppose $\mathbf{\dot{U}}^{(k)}\Sigma^{(k)}\mathbf{\dot{V}}^{(k)^{*}}$ be the QSVD of $\mathbf{\dot{X}}^{(k+1)}+\frac{\mathbf{\dot{\eta}}^{(k)}}{\beta}$, where $\Sigma^{(k)}=diag\{\sigma_{1}^{(k)}$,\\
$\sigma_{2}^{(k)},\cdots,\sigma_{r}^{(k)}\}$, $r=min\{m,n\}$, then the $\mathbf{\dot{Z}}$-subproblem has a closed-form solution as follows.
\begin{equation}\label{20}
\begin{aligned}
\mathbf{\dot{Z}}^{(k+1)}=\mathbf{\dot{U}}^{(k)}\Delta^{(k)}\mathbf{\dot{V}}^{(k)^{*}},
\end{aligned}
\end{equation}
where $\Delta^{(k)}=diag\{\delta_{1}^{(k)},\delta_{2}^{(k)},\cdots,\delta_{r}^{(k)}\}$, and $\delta_{i}^{(k)}=max(\sigma_{i}^{(k)}-\frac{w_{i}^{(k)}}{\beta},0)$.

\hspace{-6mm} \textit{Remark:} Although the QWNNM* \cite{huang2022quaternion} has shown very promising performance for color image deblurring, its convergence behavior has not been strictly studied yet. In this paper, as one of the key contributions, we supplement the convergence analysis of QWNNM* (see the Section 3.4 later) by modifying the original QADMM algorithm through a simple continuation strategy, which is formulated by
\begin{equation}\label{21}
\left\{\begin{array}{llll}
\mathbf{\dot{X}}^{(k+1)}=\mathop{\mathrm{arg}\min_{\mathbf{\dot{X}}}}\mathcal{L}(\mathbf{\dot{X}},\mathbf{\dot{Z}}^{(k)},\mathbf{\dot{\eta}}^{(k)}),\\
\mathbf{\dot{Z}}^{(k+1)}=\mathop{\mathrm{arg}\min_{\mathbf{\dot{Z}}}}\mathcal{L}(\mathbf{\dot{X}}^{(k+1)},\mathbf{\dot{Z}},\mathbf{\dot{\eta}}^{(k)}),\\
\mathbf{\dot{\eta}}^{(k+1)}=\mathbf{\dot{\eta}}^{(k)}+\beta^{(k)}(\mathbf{\dot{X}}^{(k+1)}-\mathbf{\dot{Z}}^{(k+1)}),\\
\beta^{(k+1)} = \mu\beta^{(k)}, \quad  (\mu>1).
\end{array}\right.
\end{equation}
It is worth mentioning that in the original QWNNM* algorithm, the penalty parameter $\beta^{(k)}$ is fixed as a constant $\beta$ in \eqref{17}, while it is updated monotonously by adopting a continuation strategy $\beta^{(k+1)} = \mu\beta^{(k)}$, $\mu>1$ in our adapted version \eqref{21}.
This simple modification is vital to guarantee the fixed-point convergence of QWNNM* algorithm.

\subsection{WSNM Regularization}
In order to enhance the stability and effectiveness of WNNM, Xie et al. \cite{xie2016weighted} introduced a more accurate rank surrogate and firstly proposed the WSNM model for grayscale image denoising.
\begin{equation}\label{22}
\begin{aligned}
\min_{\mathbf{X}}\frac{\lambda}{2}\|\mathbf{X-Y}\|_{F}^{2}+\|\mathbf{X}\|_{\mathbf{w},S_{p}}^{p},
\end{aligned}
\end{equation}
where $\lambda >0$ is the regularization parameter, $\mathbf{Y}$ is the observed image, $\mathbf{X}$ is the latent image to estimate, and $\|\mathbf{X}\|_{\mathbf{w},S_{p}}$ denotes the weighted Schatten $p$-norm of $\mathbf{X}$, which is defined as
\begin{equation}\label{23}
\begin{aligned}
\|\mathbf{X}\|_{\mathbf{w},S_{p}}=\left(\sum_{i=1}^{min\{n,m\}}w_{i}\sigma_{i}(\mathbf{X})^{p}\right)^{\frac{1}{p}},
\end{aligned}
\end{equation}
where $\sigma_{i}(\mathbf{X})$ is the $i$-th singular value of $\mathbf{X}$ and the weight $w_{i}$ is assigned as the same as the formula \eqref{15}. Thereby, the weighted Schatten $p$-norm of $\mathbf{X}$ with power $p$ has the form
\begin{equation}\label{24}
\begin{aligned}
\|\mathbf{X}\|_{\mathbf{w},S_{p}}^{p}=\sum_{i=1}^{min\{n,m\}}w_{i}\sigma_{i}(\mathbf{X})^{p}.
\end{aligned}
\end{equation}
It is obvious to see that WNNM \cite{gu2014weighted} is a special case of WSNM when power $p=1$.  WSNM generalizes WNNM and has more flexibility than WNNM. Moreover, it has been verified experimentally that WSNM outperforms WNNM under different noise levels \cite{xie2016weighted}.
Due to the fact that the WSNM is superior to the WNNM for a higher accuracy in signal recovery while theoretically requiring only a weaker restricted isometry property \cite{liu2014exact}, the WSNM regularization has been widely employed in various image processing applications \cite{zha2018non,feng2018image,xie2016hyperspectral}.

However, for the complicated CIR tasks, it should be mentioned that all these WSNM methods address the color images in a monochromatic way. More precisely, these WSNM methods treat the RGB channels as three independent ``grayscale" images and ignore the important correlations between them.
In the next section, we will represent the color image as a pure quaternion matrix and extend the WSNM to the quaternion domain, which is capable of fully preserving the cross-channel correlations of color images and yielding much better CIR recovery results.

\section{The Proposed Model and Algorithmic Solver}
In this section, we introduce a novel quaternion-based WSNM model (QWSNM) for CIR, after which the corresponding numerical algorithm is also presented. Moreover, the computational complexity of the proposed algorithm is discussed. Finally,  the convergence analysis of our new algorithmic solver is given to the QWNNM* \cite{huang2022quaternion} and the proposed QWSNM models.

\subsection{QWSNM Recovery Model}
By combining the benefits of quaternion representation and WSNM regularization, we put forward the QWSNM recovery model for CIR as follows.
\begin{equation}\label{28}
\begin{aligned}
\min_{\mathbf{\dot{X}}}\frac{\lambda}{2}\|\mathbf{\dot{A}}\mathbf{\dot{X}}-\mathbf{\dot{Y}}\|_{F}^{2}+\|\mathbf{\dot{X}}\|_{\mathbf{w},S_{p}}^{p},
\end{aligned}
\end{equation}
where $\mathbf{\dot{A}} \in \mathbb{H}^{m\times m}$ is the 2-dimensional quaternion degradation matrix, $\mathbf{\dot{X}}\in\mathbb{H}^{m\times n}$ is the original image with quaternion representation, $\mathbf{\dot{Y}}\in\mathbb{H}^{m\times n}$ is the degraded image with quaternion representation, and $\|\mathbf{\cdot}\|_{\mathbf{w},S_{p}}^{p}$ is the weighted Schatten $p$-norm with power $p$ defined the same as the formula \eqref{24} presents. Particularly, in the case of color image denoising, our QWSNM model \eqref{28} reduces to:
\begin{equation}\label{29}
    \begin{aligned}
        \min_{\mathbf{\dot{X}}}\frac{\lambda}{2}\|\mathbf{\dot{X}}-\mathbf{\dot{Y}}\|_{F}^{2}+\|\mathbf{\dot{X}}\|_{\mathbf{w},S_{p}}^{p}.
    \end{aligned}
\end{equation}
Unlike the QWNNM for color image denoising, it should be mentioned that though the optimization problem \eqref{29} no longer has a closed-form solution,  it can be well solved by the generalized soft-thresholding (GST) algorithm \cite{zuo2013generalized}, which will be discussed in detail in the next subsection.

\subsection{Algorithmic Solver}
To make our formulated model \eqref{28} tractable and robust, we also resort to the QADMM iterative scheme to tackle the optimization problem \eqref{28}.
By introducing an auxiliary quaternion variable $\mathbf{\dot{Z}}$ with the constraint $\mathbf{\dot{X}}=\mathbf{\dot{Z}}$, and Lagrange multiplier $\mathbf{\dot{\eta}}\in \mathbb{H}^{m\times n}$, the augmented Lagrangian function of \eqref{28} is formulated by
\begin{equation}\label{30}
\begin{aligned}
\mathcal{L}(\mathbf{\dot{X}},\mathbf{\dot{Z}},\mathbf{\dot{\eta}})=\frac{\lambda}{2}\|\mathbf{\dot{A}}\mathbf{\dot{X}}-\mathbf{\dot{Y}}\|_{F}^{2}+\|\mathbf{\dot{Z}}\|_{\mathbf{w},S_{p}}^{p}+\frac{\beta}{2}\|\mathbf{\dot{X}}-\mathbf{\dot{Z}}\|_{F}^{2}+\left\langle{\mathbf{\dot{\eta}},\mathbf{\dot{X}-\dot{Z}}}\right\rangle,
\end{aligned}
\end{equation}
where $\beta>0$ is the penalty parameter.
Different from the standard QADMM, a continuation strategy for the penalty parameter is adopted for solving \eqref{30}:
\begin{equation}\label{31}
\left\{\begin{array}{llll}
\mathbf{\dot{X}}^{(k+1)}=\mathop{\mathrm{arg}\min_{\mathbf{\dot{X}}}}\mathcal{L}(\mathbf{\dot{X}},\mathbf{\dot{Z}}^{(k)},\mathbf{\dot{\eta}}^{(k)}),\\
\mathbf{\dot{Z}}^{(k+1)}=\mathop{\mathrm{arg}\min_{\mathbf{\dot{Z}}}}\mathcal{L}(\mathbf{\dot{X}}^{(k+1)},\mathbf{\dot{Z}},\mathbf{\dot{\eta}}^{(k)}),\\
\mathbf{\dot{\eta}}^{(k+1)}=\mathbf{\dot{\eta}}^{(k)}+\beta^{(k)}(\mathbf{\dot{X}}^{(k+1)}-\mathbf{\dot{Z}}^{(k+1)}),\\
\beta^{(k+1)} = \mu\beta^{(k)}, \quad (\mu > 1).
\end{array}\right.
\end{equation}
It can be seen that the minimization of Eq. \eqref{30} involves two separable minimization subproblems, $i.e.$, $\mathbf{\dot{X}}$ and $\mathbf{\dot{Z}}$ subproblems. Next, we will show that each subproblem can be efficiently addressed. To avoid confusion, the subscribe $k$ will be omitted in the following for conciseness.

The $\mathbf{\dot{X}}$-subproblem is written as:
\begin{equation}\label{32}
\begin{aligned}
\mathbf{\dot{X}}=\mathop{\mathrm{arg}\min_{\mathbf{\dot{X}}}}\frac{\lambda}{2}\|\mathbf{\dot{A}}\mathbf{\dot{X}}-\mathbf{\dot{Y}}\|_{F}^{2}+\frac{\beta}{2}\|\mathbf{\dot{X}}-\mathbf{\dot{Z}}\|_{F}^{2}+\left\langle{\mathbf{\dot{\eta}},\mathbf{\dot{X}-\dot{Z}}}\right\rangle.
\end{aligned}
\end{equation}
According to the theory of quaternion matrix derivatives \cite{xu2015theory},
Eq. \eqref{32} obtains a closed-form solution by derivating the objective function and its solution is expressed as
\begin{equation}\label{33}
\mathbf{\dot{X}}=(\lambda\mathbf{\dot{A}}^{\ast}\mathbf{\dot{A}}+\beta\mathbf{\dot{I}})^{-1}(\lambda\mathbf{\dot{A}}^{\ast}\mathbf{\dot{Y}}+\beta\mathbf{\dot{Z}}-\mathbf{\dot{\eta}}).
\end{equation}
For color image deblurring, under the quaternion periodic boundary of $\mathbf{\dot{X}}$, we can utilize the quaternion fast Fourier transform \cite{sangwine1996fourier} to efficiently solve
the Eq. \eqref{33}.

According to Eq. \eqref{31}, the $\mathbf{\dot{Z}}$-subproblem can be equivalently rewritten as:
\begin{equation}\label{35}
\begin{aligned}
\mathbf{\dot{Z}}=\mathop{\mathrm{arg}\min_{\mathbf{\dot{Z}}}}\|\mathbf{\dot{Z}}\|_{\mathbf{w},S_{p}}^{p}+\frac{\beta}{2}\Big\|\mathbf{\dot{Z}}-\Big(\mathbf{\dot{X}}+\frac{\mathbf{\dot{\eta}}}{\beta}\Big)\Big\|_{F}^{2}.
\end{aligned}
\end{equation}
The $\mathbf{\dot{Z}}$-subproblem coincides to a color image denoising task of $\mathbf{\dot{X}}+\frac{\dot{\eta}}{\beta}$. To utilize the NSS prior in the quaternion space, we first divide the noisy image $\mathbf{\dot{X}}+\frac{\mathbf{\dot{\eta}}}{\beta}$ into overlapped patches of size $w\times w$ and assign a number of key patches with a fixed interval as reference. For each key patch, we find its $M$ most similar patches, including itself, in a search window around it (the size of search window is $W\times W$). Then we vectorize the $M$ patches as quaternion column vectors and stack them together to form a quaternion matrix, which is denoted as $\mathbf{\dot{P}}_{j}\in \mathbb{H}^{w^{2}\times M}$. As such, we will get the following optimization problem.
\begin{equation}\label{36}
\begin{aligned}
\mathbf{\dot{Z}}_{j}=\mathop{\mathrm{arg}\min_{\mathbf{\dot{Z}}_{j}}}\|\mathbf{\dot{Z}}_{j}\|_{\mathbf{w},S_{p}}^{p}+\frac{\beta}{2}\|\mathbf{\dot{Z}}_{j}-\mathbf{\dot{P}}_{j}\|_{F}^{2},
\end{aligned}
\end{equation}
where $\mathbf{\dot{Z}}_{j}\in \mathbb{H}^{w^{2}\times M}$ is the nonlocal patch-based quaternion matrix of the clean color image. In this way,  we can employ the LRMA methods to estimate the underlying clean patch from its noisy observation.
We give the following lemma and theorem to solve the problem \eqref{36}.

\begin{theopargself}
    \begin{lemma}[(Von Neumanns trace inequality \cite{mirsky1975trace})]
         For any two quaternion matrices $\mathbf{\dot{X}},\mathbf{\dot{Y}}\in \mathbb{H}^{w^{2}\times M}$, $tr(\mathbf{\dot{X}}^{T}\mathbf{\dot{Y}})\leq tr(\sigma(\mathbf{\dot{X}})^{T}\sigma(\mathbf{\dot{Y}}))$, where $\sigma(\mathbf{\dot{X}})$ and $\sigma(\mathbf{\dot{Y}})$ are the ordered singular value matrices of $\mathbf{\dot{X}}$ and $\mathbf{\dot{Y}}$ with the same order, respectively.
    \end{lemma}
\end{theopargself}

\begin{theopargself}
    \begin{theorem}
        Let $\mathbf{\dot{U}}_{j}\Sigma_{j}\mathbf{\dot{V}}_{j}^{\ast}$ be the QSVD of $\mathbf{\dot{P}}_{j}$ with $\Sigma_{j} =diag(\sigma_{1j},\cdots,\sigma_{rj}),r=min(w^{2}, M)$. Suppose that all the singular values are in non-ascending order, then the optimal of \eqref{36} will be $\mathbf{\dot{Z}}_{j}=\mathbf{\dot{U}}_{j}\Delta_{j}\mathbf{\dot{V}}_{j}^{\ast}$ with $\Delta_{j}=diag(\delta_{1j},\cdots,\delta_{rj})$, where $\delta_{ij}$ is given by solving the problems below:
\begin{equation}\label{37}
\left\{\begin{array}{ll}
\underset{\delta_{1j},\cdots,\delta_{rj}}{min}\sum\limits_{i=1}^{r}[(\delta_{ij}-\sigma_{ij})^{2}+w_{ij}\delta_{ij}^{p}],\\
s.t. \ \  \delta_{ij}\geq0,\;and\;\delta_{ij}\geq \delta_{kj},\;for \   i\leq k.\\
\end{array}\right.
\end{equation}
    \end{theorem}
\end{theopargself}

\begin{theopargself}
\begin{proof}
Let the optimal solution of \eqref{36} have the QSVD $\mathbf{\dot{Z}}_{j}=\mathbf{\dot{Q}}_{j}\Delta_{j}\mathbf{\dot{R}}_{j}^{\ast}$, and the QSVD of matrix $\mathbf{\dot{P}}_{j}=\mathbf{\dot{U}}_{j}\Sigma_{j}\mathbf{\dot{V}}_{j}^{\ast}$, where both $\Delta_{j}$ and $\Sigma_{j}$ are diagonal matrices with the same order (no-ascending). According to Lemma 1, we have
\begin{equation}\label{38}
\begin{aligned}
\|\mathbf{\dot{Z}}_{j}-\mathbf{\dot{P}}_{j}\|_{F}^{2}&=tr(\Delta_{j}^{T}\Delta_{j})+tr(\Sigma_{j}^{T}\Sigma_{j})-2tr(\mathbf{\dot{Z}}_{j}^{T}\mathbf{\dot{P}}_{j})\\
&\leq tr(\Delta_{j}^{T}\Delta_{j})+tr(\Sigma_{j}^{T}\Sigma_{j})-2tr(\Delta_{j}^{T}\Sigma_{j})\\
&=\|\Delta_{j}-\Sigma_{j}\|_{F}^{2},\\
\end{aligned}
\end{equation}
where the equality holds only when $\mathbf{\dot{Q}}_{j}=\mathbf{\dot{U}}_{j}$ and $\mathbf{\dot{R}}_{j}=\mathbf{\dot{V}}_{j}$, and the optimal solution of \eqref{36} is obtained by solving \eqref{37}.     $\blacksquare$
\end{proof}
\end{theopargself}

Therefore, based on Theorem 2, Eq. \eqref{36} can be transformed into solving Eq. \eqref{37}. To obtain the solution of Eq. \eqref{37} efficiently,  the generalized soft-thresholding (GST) algorithm \cite{zuo2013generalized} (see \textbf{Algorithm 1} later for details) can be adopted. In particular, given $p$, $\sigma_{ij}$ and $w_{ij}$, there exists a specific threshold:
\begin{equation}\label{39}
\begin{aligned}
\tau_{p}^{GST}(w_{ij})=(2w_{ij}(1-p))^{\frac{1}{2-p}}+w_{ij}p(2w_{ij}(1-p))^{\frac{p-1}{2-p}}.
\end{aligned}
\end{equation}
Here if $\sigma_{ij}<\tau_{p}^{GST}(w_{ij}),\delta_{ij}=0$; otherwise, for any $\sigma_{ij}\in(\tau_{p}^{GST}(w_{ij}),+\infty)$, Eq. \eqref{37} has one unique minimum $S_{p}^{GST}(\sigma_{ij};w_{ij})$, which can be solved by the following equation.
\begin{equation}\label{40}
\begin{aligned}
S_{p}^{GST}(\sigma_{ij};w_{ij})-\sigma_{ij}+w_{ij}p(S_{p}^{GST}(\sigma_{ij};w_{ij}))^{p-1}=0.
\end{aligned}
\end{equation}
As a consequence, we can get the closed-form solution of Eq. \eqref{37} by the GST algorithm. Based on the denoised patch $\mathbf{\dot{Z}}_{j}$, we compute the average of repeated estimates for each patch and aggregate all patches together, reconstructing the final clean image $\mathbf{\dot{Z}}$. For the update of the Lagrange penalty parameter $\beta$, the continuation strategy is also utilized the same as the adapted version of QWNNM* \eqref{23} presents, which is vital to guarantee the convergence of our algorithm.
Concluding the above analyses, the whole optimization scheme is summarized in \textbf{Algorithm 2}.

\subsection{Computational Complexity}
In this subsection, we discuss the computational complexity of the proposed QWSNM algorithm. For color image deblurring, the main computational cost of the subproblem $\mathbf{\dot{X}}$ concentrates on computing quaternion fast Fourier transform, and the complexity is $\mathcal{O}(mnlog(mn))$. The time for computing the subproblem $\mathbf{\dot{Z}}$ is mainly consumed by QSVD, and the complexity is $\mathcal{O}(w^{4}M)$. The computational cost of GST algorithm in subproblem $\mathbf{\dot{Z}}$ is $\mathcal{O}(JM)$, where $J$ is the inner iteration number of GST algorithm. The computational cost of the Lagrange multiplier $\mathbf{\dot{\eta}}$ is $\mathcal{O}(mn)$. Suppose the outer iteration number is $K$ and the number of the patch group in the subproblem $\mathbf{\dot{Z}}$ is $N$, thus, the total computational cost of the overall optimization algorithm is $\mathcal{O}(K(mnlog(mn)+N(w^{4}M+JM)+mn))$. Analogously, the main computation cost of QWSNM for color image denoising is $\mathcal{O}(KNM(w^{4}+J))$.

\begin{algorithm}[htbp]
\caption{Generalized soft-thresholding (GST) algorithm}
\begin{algorithmic}[1]
\renewcommand{\algorithmicrequire}{\textbf{Input:}}
\renewcommand{\algorithmicensure}{\textbf{Output:}}
\REQUIRE $w_{ij}$, $\sigma_{ij}$, $p$, $J;$
\ENSURE $S_{p}^{GST}(\sigma_{ij};w_{ij});$
\FOR{$i=1,2,\cdots,r$}
\STATE $\tau_{p}^{GST}(w_{ij})=(2w_{ij}(1-p))^{\frac{1}{2-p}}+w_{ij}p(2w_{ij}(1-p))^{\frac{p-1}{2-p}};$\\
\IF{$|\sigma_{ij}|\leq\tau_{p}^{GST}(w_{ij})$}
\STATE $S_{p}^{GST}(\sigma_{ij};w_{ij})=0;$
\ELSE
\STATE $s=0, \delta_{ij}^{(s)}=|\sigma_{ij}|;$
\FOR{$s=0, 1, \cdots, J-1$}
\STATE $\delta_{ij}^{(s+1)}=|\sigma_{ij}|-w_{ij}p(\delta_{ij}^{(s)})^{p-1};$
\ENDFOR
\STATE $S_{p}^{GST}(\sigma_{ij};w_{ij})=sgn(\sigma_{ij})\delta_{ij};$
\ENDIF
\ENDFOR
\RETURN $S_{p}^{GST}(\sigma_{ij};w_{ij}).$
\end{algorithmic}
\end{algorithm}

\begin{algorithm}[htbp]
\caption{The proposed QWSNM algorithm}
\renewcommand{\algorithmicrequire}{\textbf{Input:}}
\renewcommand{\algorithmicensure}{\textbf{Output:}}
\begin{algorithmic}[1]
\REQUIRE Initialize $\dot{\mathbf{X}}^{(0)}=\mathbf{\dot{Y}}$, $\mathbf{\dot{Z}}^{(0)}=\mathbf{\dot{X}}^{(0)}$, $\mathbf{\dot{\eta}}^{(0)}=0$; Set parameters $\lambda, \beta, K$;
\ENSURE {The recovered image $\mathbf{\dot{X}}^{(K)}$;}
\FOR {$k=0,1,\cdots,K-1$}
\STATE Calculate $\mathbf{\dot{X}}^{(k+1)}$ by solving Eq. (31);
\STATE Calculate $\mathbf{\dot{Z}}^{(k+1)}$ by solving Eq. (33);
\STATE Update $\mathbf{\dot{\eta}}^{(k+1)}=\mathbf{\dot{\eta}}^{(k)}+\beta^{(k)}(\mathbf{\dot{X}}^{(k+1)}-\mathbf{\dot{Z}}^{(k+1)})$;
\STATE Update $\beta^{(k+1)}=\mu\beta^{(k)}$.   \quad     $\triangleright$: (\textit{Continuation strategy})\\
\ENDFOR\\
\RETURN {The recovered image $\dot{\mathbf{X}}^{(K)}.$}
\end{algorithmic}
\end{algorithm}

\subsection{Convergence Study}
In this subsection, we will analyze the convergence of previous QWNNM* (with a continuation strategy) \cite{huang2022quaternion}  and the proposed Algorithm 2. Due to the non-convexity of QWNNM* model and the proposed QWSNM model,  making the convergence analysis become much more challenging, one also should expect a weaker form of convergence.
In the following, we provide a fixed-point convergence guarantee of the developed iterative schemes.
\begin{theorem}
Assume that the sequence of parameter $\{\beta^{(k)}\}$ is unbounded, then the sequences $\{\mathbf{\dot{X}}^{(k)}\}$ and $\{\mathbf{\dot{Z}}^{(k)}\} $ generated by QWNNM* \cite{huang2022quaternion} with a continuation strategy must satisfy the following:
\begin{equation}\label{41}
    \begin{aligned}
        \lim_{k\rightarrow +\infty}\|\mathbf{\dot{X}}^{(k+1)}-\mathbf{\dot{Z}}^{(k+1)}\|_{F}=0,\\
        \lim_{k\rightarrow +\infty}\|\mathbf{\dot{X}}^{(k+1)}-\mathbf{\dot{X}}^{(k)}\|_{F}=0,\\
        \lim_{k\rightarrow +\infty}\|\mathbf{\dot{Z}}^{(k+1)}-\mathbf{\dot{Z}}^{(k)}\|_{F}=0.\\
    \end{aligned}
\end{equation}
\end{theorem}

\begin{proof}
 We first prove that the sequence $\{\mathbf{\dot{\eta}}^{(k)}\}$ generated by QWNNM* with a continuation strategy is upper bounded. Suppose $\mathbf{\dot{P}}^{(k)}=\frac{1}{\beta^{(k)}}\mathbf{\dot{\eta}}^{(k)}+\mathbf{\dot{X}}^{(k+1)}$, then we can get $N$ patch-based quaternion matrices $\mathbf{\dot{P}}^{(k)}_{j}$ and $N$ patch-based quaternion matrices $\mathbf{\dot{Z}}^{(k+1)}_{j}(j=1,2,\cdots,N)$ by utilizing the NSS prior. Denote the QSVD of patch-based quaternion matrix $\mathbf{\dot{P}}^{(k)}_{j}$ in the $k+1$ iteration as $\mathbf{\dot{U}}^{(k)}_{j}\Sigma^{(k)}_{j}\mathbf{\dot{V}}^{(k)^{\ast}}_{j}$, where $\Sigma^{(k)}_{j}=diag(\sigma^{(k)}_{1j},\sigma^{(k)}_{2j},\cdots,\sigma^{(k)}_{rj})$ is the diagonal singular value matrix. According the formula \eqref{20}, we have $\mathbf{\dot{Z}}^{(k+1)}_{j}=\mathbf{\dot{U}}^{(k)}_{j}\Delta^{(k)}_{j}\mathbf{\dot{V}}^{(k)^{\ast}}_{j}$, where $\Delta^{(k)}_{j}=diag(\delta^{(k)}_{1j},\delta^{(k)}_{2j},\cdots,\delta^{(k)}_{rj})$ is the singular value matrix. Then
\begin{equation}\label{42}
\begin{aligned}
 \|\mathbf{\dot{\eta}}^{(k+1)}\|_{F}^{2}&=\|\mathbf{\dot{\eta}}^{(k)}+\beta^{(k)}(\mathbf{\dot{X}}^{(k+1)}-\mathbf{\dot{Z}}^{(k+1)})\|_{F}^{2}\\
&\leq({\beta^{(k)}})^{2}\sum_{j=1}^{N}\|\mathbf{\dot{U}}^{(k)}_{j}\Sigma^{(k)}_{j}\mathbf{\dot{V}}^{(k)^{\ast}}_{j}-\mathbf{\dot{U}}^{(k)}_{j}\Delta^{(k)}_{j}\mathbf{\dot{V}}^{(k)^{\ast}}_{j}\|_{F}^{2}\\
&=({\beta^{(k)}})^{2}\sum_{j=1}^{N}\|\Sigma^{(k)}_{j}-\Delta^{(k)}_{j}\|_{F}^{2}\leq\sum_{j=1}^{N}\sum_{i}(w_{ij}^{(k)})^{2}.\\
\end{aligned}
\end{equation}
Hence, the sequence $\left\{\mathbf{\dot{\eta}}^{(k)}\right\}$ generated by QWNNM* is upper bounded.

We then prove that the sequence of the Lagrange function $\{\mathcal{L}(\dot{\mathbf{X}}^{(k+1)},\mathbf{\dot{Z}}^{(k+1)},$\\
$\mathbf{\dot{\eta}}^{(k+1)},\beta^{(k+1)})\}$ is also upper bounded. Due to the updating rule of $\mathbf{\dot{\eta}}$, we can get
\begin{equation} \label{43}
    \begin{aligned}
        \mathcal{L}(\mathbf{\dot{X}}^{(k+1)},\mathbf{\dot{Z}}^{(k+1)},\mathbf{\dot{\eta}}^{(k+1)},\beta^{(k+1)})\\
        &\!\!\!\!\!\!\!\!\!\!\!\!\!\!\! \!\!\!\!\!\!\!\!\!\!\!\!\!\!\!\!\!\!\!\!\!\!\!\!\!\!\!\!\!\!\!\!\!\!\!\!\!\!\!\!\!\!\!\!\!\!\!\!\!\!\!\!\!\!\!\!\!\!\!\!\!\!\!\! \! \! \! \!= \|\mathbf{\dot{A}}\mathbf{\dot{X}}^{(k+1)}-\mathbf{\dot{Y}}\|_{F}^{2}+\|\mathbf{\dot{Z}}^{(k+1)}\|_{\mathbf{w},*}\hfill\\
        &\!\!\!\!\!\!\!\!\!\!\!\!\!\!\! \!\!\!\!\!\!\!\!\!\!\!\!\!\!\!\!\!\!\!\!\!\!\!\!\!\!\!\!\!\!\!\!\!\!\!\!\!\!\!\!\!\!\!\!\!\!\!\!\!\!\!\!\!\!\!\!\!\!\!\!\!\!\!\! \! \! \! \!+\left\langle \mathbf{\dot{\eta}}^{(k+1)},\mathbf{\dot{X}}^{(k+1)}-\mathbf{\dot{Z}}^{(k+1)} \right\rangle+\frac{\beta^{(k+1)}}{2}\|\mathbf{\dot{X}}^{(k+1)}-\mathbf{\dot{Z}}^{(k+1)}\|_{F}^{2}\\
        &\!\!\!\!\!\!\!\!\!\!\!\!\!\!\! \!\!\!\!\!\!\!\!\!\!\!\!\!\!\!\!\!\!\!\!\!\!\!\!\!\!\!\!\!\!\!\!\!\!\!\!\!\!\!\!\!\!\!\!\!\!\!\!\!\!\!\!\!\!\!\!\!\!\!\!\!\!\!\! \! \! \! \!=\mathcal{L}(\mathbf{\dot{X}}^{(k+1)},\mathbf{\dot{Z}}^{(k+1)},\mathbf{\dot{\eta}}^{(k)},\beta^{(k)})\\
&\!\!\!\!\!\!\!\!\!\!\!\!\!\!\! \!\!\!\!\!\!\!\!\!\!\!\!\!\!\!\!\!\!\!\!\!\!\!\!\!\!\!\!\!\!\!\!\!\!\!\!\!\!\!\!\!\!\!\!\!\!\!\!\!\!\!\!\!\!\!\!\!\!\!\!\!\!\!\! \! \! \! \!+\left\langle\mathbf{\dot{\eta}}^{(k+1)}-\mathbf{\dot{\eta}}^{(k)},\mathbf{\dot{X}}^{(k+1)}-{\mathbf{\dot{Z}}}^{(k+1)} \right\rangle+\frac{\beta^{(k+1)}-\beta^{(k)}}{2}\|\mathbf{\dot{X}}^{(k+1)}-\mathbf{\dot{Z}}^{(k+1)}\|_{F}^{2}\\
        &\!\!\!\!\!\!\!\!\!\!\!\!\!\!\! \!\!\!\!\!\!\!\!\!\!\!\!\!\!\!\!\!\!\!\!\!\!\!\!\!\!\!\!\!\!\!\!\!\!\!\!\!\!\!\!\!\!\!\!\!\!\!\!\!\!\!\!\!\!\!\!\!\!\!\!\!\!\!\! \! \! \! \!=\mathcal{L}(\mathbf{\dot{X}}^{(k+1)},\mathbf{\dot{Z}}^{(k+1)},\mathbf{\dot{\eta}}^{(k)},\beta^{(k)})\\
&\!\!\!\!\!\!\!\!\!\!\!\!\!\!\! \!\!\!\!\!\!\!\!\!\!\!\!\!\!\!\!\!\!\!\!\!\!\!\!\!\!\!\!\!\!\!\!\!\!\!\!\!\!\!\!\!\!\!\!\!\!\!\!\!\!\!\!\!\!\!\!\!\!\!\!\!\!\!\! \! \! \! \!+\left\langle\mathbf{\dot{\eta}}^{(k+1)}-\mathbf{\dot{\eta}}^{(k)},\frac{\mathbf{\dot{\eta}}^{(k+1)}-\mathbf{\dot{\eta}}^{(k)}}{\beta^{(k)}} \right\rangle+\frac{\beta^{(k+1)}-\beta^{(k)}}{2}\Big\|\frac{\mathbf{\dot{\eta}}^{(k+1)}-\mathbf{\dot{\eta}}^{(k)}}{\beta^{(k)}}\Big\|_{F}^{2}\\
        &\!\!\!\!\!\!\!\!\!\!\!\!\!\!\! \!\!\!\!\!\!\!\!\!\!\!\!\!\!\!\!\!\!\!\!\!\!\!\!\!\!\!\!\!\!\!\!\!\!\!\!\!\!\!\!\!\!\!\!\!\!\!\!\!\!\!\!\!\!\!\!\!\!\!\!\!\!\!\! \! \! \! \!=\mathcal{L}(\mathbf{\dot{X}}^{(k+1)},\mathbf{\dot{Z}}^{(k+1)},{\mathbf{\dot{\eta}}^{(k)}},\beta^{(k)})+\frac{\beta^{(k+1)}+\beta^{(k)}}{2(\beta^{(k)})^{2}}\|\mathbf{\dot{\eta}}^{(k+1)}-\mathbf{\dot{\eta}}^{(k)}\|_{F}^{2}.
    \end{aligned}
\end{equation}
Since $\{\mathbf{\dot{\eta}}^{(k)}\}$ is upper bounded, the sequence $\{\mathbf{\dot{\eta}}^{(k+1)}-\mathbf{\dot{\eta}}^{(k)}\}$ is also upper bounded. Suppose $M_{0}$ is the upper bound of $\{\mathbf{\dot{\eta}}^{(k+1)}-\mathbf{\dot{\eta}}^{(k)}\}$ for all $k\geq 0$, i.e., $\|\mathbf{\dot{\eta}}^{(k+1)}-\mathbf{\dot{\eta}}^{(k)}\|_{F}^{2}\leq M_{0}^{2}, \forall \ k \geq 0$. Besides, the inequality $\mathcal{L}(\mathbf{\dot{X}}^{(k+1)},\mathbf{\dot{Z}}^{(k+1)},\mathbf{\dot{\eta}}^{(k)},$\\
$\beta^{(k)})\leq \mathcal{L}(\mathbf{\dot{X}}^{(k)},\mathbf{\dot{Z}}^{(k)},\mathbf{\dot{\eta}}^{(k)},\beta^{(k)})$ always holds since we have the globally optimal solution of $\mathbf{\dot{X}}$ and $\mathbf{\dot{Z}}$ in their corresponding sub-problems. Therefore, we have
\begin{equation} \label{44}
    \begin{aligned}
        \mathcal{L}(\mathbf{\dot{X}}^{(k+1)},\mathbf{\dot{Z}}^{(k+1)},\mathbf{\dot{\eta}}^{(k+1)},\beta^{(k+1)})&\leq\mathcal{L}(\mathbf{\dot{X}}^{(k+1)},\mathbf{\dot{Z}}^{(k+1)},\mathbf{\dot{\eta}}^{(k)},\beta^{(k)})+\frac{\beta^{(k+1)}+\beta^{(k)}}{2(\beta^{(k)})^{2}}M_{0}^{2}\\
        &\leq \mathcal{L}(\mathbf{\dot{X}}^{(1)},\mathbf{\dot{Z}}^{(1)},\mathbf{\dot{\eta}}^{(0)},\beta^{(0)})+M_{0}^{2}\sum_{k=0}^{\infty}\frac{\beta^{(k+1)}+\beta^{(k)}}{2(\beta^{(k)})^{2}}\\
        &=\mathcal{L}(\mathbf{\dot{X}}^{(1)},\mathbf{\dot{Z}}^{(1)},\mathbf{\dot{\eta}}^{(0)},\beta^{(0)})+M_{0}^{2}\sum_{k=0}^{\infty}\frac{1+\mu}{2\beta^{(0)}\mu^{k}}\\
        &\leq\mathcal{L}(\mathbf{\dot{X}}^{(1)},\mathbf{\dot{Z}}^{(1)},\mathbf{\dot{\eta}}^{(0)},\beta^{(0)})+\frac{M_{0}^{2}}{\beta^{(0)}}\sum_{k=0}^{\infty}\frac{1}{\mu^{k-1}}\\
        &< +\infty.
    \end{aligned}
\end{equation}
Therefore, $\{\mathcal{L}(\mathbf{\dot{X}}^{(k+1)},\mathbf{\dot{Z}}^{(k+1)},\mathbf{\dot{\eta}}^{(k+1)},\beta^{(k+1)})\}$ is upper bounded.

We next prove the sequences of $\{\dot{\mathbf{X}}^{(k)}\}$ and $\{\mathbf{\dot{Z}}^{(k)}\}$ are upper bounded. According to the formula \eqref{16}, we have
\begin{equation}\label{45}
\begin{aligned}
    \|\mathbf{\dot{A}}\mathbf{\dot{X}}^{(k+1)}-\mathbf{\dot{Y}}\|_{F}^{2}+\|\mathbf{\dot{Z}}^{(k+1)}\|_{\mathbf{w},*}\\
&\!\!\!\!\!\!\!\!\!\!\!\!\!\!\!\!\!\!\! \!\!\!\!\!\!\!\!\!\!\!\!\!\!\!\!\!\!\!\!\!\!\!\!\!\!\!\!\!\!\!\!\!\!\!\!\!\!\!\!\!\!\!\!\!\!\!\!\!\!\!\!\!\!\!\!\!\!\!\!\!\!\!\! \! \! \! \!=\mathcal{L}(\mathbf{\dot{X}}^{(k)},\mathbf{\dot{Z}}^{(k)},\mathbf{\dot{\eta}}^{(k-1)},\beta^{(k-1)})-\left\langle\mathbf{\dot{\eta}}^{(k)},\mathbf{\dot{X}}^{(k)}-\mathbf{\dot{Z}}^{(k)} \right\rangle-\frac{\beta^{(k-1)}}{2}\|\mathbf{\dot{X}}^{(k)}-\mathbf{\dot{Z}}^{(k)}\|_{F}^{2}\\
&\!\!\!\!\!\!\!\!\!\!\!\!\!\!\!\!\!\!\! \!\!\!\!\!\!\!\!\!\!\!\!\!\!\!\!\!\!\!\!\!\!\!\!\!\!\!\!\!\!\!\!\!\!\!\!\!\!\!\!\!\!\!\!\!\!\!\!\!\!\!\!\!\!\!\!\!\!\!\!\!\!\!\! \! \! \! \!=\mathcal{L}(\mathbf{\dot{X}}^{(k)},\mathbf{\dot{Z}}^{(k)},\mathbf{\dot{\eta}}^{(k-1)},\beta^{(k-1)})-\left\langle\mathbf{\dot{\eta}}^{(k-1)},\frac{\mathbf{\dot{\eta}}^{(k)}-\mathbf{\dot{\eta}}^{(k-1)}}{\beta^{(k)}} \right\rangle-\frac{\beta^{(k-1)}}{2}\Big\|\frac{\mathbf{\dot{\eta}}^{(k)}-\mathbf{\dot{\eta}}^{(k-1)}}{\beta^{(k)}}\Big\|_{F}^{2}\\
&\!\!\!\!\!\!\!\!\!\!\!\!\!\!\!\!\!\!\! \!\!\!\!\!\!\!\!\!\!\!\!\!\!\!\!\!\!\!\!\!\!\!\!\!\!\!\!\!\!\!\!\!\!\!\!\!\!\!\!\!\!\!\!\!\!\!\!\!\!\!\!\!\!\!\!\!\!\!\!\!\!\!\! \! \! \! \!=\mathcal{L}(\mathbf{\dot{X}}^{(k)},\mathbf{\dot{Z}}^{(k)},\mathbf{\dot{\eta}}^{(k-1)},\beta^{(k-1)})+\frac{\|\mathbf{\dot{\eta}}^{(k-1)}\|_{F}^{2}-\|\mathbf{\dot{\eta}}^{(k)}\|_{F}^{2}}{2\beta^{(k-1)}}.
\end{aligned}
\end{equation}
Then, we can conclude that the sequence $\{\mathbf{\dot{Z}}^{(k)}\}$ is upper bounded. Since $\mathbf{\dot{\eta}}^{(k+1)}=\mathbf{\dot{\eta}}^{(k)}+\beta^{(k)}(\mathbf{\dot{X}}^{(k+1)}-\mathbf{\dot{Z}}^{(k+1)})$, $\{\mathbf{\dot{X}}^{(k)}\}$ is also upper bounded. Thus, there exists at least one accumulation point for $\{\mathbf{\dot{X}}^{(k)},\mathbf{\dot{Z}}^{(k)}\}$. Moreover, we obtain that
\begin{equation}\label{46}
    \lim_{k\rightarrow +\infty}\|\dot{\mathbf{X}}^{(k+1)}-\mathbf{\dot{Z}}^{(k+1)}\|_{F}=\lim_{k\rightarrow +\infty}\frac{1}{\beta^{(k)}}\|\mathbf{\dot{\eta}}^{(k+1)}-\mathbf{\dot{\eta}}^{(k)}\|_{F}=0,
\end{equation}
and that the accumulation point is a feasible solution to the objective function. Thus, the first equation in \eqref{41} is proved.

Finally, we prove that the change of sequence $\{\mathbf{\dot{X}}^{(k)}\}$ and $\{\mathbf{\dot{Z}}^{(k)}\}$ between consecutive iterations tends to be 0. For $\mathbf{\dot{X}}^{(k+1)}$ and $\mathbf{\dot{X}}^{(k)}=\frac{1}{\beta^{(k-1)}}(\mathbf{\dot{\eta}}^{(k)}-\mathbf{\dot{\eta}}^{(k-1)})+\mathbf{\dot{Z}}^{(k)}$, we have
\begin{equation}\label{47}
    \begin{aligned}
        \lim_{k\rightarrow +\infty}\|\mathbf{\dot{X}}^{(k+1)}-\mathbf{\dot{X}}^{(k)}\|_{F}&=\lim_{k\rightarrow +\infty}\Big\|(\lambda\mathbf{\dot{A}}^{\ast}\mathbf{\dot{A}}+\beta^{(k)}\mathbf{\dot{I}})^{-1}(\lambda\mathbf{\dot{A}}^{\ast}\mathbf{\dot{Y}}+\beta^{(k)}\mathbf{\dot{Z}}^{(k)}-\mathbf{\dot{\eta}}^{(k)})\\
        &-\frac{1}{\beta^{(k-1)}}(\mathbf{\dot{\eta}}^{(k)}-\mathbf{\dot{\eta}}^{(k-1)})-\mathbf{\dot{Z}}^{(k)}\Big\|_{F}\\
        &=\lim_{k\rightarrow +\infty}\Big\|(\lambda\mathbf{\dot{A}}^{\ast}\mathbf{\dot{A}}+\beta^{(k)}\mathbf{\dot{I}})^{-1}(\lambda\mathbf{\dot{A}}^{\ast}\mathbf{\dot{Y}}-\lambda\mathbf{\dot{A}}^{\ast}\mathbf{\dot{A}}\mathbf{\dot{Z}}^{(k)}-\mathbf{\dot{\eta}}^{(k)})\\
        &-\frac{1}{\beta^{(k-1)}}(\mathbf{\dot{\eta}}^{(k)}-\mathbf{\dot{\eta}}^{(k-1)})\Big\|_{F}\\
        &\leq \lim_{k\rightarrow +\infty}\|(\lambda\mathbf{\dot{A}}^{\ast}\mathbf{\dot{A}}+\beta^{(k)}\mathbf{\dot{I}})^{-1}(\lambda\mathbf{\dot{A}}^{\ast}\mathbf{\dot{Y}}-\lambda\mathbf{\dot{A}}^{\ast}\mathbf{\dot{A}}\mathbf{\dot{Z}}^{(k)}-\mathbf{\dot{\eta}}^{(k)})\|_{F}\\
        &+\frac{1}{\beta^{(k-1)}}\|\mathbf{\dot{\eta}}^{(k)}-\mathbf{\dot{\eta}}^{(k-1)}\|_{F}=0.
    \end{aligned}
\end{equation}
Since $\mathbf{\dot{Z}}^{(k+1)}=\frac{1}{\beta^{(k)}}\mathbf{\dot{\eta}}^{(k)}-\frac{1}{\beta^{(k)}}\mathbf{\dot{\eta}}^{(k+1)}+\mathbf{\dot{X}}^{(k+1)}$, we conclude that
\begin{equation}\label{48}
    \begin{aligned}
        \lim_{k\rightarrow +\infty}\|\dot{\mathbf{Z}}^{(k+1)}-\mathbf{\mathbf{Z}}^{(k)}\|_{F}&=\lim_{k\rightarrow +\infty}\Big\|\frac{1}{\beta^{(k)}}\mathbf{\dot{\eta}}^{(k)}-\frac{1}{\beta^{(k)}}\mathbf{\dot{\eta}}^{(k+1)}+\mathbf{\dot{X}}^{(k+1)}-\mathbf{\dot{Z}}^{(k)}\Big\|_{F}\\
        &=\lim_{k\rightarrow +\infty}\Big\|\mathbf{\dot{X}}^{(k)}+\frac{1}{\beta^{(k-1)}}\mathbf{\dot{\eta}}^{(k-1)}-\mathbf{\dot{Z}}^{(k)}+\mathbf{\dot{X}}^{(k+1)}\\
        &-\mathbf{\dot{X}}^{(k)}-\frac{1}{\beta^{(k-1)}}\mathbf{\dot{\eta}}^{(k-1)}+\frac{1}{\beta^{(k)}}\mathbf{\dot{\eta}}^{(k)}-\frac{1}{\beta^{(k)}}\mathbf{\dot{\eta}}^{(k+1)}\Big\|_{F}\\
        &\leq \lim_{k\rightarrow +\infty}\sum_{j=1}^{N}\|\sum_{i}w_{ij}^{(k-1)}/\beta^{(k-1)}\|_{F}+\|\mathbf{\dot{X}}^{(k+1)}-\mathbf{\dot{X}}^{(k)}\|_{F}\\\
        &+\Big\|\frac{1}{\beta^{(k-1)}}\mathbf{\dot{\eta}}^{(k-1)}-\frac{1}{\beta^{(k)}}\mathbf{\dot{\eta}}^{(k)}+\frac{1}{\beta^{(k)}}\mathbf{\dot{\eta}}^{(k+1)}\Big\|_{F}=0.
    \end{aligned}
\end{equation}
This completes the proof. $\blacksquare$
\end{proof}

\begin{theorem}
Assume that the sequence of parameter $\{\beta^{(k)}\}$ is unbounded, then the sequences $\{\mathbf{\dot{X}}^{(k)}\}$ and $\{\mathbf{\dot{Z}}^{(k)}\} $ generated by Algorithm 2 must satisfy the following:
\begin{equation}\label{49}
    \begin{aligned}
        \lim_{k\rightarrow +\infty}\|\mathbf{\dot{X}}^{(k+1)}-\mathbf{\dot{Z}}^{(k+1)}\|_{F}=0,\\
        \lim_{k\rightarrow +\infty}\|\mathbf{\dot{X}}^{(k+1)}-\mathbf{\dot{X}}^{(k)}\|_{F}=0,\\
        \lim_{k\rightarrow +\infty}\|\mathbf{\dot{Z}}^{(k+1)}-\mathbf{\dot{Z}}^{(k)}\|_{F}=0.\\
    \end{aligned}
\end{equation}
\end{theorem}

\begin{proof}
 We first prove that the sequence $\{\mathbf{\dot{\eta}}^{(k)}\}$ generated by Algorithm 2 is upper bounded. Suppose $\mathbf{\dot{P}}^{(k)}=\frac{1}{\beta^{(k)}}\mathbf{\dot{\eta}}^{(k)}+\mathbf{\dot{X}}^{(k+1)}$, then we can get $N$ patch-based quaternion matrices $\mathbf{\dot{P}}^{(k)}_{j}$ and $N$ patch-based quaternion matrices $\mathbf{\dot{Z}}^{(k+1)}_{j}(j=1,2,\cdots,N)$ by utilizing the NSS prior. Denote the QSVD of patch-based quaternion matrix $\mathbf{\dot{P}}^{(k)}_{j}$ in the $k+1$ iteration as $\mathbf{\dot{U}}^{(k)}_{j}\Sigma^{(k)}_{j}\mathbf{\dot{V}}^{(k)^{\ast}}_{j}$, where $\Sigma^{(k)}_{j}=diag(\sigma^{(k)}_{1j},\sigma^{(k)}_{2j},\cdots,\sigma^{(k)}_{rj})$ is the diagonal singular value matrix. By using the GST algorithm for $\mathbf{\dot{Z}}$ subproblem, we have $\mathbf{\dot{Z}}^{(k+1)}_{j}=\mathbf{\dot{U}}^{(k)}_{j}\Delta^{(k)}_{j}\mathbf{\dot{V}}^{(k)^{\ast}}_{j}$, where $\Delta^{(k)}_{j}=diag(\delta^{(k)}_{1j},\delta^{(k)}_{2j},\cdots,\delta^{(k)}_{rj})$ is the singular value matrix after generalized soft-thresholding operation. Then,
\begin{equation}\label{50}
\begin{aligned}
\|\mathbf{\dot{\eta}}^{(k+1)}\|_{F}^{2}&=\|\mathbf{\dot{\eta}}^{(k)}+\beta^{(k)}(\mathbf{\dot{X}}^{(k+1)}-\mathbf{\dot{Z}}^{(k+1)})\|_{F}^{2}\\
&=(\beta^{(k)})^{2}\|\frac{1}{\beta^{(k)}}\mathbf{\dot{\eta}}^{(k)}+\mathbf{\dot{X}}^{(k+1)}-\mathbf{\dot{Z}}^{(k+1)}\|_{F}^{2}\\
&\leq({\beta^{(k)}})^{2}\sum_{j=1}^{N}\|\mathbf{\dot{U}}^{(k)}_{j}\Sigma^{(k)}_{j}\mathbf{\dot{V}}^{(k)^{\ast}}_{j}-\mathbf{\dot{U}}^{(k)}_{j}\Delta^{(k)}_{j}\mathbf{\dot{V}}^{(k)^{\ast}}_{j}\|_{F}^{2}\\
&=({\beta^{(k)}})^{2}\sum_{j=1}^{N}\|\Sigma^{(k)}_{j}-\Delta^{(k)}_{j}\|_{F}^{2}=\sum_{j=1}^{N}({\beta^{(k)}})^{2}\|\sum_{i}Jw_{ij}^{(k)}/\beta^{(k)}\|_{F}^{2}\\
&=\sum_{j=1}^{N}\|J\sum_{i}w_{ij}^{(k)}\|_{F}^{2}.
\end{aligned}
\end{equation}
Hence, the sequence $\{\mathbf{\dot{\eta}}^{(k)}\}$ generated by Algorithm 2 is upper bounded.

We then prove that the sequence of the Lagrange function $\{\mathcal{L}(\dot{\mathbf{X}}^{(k+1)},\mathbf{\dot{Z}}^{(k+1)},$\\
$\mathbf{\dot{\eta}}^{(k+1)},\beta^{(k+1)})\}$ is also upper bounded. Due to the updating rule of $\mathbf{\dot{\eta}}$, we can get
\begin{equation} \label{51}
    \begin{aligned}
        \mathcal{L}(\mathbf{\dot{X}}^{(k+1)},\mathbf{\dot{Z}}^{(k+1)},\mathbf{\dot{\eta}}^{(k+1)},\beta^{(k+1)})\\
        &\!\!\!\!\!\!\!\!\!\!\!\!\!\!\! \!\!\!\!\!\!\!\!\!\!\!\!\!\!\!\!\!\!\!\!\!\!\!\!\!\!\!\!\!\!\!\!\!\!\!\!\!\!\!\!\!\!\!\!\!\!\!\!\!\!\!\!\!\!\!\!\!\!\!\!\!\!\!\! \! \! \! \!= \|\mathbf{\dot{A}}\mathbf{\dot{X}}^{(k+1)}-\mathbf{\dot{Y}}\|_{F}^{2}+\|\mathbf{\dot{Z}}^{(k+1)}\|_{\mathbf{w},S_{p}}^{p}\hfill\\
        &\!\!\!\!\!\!\!\!\!\!\!\!\!\!\! \!\!\!\!\!\!\!\!\!\!\!\!\!\!\!\!\!\!\!\!\!\!\!\!\!\!\!\!\!\!\!\!\!\!\!\!\!\!\!\!\!\!\!\!\!\!\!\!\!\!\!\!\!\!\!\!\!\!\!\!\!\!\!\! \! \! \! \!+\left\langle \mathbf{\dot{\eta}}^{(k+1)},\mathbf{\dot{X}}^{(k+1)}-\mathbf{\dot{Z}}^{(k+1)} \right\rangle+\frac{\beta^{(k+1)}}{2}\|\mathbf{\dot{X}}^{(k+1)}-\mathbf{\dot{Z}}^{(k+1)}\|_{F}^{2}\\
        &\!\!\!\!\!\!\!\!\!\!\!\!\!\!\! \!\!\!\!\!\!\!\!\!\!\!\!\!\!\!\!\!\!\!\!\!\!\!\!\!\!\!\!\!\!\!\!\!\!\!\!\!\!\!\!\!\!\!\!\!\!\!\!\!\!\!\!\!\!\!\!\!\!\!\!\!\!\!\! \! \! \! \!=\mathcal{L}(\mathbf{\dot{X}}^{(k+1)},\mathbf{\dot{Z}}^{(k+1)},\mathbf{\dot{\eta}}^{(k)},\beta^{(k)})\\
&\!\!\!\!\!\!\!\!\!\!\!\!\!\!\! \!\!\!\!\!\!\!\!\!\!\!\!\!\!\!\!\!\!\!\!\!\!\!\!\!\!\!\!\!\!\!\!\!\!\!\!\!\!\!\!\!\!\!\!\!\!\!\!\!\!\!\!\!\!\!\!\!\!\!\!\!\!\!\! \! \! \! \!+\left\langle\mathbf{\dot{\eta}}^{(k+1)}-\mathbf{\dot{\eta}}^{(k)},\mathbf{\dot{X}}^{(k+1)}-{\mathbf{\dot{Z}}}^{(k+1)} \right\rangle+\frac{\beta^{(k+1)}-\beta^{(k)}}{2}\|\mathbf{\dot{X}}^{(k+1)}-\mathbf{\dot{Z}}^{(k+1)}\|_{F}^{2}\\
        &\!\!\!\!\!\!\!\!\!\!\!\!\!\!\! \!\!\!\!\!\!\!\!\!\!\!\!\!\!\!\!\!\!\!\!\!\!\!\!\!\!\!\!\!\!\!\!\!\!\!\!\!\!\!\!\!\!\!\!\!\!\!\!\!\!\!\!\!\!\!\!\!\!\!\!\!\!\!\! \! \! \! \!=\mathcal{L}(\mathbf{\dot{X}}^{(k+1)},\mathbf{\dot{Z}}^{(k+1)},\mathbf{\dot{\eta}}^{(k)},\beta^{(k)})\\
&\!\!\!\!\!\!\!\!\!\!\!\!\!\!\! \!\!\!\!\!\!\!\!\!\!\!\!\!\!\!\!\!\!\!\!\!\!\!\!\!\!\!\!\!\!\!\!\!\!\!\!\!\!\!\!\!\!\!\!\!\!\!\!\!\!\!\!\!\!\!\!\!\!\!\!\!\!\!\! \! \! \! \!+\left\langle\mathbf{\dot{\eta}}^{(k+1)}-\mathbf{\dot{\eta}}^{(k)},\frac{\mathbf{\dot{\eta}}^{(k+1)}-\mathbf{\dot{\eta}}^{(k)}}{\beta^{(k)}} \right\rangle+\frac{\beta^{(k+1)}-\beta^{(k)}}{2}\Big\|\frac{\mathbf{\dot{\eta}}^{(k+1)}-\mathbf{\dot{\eta}}^{(k)}}{\beta^{(k)}}\Big\|_{F}^{2}\\
        &\!\!\!\!\!\!\!\!\!\!\!\!\!\!\! \!\!\!\!\!\!\!\!\!\!\!\!\!\!\!\!\!\!\!\!\!\!\!\!\!\!\!\!\!\!\!\!\!\!\!\!\!\!\!\!\!\!\!\!\!\!\!\!\!\!\!\!\!\!\!\!\!\!\!\!\!\!\!\! \! \! \! \!=\mathcal{L}(\mathbf{\dot{X}}^{(k+1)},\mathbf{\dot{Z}}^{(k+1)},{\mathbf{\dot{\eta}}^{(k)}},\beta^{(k)})+\frac{\beta^{(k+1)}+\beta^{(k)}}{2(\beta^{(k)})^{2}}\|\mathbf{\dot{\eta}}^{(k+1)}-\mathbf{\dot{\eta}}^{(k)}\|_{F}^{2}.
    \end{aligned}
\end{equation}
Since $\{\mathbf{\dot{\eta}}^{(k)}\}$ is upper bounded, the sequence $\{\mathbf{\dot{\eta}}^{(k+1)}-\mathbf{\dot{\eta}}^{(k)}\}$ is also upper bounded. Suppose $M_{0}$ is the upper bound of $\{\mathbf{\dot{\eta}}^{(k+1)}-\mathbf{\dot{\eta}}^{(k)}\}$ for all $k\geq 0$, i.e., $\|\mathbf{\dot{\eta}}^{(k+1)}-\mathbf{\dot{\eta}}^{(k)}\|_{F}^{2}\leq M_{0}^{2}, \forall \ k\geq 0$. Besides, the inequality $\mathcal{L}(\mathbf{\dot{X}}^{(k+1)},\mathbf{\dot{Z}}^{(k+1)},\mathbf{\dot{\eta}}^{(k)},$\\
$\beta^{(k)})\leq \mathcal{L}(\mathbf{\dot{X}}^{(k)},\mathbf{\dot{Z}}^{(k)},\mathbf{\dot{\eta}}^{(k)},\beta^{(k)})$ always holds since we have the globally optimal solution of $\mathbf{\dot{X}}$ and $\mathbf{\dot{Z}}$ in their corresponding sub-problems. Therefore, we have
\begin{equation} \label{52}
    \begin{aligned}
        \mathcal{L}(\mathbf{\dot{X}}^{(k+1)},\mathbf{\dot{Z}}^{(k+1)},\mathbf{\dot{\eta}}^{(k+1)},\beta^{(k+1)})&\leq\mathcal{L}(\mathbf{\dot{X}}^{(k+1)},\mathbf{\dot{Z}}^{(k+1)},\mathbf{\dot{\eta}}^{(k)},\beta^{(k)})+\frac{\beta^{(k+1)}+\beta^{(k)}}{2(\beta^{(k)})^{2}}M_{0}^{2}\\
        &\leq \mathcal{L}(\mathbf{\dot{X}}^{(1)},\mathbf{\dot{Z}}^{(1)},\mathbf{\dot{\eta}}^{(0)},\beta^{(0)})+M_{0}^{2}\sum_{k=0}^{\infty}\frac{\beta^{(k+1)}+\beta^{(k)}}{2(\beta^{(k)})^{2}}\\
        &=\mathcal{L}(\mathbf{\dot{X}}^{(1)},\mathbf{\dot{Z}}^{(1)},\mathbf{\dot{\eta}}^{(0)},\beta^{(0)})+M_{0}^{2}\sum_{k=0}^{\infty}\frac{1+\mu}{2\beta^{(0)}\mu^{k}}\\
        &\leq\mathcal{L}(\mathbf{\dot{X}}^{(1)},\mathbf{\dot{Z}}^{(1)},\mathbf{\dot{\eta}}^{(0)},\beta^{(0)})+\frac{M_{0}^{2}}{\beta^{(0)}}\sum_{k=0}^{\infty}\frac{1}{\mu^{k-1}}\\
        &< +\infty.
    \end{aligned}
\end{equation}
Thus, $\{\mathcal{L}(\mathbf{\dot{X}}^{(k+1)},\mathbf{\dot{Z}}^{(k+1)},\mathbf{\dot{\eta}}^{(k+1)},\beta^{(k+1)})\}$ is upper bounded.

We next prove the sequences of $\{\dot{\mathbf{X}}^{(k)}\}$ and $\{\mathbf{\dot{Z}}^{(k)}\}$ are upper bounded. According to the formula \eqref{30}, we have
\begin{equation}\label{53}
\begin{aligned}
    \|\mathbf{\dot{A}}\mathbf{\dot{X}}^{(k+1)}-\mathbf{\dot{Y}}\|_{F}^{2}+\|\mathbf{\dot{Z}}^{(k+1)}\|_{\mathbf{w},S_{p}}^{p}\\
&\!\!\!\!\!\!\!\!\!\!\!\!\!\!\!\!\!\!\!\! \!\!\!\!\!\!\!\!\!\!\!\!\!\!\!\!\!\!\!\!\!\!\!\!\!\!\!\!\!\!\!\!\!\!\!\!\!\!\!\!\!\!\!\!\!\!\!\!\!\!\!\!\!\!\!\!\!\!\!\!\!\!\!\! \! \! \! \!=\mathcal{L}(\mathbf{\dot{X}}^{(k)},\mathbf{\dot{Z}}^{(k)},\mathbf{\dot{\eta}}^{(k-1)},\beta^{(k-1)})-\left\langle\mathbf{\dot{\eta}}^{(k)},\mathbf{\dot{X}}^{(k)}-\mathbf{\dot{Z}}^{(k)} \right\rangle-\frac{\beta^{(k-1)}}{2}\|\mathbf{\dot{X}}^{(k)}-\mathbf{\dot{Z}}^{(k)}\|_{F}^{2}\\
&\!\!\!\!\!\!\!\!\!\!\!\!\!\!\!\!\!\!\!\! \!\!\!\!\!\!\!\!\!\!\!\!\!\!\!\!\!\!\!\!\!\!\!\!\!\!\!\!\!\!\!\!\!\!\!\!\!\!\!\!\!\!\!\!\!\!\!\!\!\!\!\!\!\!\!\!\!\!\!\!\!\!\!\! \! \! \! \!=\mathcal{L}(\mathbf{\dot{X}}^{(k)},\mathbf{\dot{Z}}^{(k)},\mathbf{\dot{\eta}}^{(k-1)},\beta^{(k-1)})-\left\langle\mathbf{\dot{\eta}}^{(k-1)},\frac{\mathbf{\dot{\eta}}^{(k)}-\mathbf{\dot{\eta}}^{(k-1)}}{\beta^{(k)}} \right\rangle-\frac{\beta^{(k-1)}}{2}\Big\|\frac{\mathbf{\dot{\eta}}^{(k)}-\mathbf{\dot{\eta}}^{(k-1)}}{\beta^{(k)}}\Big\|_{F}^{2}\\
&\!\!\!\!\!\!\!\!\!\!\!\!\!\!\!\!\!\!\!\! \!\!\!\!\!\!\!\!\!\!\!\!\!\!\!\!\!\!\!\!\!\!\!\!\!\!\!\!\!\!\!\!\!\!\!\!\!\!\!\!\!\!\!\!\!\!\!\!\!\!\!\!\!\!\!\!\!\!\!\!\!\!\!\! \! \! \! \!=\mathcal{L}(\mathbf{\dot{X}}^{(k)},\mathbf{\dot{Z}}^{(k)},\mathbf{\dot{\eta}}^{(k-1)},\beta^{(k-1)})+\frac{\|\mathbf{\dot{\eta}}^{(k-1)}\|_{F}^{2}-\|\mathbf{\dot{\eta}}^{(k)}\|_{F}^{2}}{2\beta^{(k-1)}}.
\end{aligned}
\end{equation}
Therefore, we can conclude that the sequence $\{\mathbf{\dot{Z}}^{(k)}\}$ is upper bounded. Since $\mathbf{\dot{\eta}}^{(k+1)}=\mathbf{\dot{\eta}}^{(k)}+\beta^{(k)}(\mathbf{\dot{X}}^{(k+1)}-\mathbf{\dot{Z}}^{(k+1)})$, $\{\mathbf{\dot{X}}^{(k)}\}$ is also upper bounded. Thus, there exists at least one accumulation point for $\{\mathbf{\dot{X}}^{(k)},\mathbf{\dot{Z}}^{(k)}\}$. Moreover, we obtain that
\begin{equation}\label{54}
    \lim_{k\rightarrow +\infty}\|\dot{\mathbf{X}}^{(k+1)}-\mathbf{\dot{Z}}^{(k+1)}\|_{F}=\lim_{k\rightarrow +\infty}\frac{1}{\beta^{(k)}}\|\mathbf{\dot{\eta}}^{(k+1)}-\mathbf{\dot{\eta}}^{(k)}\|_{F}=0,
\end{equation}
and that the accumulation point is a feasible solution to the objective function. Thus, the first equation in \eqref{49} is proved.

Finally, we prove that the change of sequence $\{\mathbf{\dot{X}}^{(k)}\}$ and $\{\mathbf{\dot{Z}}^{(k)}\}$ between consecutive iterations tends to be 0. For $\mathbf{\dot{X}}^{(k+1)}$ and $\mathbf{\dot{X}}^{(k)}=\frac{1}{\beta^{(k-1)}}(\mathbf{\dot{\eta}}^{(k)}-\mathbf{\dot{\eta}}^{(k-1)})+\mathbf{\dot{Z}}^{(k)}$, we have
\begin{equation}\label{55}
    \begin{aligned}
        \lim_{k\rightarrow +\infty}\|\mathbf{\dot{X}}^{(k+1)}-\mathbf{\dot{X}}^{(k)}\|_{F}&=\lim_{k\rightarrow +\infty}\Big\|(\lambda\mathbf{\dot{A}}^{\ast}\mathbf{\dot{A}}+\beta^{(k)}\mathbf{\dot{I}})^{-1}(\lambda\mathbf{\dot{A}}^{\ast}\mathbf{\dot{Y}}+\beta^{(k)}\mathbf{\dot{Z}}^{(k)}-\mathbf{\dot{\eta}}^{(k)})\\
        &-\frac{1}{\beta^{(k-1)}}(\mathbf{\dot{\eta}}^{(k)}-\mathbf{\dot{\eta}}^{(k-1)})-\mathbf{\dot{Z}}^{(k)}\Big\|_{F}\\
        &=\lim_{k\rightarrow +\infty}\Big\|(\lambda\mathbf{\dot{A}}^{\ast}\mathbf{\dot{A}}+\beta^{(k)}\mathbf{\dot{I}})^{-1}(\lambda\mathbf{\dot{A}}^{\ast}\mathbf{\dot{Y}}-\lambda\mathbf{\dot{A}}^{\ast}\mathbf{\dot{A}}\mathbf{\dot{Z}}^{(k)}-\mathbf{\dot{\eta}}^{(k)})\\
        &-\frac{1}{\beta^{(k-1)}}(\mathbf{\dot{\eta}}^{(k)}-\mathbf{\dot{\eta}}^{(k-1)})\Big\|_{F}\\
        &\leq \lim_{k\rightarrow +\infty}\|(\lambda\mathbf{\dot{A}}^{\ast}\mathbf{\dot{A}}+\beta^{(k)}\mathbf{\dot{I}})^{-1}(\lambda\mathbf{\dot{A}}^{\ast}\mathbf{\dot{Y}}-\lambda\mathbf{\dot{A}}^{\ast}\mathbf{\dot{A}}\mathbf{\dot{Z}}^{(k)}-\mathbf{\dot{\eta}}^{(k)})\|_{F}\\
        &+\frac{1}{\beta^{(k-1)}}\|\mathbf{\dot{\eta}}^{(k)}-\mathbf{\dot{\eta}}^{(k-1)}\|_{F}=0.
    \end{aligned}
\end{equation}
Since $\mathbf{\dot{Z}}^{(k+1)}=\frac{1}{\beta^{(k)}}\mathbf{\dot{\eta}}^{(k)}-\frac{1}{\beta^{(k)}}\mathbf{\dot{\eta}}^{(k+1)}+\mathbf{\dot{X}}^{(k+1)}$, we conclude that
\begin{equation}\label{56}
    \begin{aligned}
        \lim_{k\rightarrow +\infty}\|\dot{\mathbf{Z}}^{(k+1)}-\mathbf{\mathbf{Z}}^{(k)}\|_{F}&=\lim_{k\rightarrow +\infty}\Big\|\frac{1}{\beta^{(k)}}\mathbf{\dot{\eta}}^{(k)}-\frac{1}{\beta^{(k)}}\mathbf{\dot{\eta}}^{(k+1)}+\mathbf{\dot{X}}^{(k+1)}-\mathbf{\dot{Z}}^{(k)}\Big\|_{F}\\
        &=\lim_{k\rightarrow +\infty}\Big\|\mathbf{\dot{X}}^{(k)}+\frac{1}{\beta^{(k-1)}}\mathbf{\dot{\eta}}^{(k-1)}-\mathbf{\dot{Z}}^{(k)}+\mathbf{\dot{X}}^{(k+1)}\\
        &-\mathbf{\dot{X}}^{(k)}-\frac{1}{\beta^{(k-1)}}\mathbf{\dot{\eta}}^{(k-1)}+\frac{1}{\beta^{(k)}}\mathbf{\dot{\eta}}^{(k)}-\frac{1}{\beta^{(k)}}\mathbf{\dot{\eta}}^{(k+1)}\Big\|_{F}\\
        &\leq \lim_{k\rightarrow +\infty}\sum_{j=1}^{N}\|\sum_{i}Jw_{ij}^{(k-1)}/\beta^{(k-1)}\|_{F}+\|\mathbf{\dot{X}}^{(k+1)}-\mathbf{\dot{X}}^{(k)}\|_{F}\\\
        &+\Big\|\frac{1}{\beta^{(k-1)}}\mathbf{\dot{\eta}}^{(k-1)}-\frac{1}{\beta^{(k)}}\mathbf{\dot{\eta}}^{(k)}+\frac{1}{\beta^{(k)}}\mathbf{\dot{\eta}}^{(k+1)}\Big\|_{F}=0.
    \end{aligned}
\end{equation}
This completes the proof. $\blacksquare$
\end{proof}

It can be seen intuitively from Fig. 1 that $\|\mathbf{\dot{X}}^{(k+1)}-\mathbf{\dot{X}}^{(k)}\|_{F}, \|\mathbf{\dot{Z}}^{(k+1)}-\mathbf{\dot{Z}}^{(k)}\|_{F}$ and $\|\mathbf{\dot{X}}^{(k+1)}-\mathbf{\dot{Z}}^{(k+1)}\|_{F}$ simultaneously approach 0 during the iteration process.



\begin{figure}[htbp]
    \centering
    \includegraphics[width=0.9\textwidth]{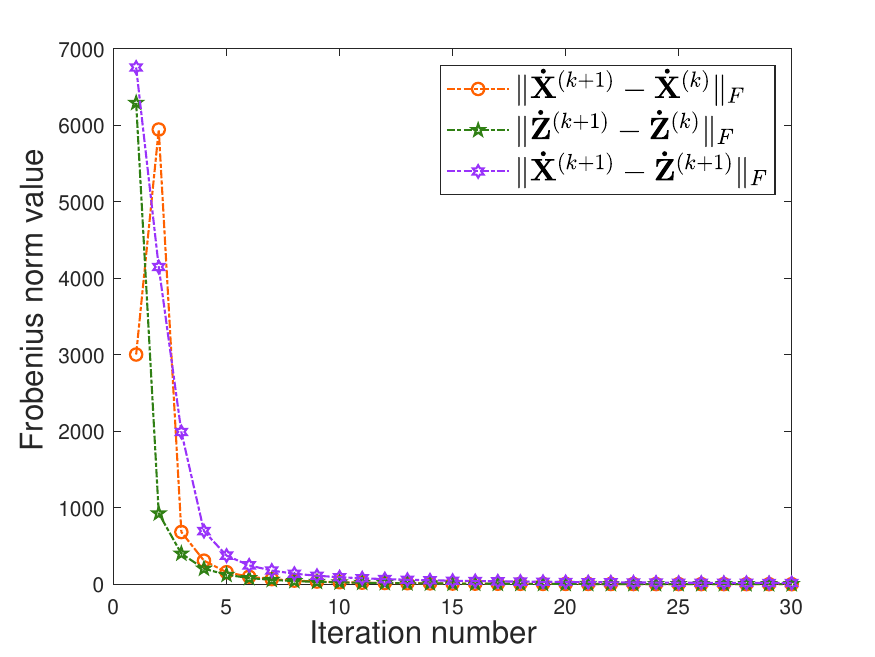}
    \caption{The convergence curves of $\|\mathbf{\dot{X}}^{(k+1)}-\mathbf{\dot{X}
    }^{(k)}\|_{F}$, $\|\mathbf{\dot{Z}}^{(k+1)}-\mathbf{\dot{Z}}^{(k)}\|_{F}$ and $\|\mathbf{\dot{X}}^{(k+1)}-\mathbf{\dot{Z}}^{(k+1)}\|_{F}$ of the proposed QWSNM algorithm. Test image: ``Bird''.}
\end{figure}

\section{Experimental Results}
To demonstrate the effectiveness of our proposed QWSNM algorithm, we carry out extensive experiments on two representative CIR tasks, i.e.,  color image denoising and deblurring. To evaluate the quality of the restored images, the peak signal to noise ratio (PSNR) and structural similarity (SSIM)  \cite{2004Image} are chosen as the quantitative metrics.
Generally, a higher PSNR value suggests the recovery result is of higher quality, but in some scenarios it may not. To this end, the perceptual quality metric SSIM  \cite{2004Image} is calculated to evaluate the visual quality. The higher SSIM value indicates the better visual quality.
All the experiments were implemented in Matlab 2021b on a desktop with 3.30GHz CPU and 16GB RAM.

\begin{figure}
    \centering
    \includegraphics[width=0.95\textwidth]{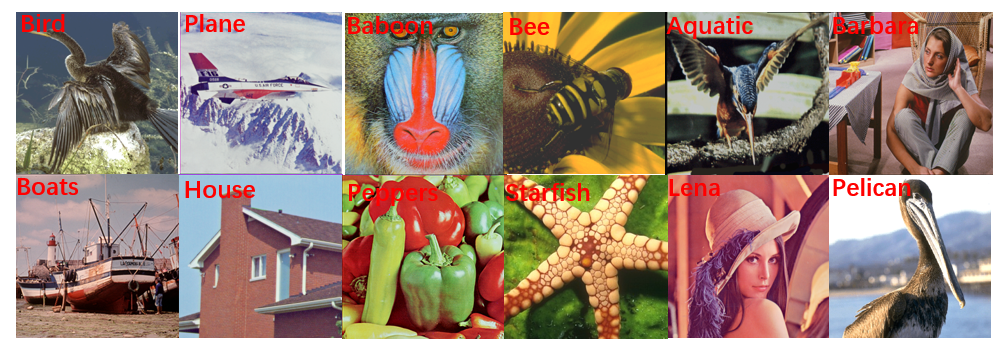}
    \caption{Images in Set12 \cite{wang2019variational}.}
\end{figure}

\begin{table*}[htbp]
\begin{center}{ \tiny
\resizebox{\textwidth}{90mm}{
\begin{tabular}{|c||c|c|c|c|c|c|}
\hline
\ Methods & SV-TV \cite{jia2019color} & CBM3D \cite{dabov2007image} & IRCNN \cite{zhang2017learning} & pQSTV \cite{wu2022total} & QWNNM \cite{yu2019quaternion} & QWSNM (Ours) \\
\hline
\multicolumn{7}{|c|}{$\sigma = 15$}\\
\hline
Bird & 31.16/0.9383 & 32.54/0.9567 & \textbf{32.96/0.9607} & 31.17/0.9414 & 32.77/0.9565 & \underline{32.86/0.9593} \\
Plane & 30.82/0.8940 & 33.76/0.9554 & 34.25/0.9573 & 32.72/0.9382 & \underline{34.69/0.9598} & \textbf{34.81/0.9607} \\
Baboon & 28.14/0.9321 & 29.59/0.9474 & 29.96/0.9514 & 28.68/0.9344 & \underline{30.18/0.9537} & \textbf{30.24/0.9539} \\
Bee & 31.09/0.9394 & 34.69/0.9647 & 35.31/0.9696 & 34.06/0.9606 & \underline{35.78/0.9729} & \textbf{35.81/0.9733} \\
Aquatic & 31.19/0.9093 & 33.31/0.9491 & 33.82/0.9517 & 31.96/0.9335 & \underline{33.88/0.9532} & \textbf{34.06/0.9537} \\
Barbara & 29.64/0.9349 & 32.67/0.9641 & 32.57/0.9603 & 31.57/0.9526 & \underline{33.14/0.9662} & \textbf{33.24/0.9672} \\
Boat & 30.38/0.9132 & 32.71/0.9631 & \textbf{33.30/0.9672} & 31.47/0.9501 & 33.08/0.9659 & \underline{33.15/0.9661} \\
House & 30.82/0.9575 & \textbf{34.85/0.9801} & 34.07/0.9779 & 33.42/0.9746 & \underline{34.74/0.9796} & 34.72/0.9794 \\
Peppers & 28.82/0.9750 & 32.53/0.9888 & 32.85/0.9899 & 32.18/0.9881 & \underline{33.35/0.9907} & \textbf{33.42/0.9908} \\
Starfish & 31.00/0.9718 & 32.96/0.9862 & \textbf{33.61/0.9879} & 31.79/0.9822 & 33.38/0.9866 & \underline{33.43/0.9869} \\
Lena & 32.15/0.9843 & 33.94/0.9892 & 33.91/0.9891 & 33.00/09866 & \underline{34.13/0.9898} & \textbf{34.15/0.9899} \\
Pelican & 31.15/0.9084 & 32.65/0.9449 & 32.37/0.9376 & 31.26/0.9218 & \underline{32.88/0.9479} & \textbf{32.94/0.9489} \\
Average & 30.53/0.9382 & 33.02/0.9658 & 33.25/0.9667 & 31.94/0.9553 & \underline{33.50/0.9686} & \textbf{33.57/0.9692} \\
\hline
\multicolumn{7}{|c|}{$\sigma = 25$}\\
\hline
Bird & 28.23/0.8875 & 29.73/0.9225 & \textbf{30.22/0.9311} & 28.39/0.8973 & 29.52/0.9223 & \underline{30.11/0.9291} \\
Plane & 28.95/0.8586 & 31.14/0.8586 & 31.63/0.9346 & 30.06/0.9055 & \underline{32.10/0.9415} & \textbf{32.16/0.9425} \\
Baboon & 25.46/0.8650 & 27.05/0.9032 & 27.44/0.9132 & 26.19/0.8827 & \underline{27.62/0.9145} & \textbf{27.66/0.9166} \\
Bee & 29.73/0.9219 & 32.48/0.9483 & 32.98/0.9537 & 31.90/0.9417 & \underline{33.30/0.9563} & \textbf{33.38/0.9569} \\
Aquatic & 28.80/0.8716 & 30.68/0.9215 & 31.23/0.9275 & 29.39/0.8992 & \underline{31.27/0.9290} & \textbf{31.30/0.9301} \\
Barbara & 27.39/0.9006 & 30.04/0.9390 & 30.10/0.9376 & 28.89/0.9212 & \underline{30.49/0.9445} & \textbf{30.59/0.9455} \\
Boat & 28.06/0.8949 & 29.92/0.9372 & \textbf{30.59/0.9458} & 28.70/0.9181 & 30.31/0.9437 & \underline{30.39/0.9443} \\
House & 29.15/0.9479 & \textbf{33.03/0.9721} & 32.13/0.9694 & 31.78/0.9649 & 32.93/0.9717 & \underline{33.01/0.9719} \\
Peppers & 27.09/0.9655 & 30.06/0.9812 & 30.46/0.9832 & 29.51/0.9794 & \underline{30.84/0.9841} & \textbf{30.92/0.9844} \\
Starfish & 28.90/0.9666 & 30.27/0.9770 & \textbf{30.99/0.9800} & 28.99/0.9705 & 30.89/0.9786 & \underline{30.90/0.9787} \\
Lena & 29.78/0.9738 & 32.27/0.9847 & 32.30/0.9849 & 31.42/0.9814 & \underline{32.49/0.9856} & \textbf{32.55/0.9858} \\
Pelican & 27.77/0.8491 & 30.51/0.9206 & 30.58/0.9186 & 29.48/0.8976 & \underline{30.74/0.9226} & \textbf{30.76/0.9237} \\
Average & 28.28/0.9086 & 30.60/0.9449 & 30.86/0.9483 & 29.56/0.9300 & \underline{31.04/0.9485} & \textbf{31.14/0.9508} \\
\hline
\multicolumn{7}{|c|}{$\sigma = 35$}\\
\hline
Bird & 26.55/0.8869 & 27.92/0.8869 & \textbf{28.53/0.9022} & 26.91/0.8574 & 28.32/0.8974 & \underline{28.40/0.8996} \\
Plane & 27.36/0.7858 & 29.34/0.9055 & 29.89/0.9113 & 28.41/0.8734 & \underline{30.32/0.9233} & \textbf{30.41/0.9253} \\
Baboon & 24.38/0.8278 & 25.40/0.8555 & 25.94/0.8718 & 24.71/0.8329 & \underline{26.00/0.8726} & \textbf{26.10/0.8779} \\
Bee & 28.31/0.8914 & 30.86/0.9333 & 31.50/0.9405 & 30.56/0.9268 & \underline{31.82/0.9433} & \textbf{31.92/0.9439} \\
Aquatic & 27.05/0.8099 & 28.81/0.8890 & 29.47/0.9024 & 27.56/0.8563 & \underline{29.48/0.9041} & \textbf{29.62/0.9075} \\
Barbara & 26.01/0.8682 & 28.15/0.9120 & 28.57/0.9179 & 27.30/0.8937 & \underline{28.77/0.9219} & \textbf{28.89/0.9237}\\
Boat & 26.45/0.8472 & 28.00/0.9083 & \textbf{28.86/0.9238} & 27.19/0.8904 & 28.48/0.9183 & \underline{28.57/0.9194} \\
House & 27.75/0.9283 & 31.58/0.9636 & 30.78/0.9617 & 30.53/0.9564 & \underline{31.88/0.9659} & \textbf{32.06/0.9668} \\
Peppers & 26.04/0.9562 & 28.19/0.9722 & 28.88/0.9764 & 28.02/0.9712 & \underline{29.21/0.9774} & \textbf{29.30/0.9780} \\
Starfish & 27.19/0.9485 & 28.38/0.9672 & \textbf{29.36/0.9725} & 27.50/0.9602 & 29.23/0.9713 & \underline{29.25/0.9715} \\
Lena & 28.60/0.9665 & 30.91/0.9799 & 31.16/0.9812 & 30.22/0.9764 & \underline{31.30/0.9818} & \textbf{31.38/0.9821} \\
Pelican & 26.62/0.8042 & 29.09/0.8994 & 29.34/0.9011 & 28.56/0.8743 & \underline{29.36/0.9039} & \textbf{29.41/0.9066} \\
Average & 26.86/0.8729 & 28.89/0.9227 & 29.35/0.9307 & 28.12/0.9059 & \underline{29.52/0.9318} & \textbf{29.61/0.9335} \\
\hline
\multicolumn{7}{|c|}{$\sigma = 45$}\\
\hline
Bird & 25.19/0.7892 & 26.87/0.8602 & \textbf{27.33/0.8743} & 25.84/0.8179 & 26.90/0.8594 & \underline{27.13/0.8684} \\
Plane & 26.16/0.8385 & 28.32/0.8887 & 28.72/0.8920 & 27.24/0.8496 & \underline{29.00/0.9040} & \textbf{29.11/0.9086} \\
Baboon & 23.38/0.7797 & 24.43/0.8150 & \underline{24.91/0.8431} & 23.81/0.7951 & 24.75/0.8278 & \textbf{25.10/0.8512}\\
Bee & 27.39/0.8721 & 30.06/0.9237 & 30.43/0.9293 & 29.44/0.9122 & \underline{30.73/0.9309} & \textbf{30.83/0.9338} \\
Aquatic & 25.74/0.7645 & 27.73/0.8656 & \underline{28.21/0.8780} & 26.29/0.8249 & 28.01/0.8745 & \textbf{28.27/0.8833}\\
Barbara & 24.81/0.8349 & 27.25/0.8658 & \underline{27.49/0.9015} & 26.01/0.8689 & 27.45/0.8988 & \textbf{27.55/0.9035} \\
Boat & 25.08/0.8060 & 26.79/0.8851 & \textbf{27.62/0.9029} & 25.81/0.8612 & 26.98/0.8866 & \underline{27.23/0.8942} \\
House & 26.79/0.9120 & 30.86/0.9584 & 29.95/0.9547 & 29.46/0.9455 & \textbf{31.31/0.9620} & \underline{31.29/0.9619} \\
Peppers & 25.19/0.9474 & 27.02/0.9647 & 27.74/0.9699 & 26.83/0.9634 & \textbf{28.01/0.9713} & \underline{27.97/0.9710} \\
Starfish & 25.85/0.9339 & 27.36/0.9612 & \textbf{28.04/0.9657} & 26.30/0.9495 & 28.01/0.9625 & \underline{28.02/0.9638}\\
Lena & 27.37/0.9568 & 30.26/0.9772 & \underline{30.29/0.9778} & 29.21/0.9714 & 30.18/0.9773 & \textbf{30.37/0.9783} \\
Pelican & 25.33/0.7549 & 28.25/0.8876 & \underline{28.44/0.8889} & 27.23/0.8505 & 28.28/0.8824 & \textbf{28.45/0.8908}\\
Average & 25.69/0.8492 & 27.93/0.9044 & 28.26/\underline{0.9147} & 26.96/0.8842 & \underline{28.30}/0.9115 & \textbf{28.44/0.9174} \\
\hline
\end{tabular}}
}
\caption{Comparison of the denoising PSNR (dB) values. The best results are highlighted in \textbf{bold} and the second-best results are \underline{underlined}.}
\end{center}
\end{table*}

\begin{figure}[htbp]
\centering
\subfigure[Baboon]{\includegraphics[width=0.30\textwidth]{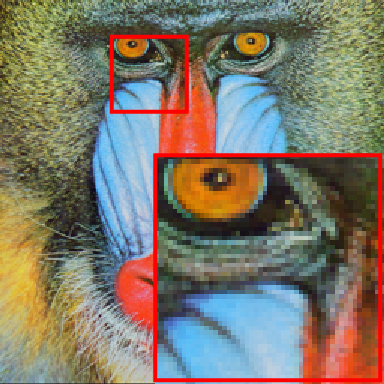}}
\subfigure[(15.08/0.4890)]{\includegraphics[width=0.30\textwidth]{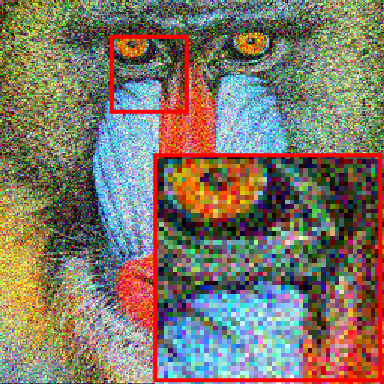}}
\subfigure[(23.38/0.7797)]{\includegraphics[width=0.30\textwidth]{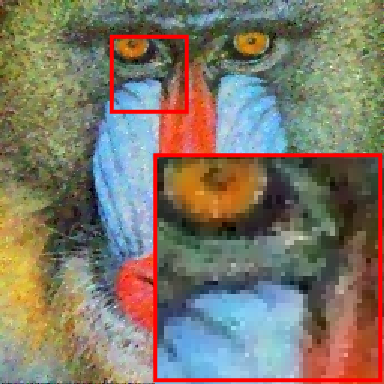}}
\subfigure[(24.43/0.8150)]{\includegraphics[width=0.30\textwidth]{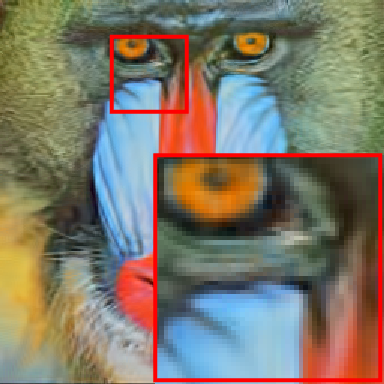}}
\subfigure[(24.91/0.8431)]{\includegraphics[width=0.30\textwidth]{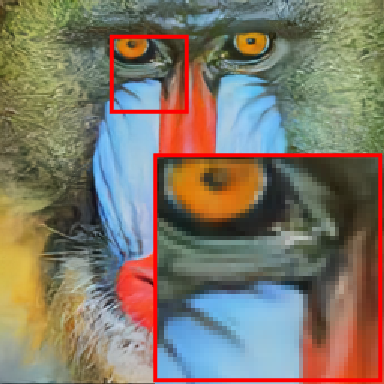}}
\subfigure[(23.81/0.7951)]{\includegraphics[width=0.30\textwidth]{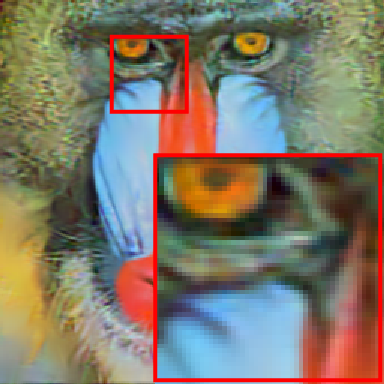}}
\subfigure[(24.75/0.8278)]{\includegraphics[width=0.30\textwidth]{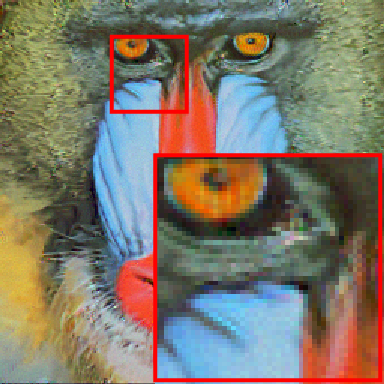}}
\subfigure[(25.10/0.8512)]{\includegraphics[width=0.30\textwidth]{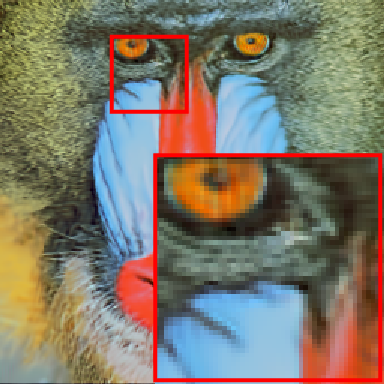}}
\caption{Denoising performance on ``Baboon'' with visual quality and numerical results (PSNR/SSIM). (a) Ground truth; (b) Noisy image ($\sigma = 45$); (c) SV-TV \cite{jia2019color}; (d) CBM3D \cite{dabov2007image}; (e) IRCNN \cite{zhang2017learning}; (f) pQSTV \cite{wu2022total}; (g) QWNNM \cite{yu2019quaternion}; (h) QWSNM (Ours).}
\end{figure}

\subsection{Color Image Denoising}
For the color image denoising, the recovery performance is evaluated on the benchmark Set12\footnote{The test images of Set12 are mainly collected from Wang et al. \cite{wang2019variational}. Concretely, Bird, Plane, Baboon, Bee, Aquatic, Barbara, Boat, House, Peppers and Starfish are all with the size of $256 \times 256$, Lena and Pelican are with the size of $512 \times 512$.} dataset, which contains 12 natural RGB color images shown in Fig. 2. All the color noisy images are synthesized by adding the Gaussian noise with mean zero and variance $\sigma^{2}$.
To be concrete, we test the scenarios of noise levels $\sigma$ equal to 15, 25, 35 and 45 respectively. In our experiments, we compare our proposed QWSNM algorithm with several prevailing methods, including SV-TV \cite{jia2019color}, CBM3D \cite{dabov2007image}, IRCNN \cite{zhang2017learning}, QWNNM \cite{yu2019quaternion}, pQSTV \cite{wu2022total}. Particularly, it is noted that SV-TV, QWNNM and pQSTV are all quaternion-based methods which represent the color image via a pure quaternion matrix to preserve the relationship of three channels. IRCNN \cite{zhang2017learning} is a convolution neural network method which was trained to handle a wide range of noise levels. CBM3D \cite{dabov2007image} applies the classical BM3D algorithm to the denoising of color images and its working principle is to convert the color image from the RGB space into a less correlated luminance-chrominance space. For the comparison methods, the involved parameters are all manually-tuned or automatically set up according to the reference papers to achieve the best performance.
In our proposed algorithm, parameters that need to be considered include the patch size $w$, similar patch number $M$, search window $W$, the power $p$ value in the weighted Schatten $p$ norm regularizer, the iteration number $K$, the parameters $c$ and $\epsilon$ in the weight $w_{i}=\frac{c}{\sigma_{i}(\mathbf{\dot{X}})+\epsilon}$. For $\sigma\leq 40$, we set the search window $W$ as $30$. When $0<\sigma \leq 20$, we set the similar patch number $M$, the patch size $w\times w$ and the iteration number $K$ as $70$, $4\times4$, and $8$, respectively. When $20<\sigma \leq 40$, we set the similar number $M$, the patch size $w\times w$ and the iteration number $K$ as $90$, $5\times5$, and $12$, respectively. For $40<\sigma\leq 50$, we set the search window $W$, the similar patch number $M$, the patch size $w\times w$ and the iteration number $K$ as $40$, $120$, $5\times5$, and $14$, respectively. Throughout the denoising experiments, we fix the power $p$ value and the regularization parameter $\lambda$ as $0.95$ and $1$ for simplicity. In addition, the parameters $c$ and $\epsilon$ are set as $\sqrt{2}$ and \verb"eps" (a Matlab built-in function) respectively, which is as the same as the parameter setting in the WNNM \cite{gu2014weighted}.

Table 2 reports the PSNR and SSIM values of each image as well as the average results of the competing approaches under different noise levels. The best results are highlighted in bold and the second-best results are underlined. One can observe that our proposed QWSNM algorithm achieves the best objective performance in most cases. For the PSNR metric, our method achieves $0.32$dB, $0.28$dB, $0.26$dB and $0.18$dB improvement over the state-of-the-art approach IRCNN \cite{zhang2017learning} on average with respect to $\sigma=15,25,35,45$. The main reason behind this phenomenon is that we adopt the quaternion representation, which can better preserve the inter-relationship between RGB channels. Compared with the most similar quaternion-based approach QWNNM \cite{yu2019quaternion}, our method achieves slight quantitative gain of $0.07$dB, $0.10$dB, $0.09$dB and $0.14$dB on average with respect to $\sigma=15,25,35,45$. A reasonable explanation is that QWNNM \cite{yu2019quaternion} tends to over-shrink the dominant rank components, which constraints its capacity of recovering the images with rich textures. The above quantitative results fully demonstrate the effectiveness of exploiting both the quaternion representation and WSNM regularization jointly, as in our proposed QWSNM approach. Human subject perception is the ultimate judge of image quality, which also plays a crucial role in the estimation of denoising algorithms. The visual quality comparisons for the test image ``Baboon'' are shown in Fig. 3. As can be observed from the highlighted window, SV-TV method produces agonizing visual artifacts, while CBM3D, IRCNN and pQSTV approaches are pone to over-smooth the image, losing some of the tiny details. Although QWNNM method overcomes the over-smooth phenomenon, it generates unpleasant visual artifacts. In contrast, our proposed QWSNM approach is not only capable of achieving a better visual perception of both texture and detail information, but also significantly eliminating visual artifacts.  


\subsection{Color Image Deblurring}
In this subsection, we conduct extensive color image deblurring experiments to demonstrate the superiority of our proposed QWSNM. The test images are selected from Set12 and Set27\footnote{The test images of Set27 are downloaded from \url{https://github.com/Huang-chao-yan/dataset}. Image 1-7 and 16 are with the size of $500\times500$; Image 8 is with the size of $560\times392$; Image 9 and 22-25 are with the size of $481\times321$; Image 10 and 21 are with the size of $1024\times683$; Image 11 is with the size of $1024\times764$; Image 12 is with the size of $321\times481$; Image 13 is with the size of $1024\times701$; Image 14 is with the size of $1024\times783$; Image 15 is with the size of $1024\times684$; Image 17 is with the size of $1024\times605$; Image 18 is with the size of $448\times296$; Image 19 is with the size of $256\times256$; Image 20 is with the size of $288\times288$; Image 26 is with the size of $418\times378$; Image 27 is with the size of $1024\times937$.}, which are shown in Figs. 2 and 4, respectively.  Three sets of experiments are conducted for simulated color image deblurring. The first is a $9\times9$ uniform kernel, the second is a Gaussian kernel with standard variation 1.6, and the third is a motion kernel with motion length 20 and motion angle 60. Then, additive Gaussian noise with $\sigma=15$ is added to the blurred images. The involved parameters of our algorithm are set as follows: $\lambda=65$, $\beta=7.5$ for Gaussian blur with noise; $\lambda=115$, $\beta=7.5$ for motion blur with noise; $\lambda=115$, $\beta=8.5$ for average blur with noise. Particularly, we fix the search window size $W$, the similar number $M$ and the patch size $w$ as 30, 155 and 6$\times$6, respectively. Moreover, we study the influence of the parameters $c$ and $\epsilon$ in the weight $w_{i}=\frac{c}{\sigma_{i}(\mathbf{\dot{X}})+\epsilon}$. Experiments on the test image ``Bird'' with motion blur and noise are conducted, where $c$ is set to $[0.1:5]*\sqrt{2}$ with an increment of 0.2 and $\epsilon$ is set to $[1,250000]*$\verb"eps" with step size 10000. As shown in Fig. 4, the parameter $\epsilon$ does not dramatically affect the results and we set $\epsilon=$ \verb"eps" for all our deblurring experiments. The adjustment of parameter $c$ has a greater impact on the results and we set $c=2.2*\sqrt{2}$ for all deblurring experiments through trial and error.

We compare the proposed QWSNM with five representative methods including SV-TV \cite{jia2019color}, IDD-BM3D \cite{danielyan2011bm3d}, IRCNN \cite{zhang2017learning}, MSWNNM \cite{yair2018multi} and QWNNM* \cite{huang2022quaternion}. The PSNR and SSIM values of compared deblurring methods with average blur AB(9,9)/$\sigma=15$ are shown in Table 3. The corresponding results with motion blur MB(20,60)/$\sigma=15$ and with Gaussian blur GB(25,1.6)/$\sigma=15$ are reported in Tables 4 and 5, respectively. The best results are highlighted in bold and the second-best results are underlined. As expected, it can be observed that the proposed QWSNM approach consistently generates much better deblurring results on average on different types of blur. Moreover, we display the restored images to demonstrate the perceptual superiority of the proposed method. Figs. 6, 7, 8 show the visual results of these competing methods. For Img26, we can clearly observe that SV-TV, IDD-BM3D, IRCNN and MSWNNM are apt to over-smooth the text image, leading to the characters unrecognized. Although QWNNM approach achieves better visual results than the above four methods, it still tends to generate the over-smooth phenomenon. By contrast, the text image recovered by our proposed QWSNM is sharper and the characters can be better recognized. For Img1, it is evident that SV-TV method generates color distortions and color artifacts, while IDD-BM3D, IRCNN and MSWNNM approaches are prone to over-smooth the image, losing many vital details. Although QWNNM method can prevent the over-smooth phenomenon to some extent, benefiting from the advantages of quaternion representation, it is still not able to recover some of the texture structures. Compared with these competing methods, our proposed QWSNM achieves the best visual results on the whole, which not only preserves more texture structures in the wallpaper of the window, but also effectively eliminates the undesirable color distortions.          


\begin{figure}
    \centering
    \includegraphics[width=0.95\textwidth]{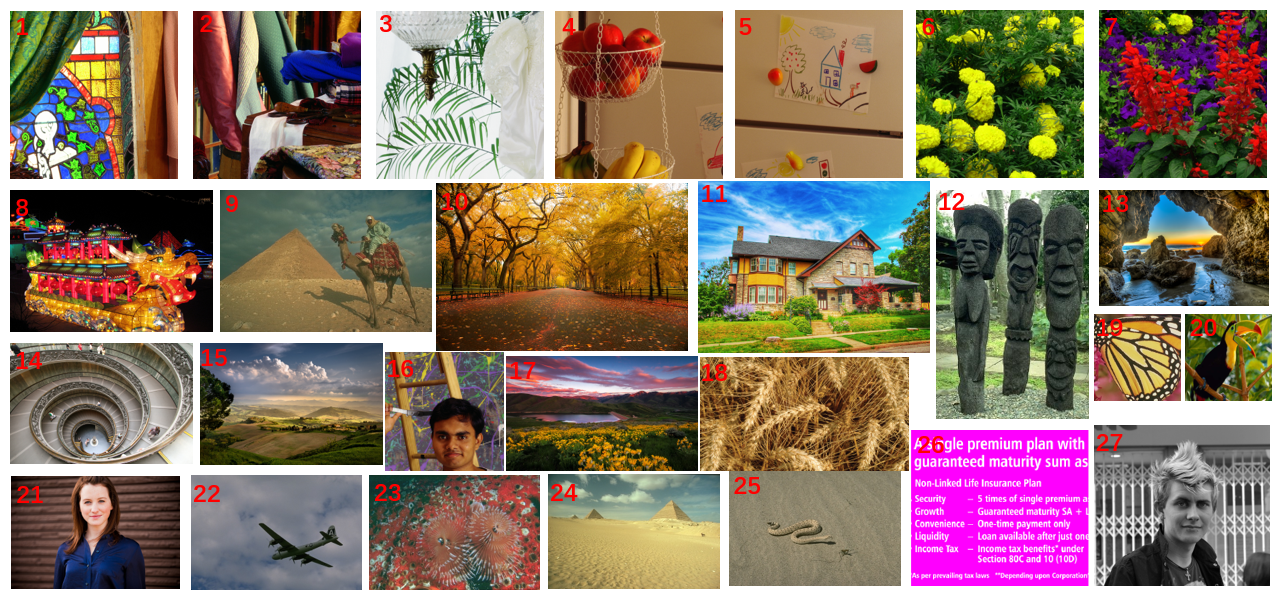}
    \caption{Images in Set27 \cite{huang2022quaternion}.}
\end{figure}

\begin{figure}[htbp]
    \centering
    \includegraphics[width=0.80\textwidth]{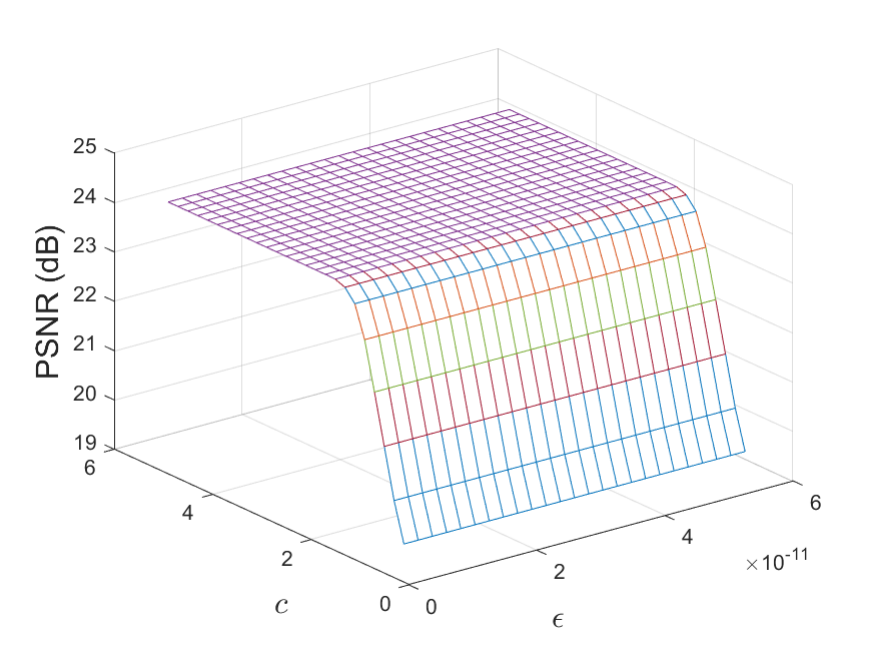}
    \caption{Evolution of PSNR values with different parameters $c$ and $\epsilon$ in the case of motion deblurring. Test image: ``Bird''.}
\end{figure}

\begin{figure}[htbp]
\centering
\subfigure[Img26]{\includegraphics[width=0.30\textwidth]{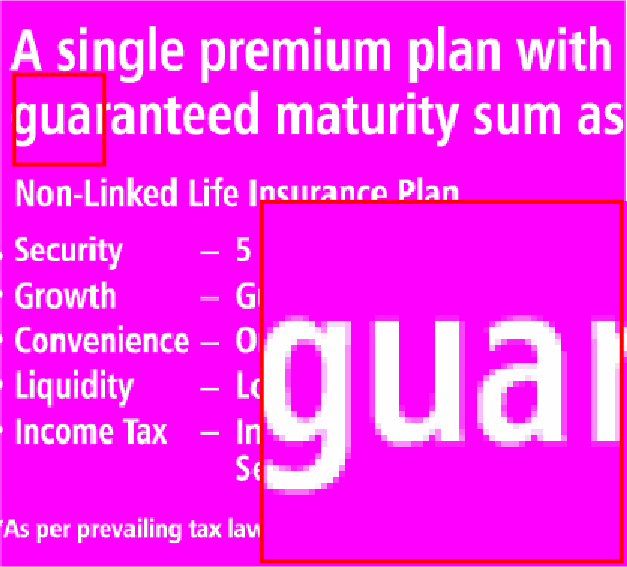}}
\subfigure[(15.55/0.8716)]{\includegraphics[width=0.30\textwidth]{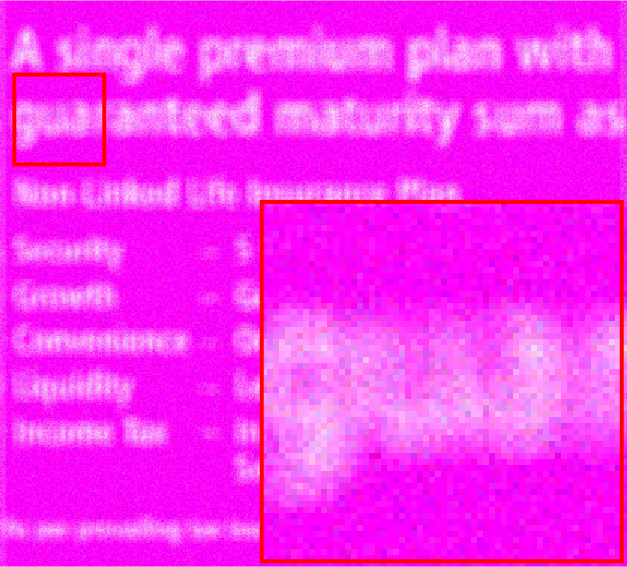}}
\subfigure[(17.23/0.9039)]{\includegraphics[width=0.30\textwidth]{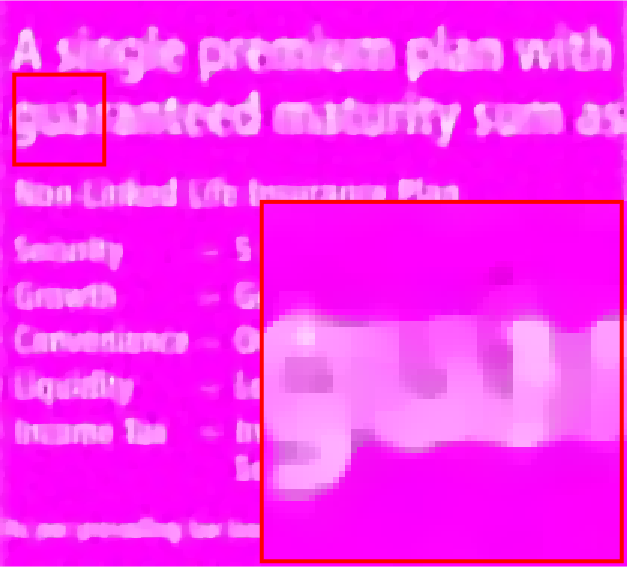}}
\subfigure[(18.23/0.9241)]{\includegraphics[width=0.30\textwidth]{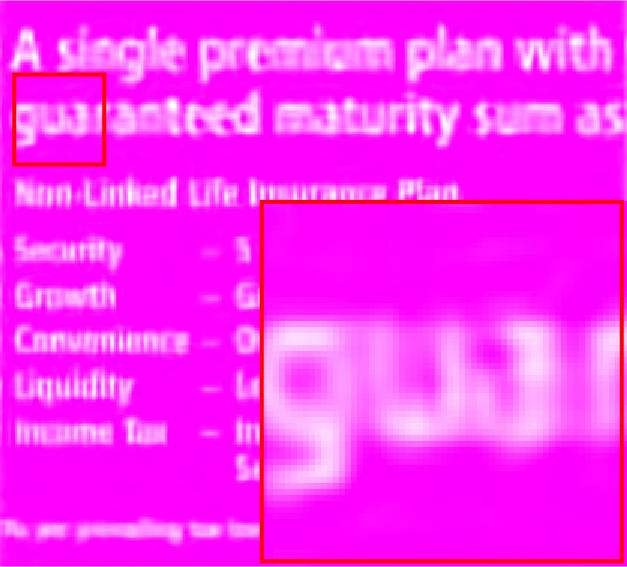}}
\subfigure[(19.31/0.9480)]{\includegraphics[width=0.30\textwidth]{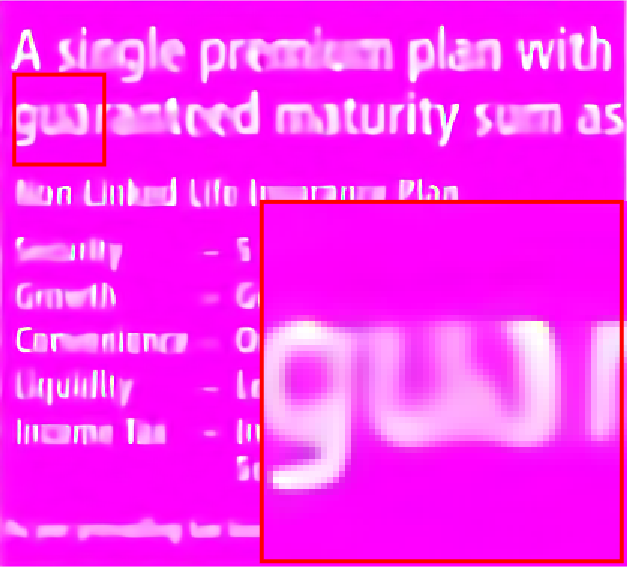}}
\subfigure[(19.03/0.9380)]{\includegraphics[width=0.30\textwidth]{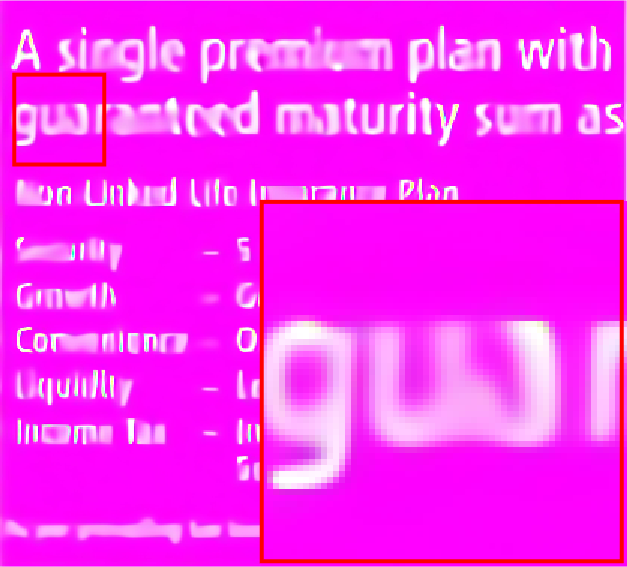}}
\subfigure[(20.48/0.9544)]{\includegraphics[width=0.30\textwidth]{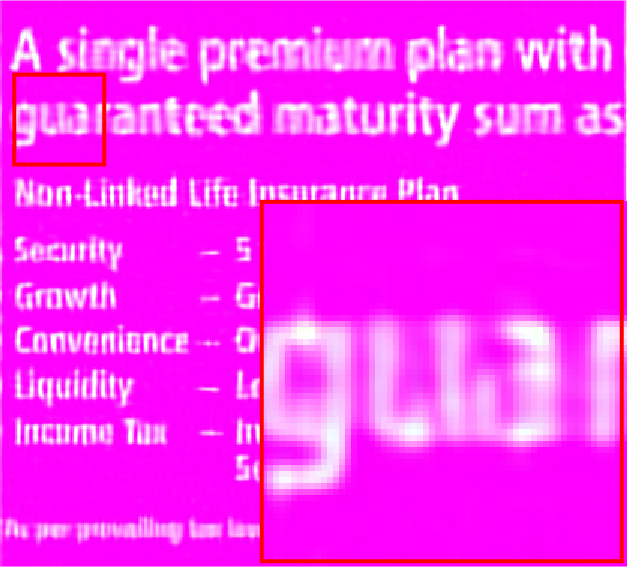}}
\subfigure[(21.06/0.9610)]{\includegraphics[width=0.30\textwidth]{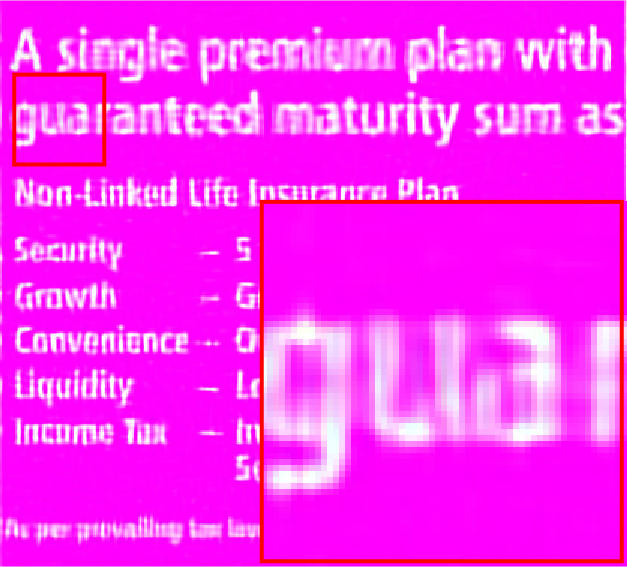}}
\caption{Deblurring performance on ``Img26'' with visual quality and numerical results (PSNR/SSIM). (a) Ground truth; (b) Degraded image with average kernel (9,9) and Gaussian noise level
$\sigma = 15$; Restored image by (c) SV-TV \cite{jia2019color}; (d) IDD-BM3D \cite{danielyan2011bm3d}; (e) IRCNN \cite{zhang2017learning}; (f) MSWNNM \cite{yair2018multi}; (g) QWNNM* \cite{huang2022quaternion}; (h) QWSNM (Ours).}
\end{figure}

\begin{figure}[htbp]
\centering
\subfigure[Plane]{\includegraphics[width=0.30\textwidth]{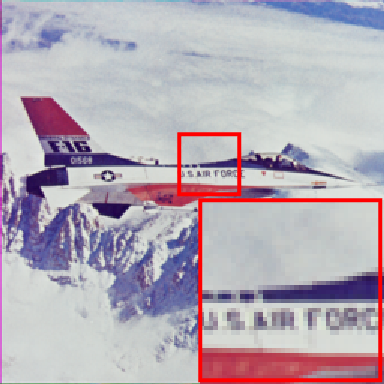}}
\subfigure[(21.88/0.4406)]{\includegraphics[width=0.30\textwidth]{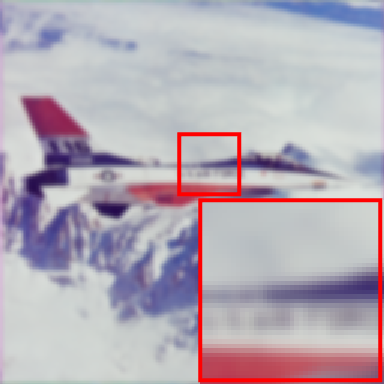}}
\subfigure[(24.52/0.7808)]{\includegraphics[width=0.30\textwidth]{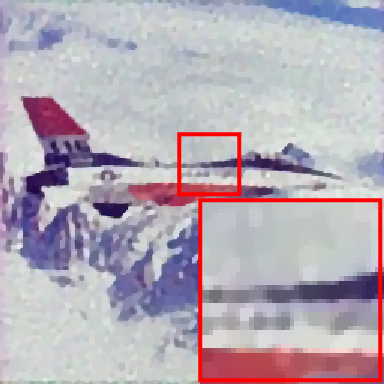}}
\subfigure[(26.04/0.8527)]{\includegraphics[width=0.30\textwidth]{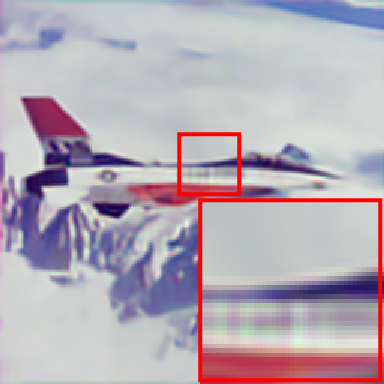}}
\subfigure[(27.04/0.8561)]{\includegraphics[width=0.30\textwidth]{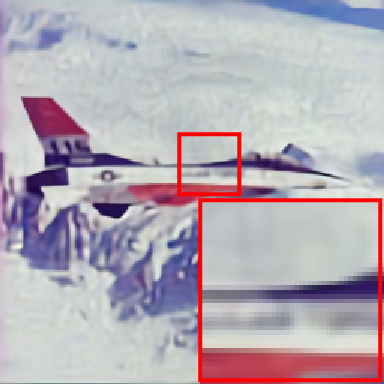}}
\subfigure[(27.18/0.8744)]{\includegraphics[width=0.30\textwidth]{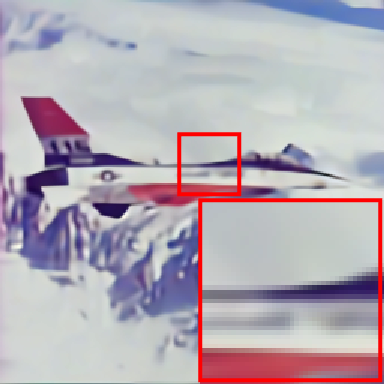}}
\subfigure[(27.19/0.8776)]{\includegraphics[width=0.30\textwidth]{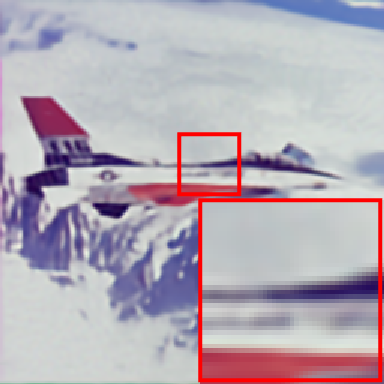}}
\subfigure[(27.30/0.8895)]{\includegraphics[width=0.30\textwidth]{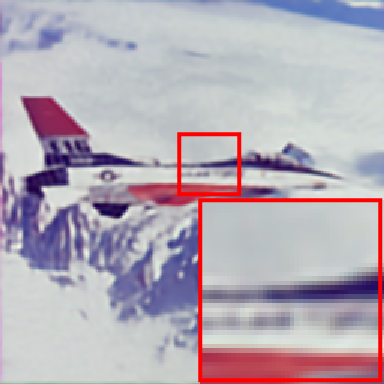}}
\caption{Deblurring performance on ``Plane'' with visual quality and numerical results (PSNR/SSIM). (a) Ground truth; (b) Degraded image with Gaussian kernel (25,1.6) and Gaussian noise level $\sigma =15$; Restored image by  (c) SV-TV \cite{jia2019color}; (d) IDD-BM3D \cite{danielyan2011bm3d}; (e) IRCNN \cite{zhang2017learning}; (f) MSWNNM \cite{yair2018multi}; (g) QWNNM* \cite{huang2022quaternion}; (h) QWSNM (Ours).}
\end{figure}

\begin{figure}[htbp]
\centering
\subfigure[Img1]{\includegraphics[width=0.30\textwidth]{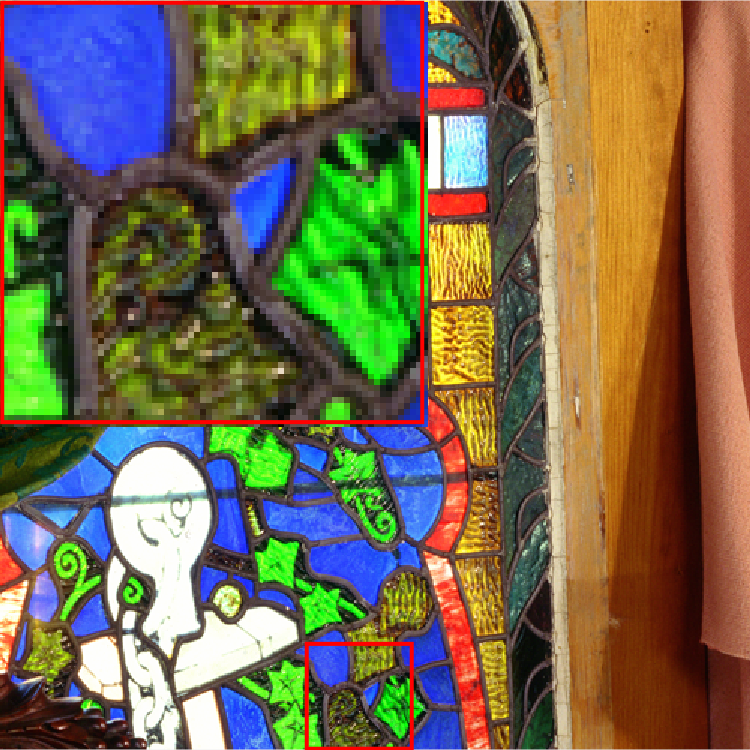}}
\subfigure[(17.47/0.6858)]{\includegraphics[width=0.30\textwidth]{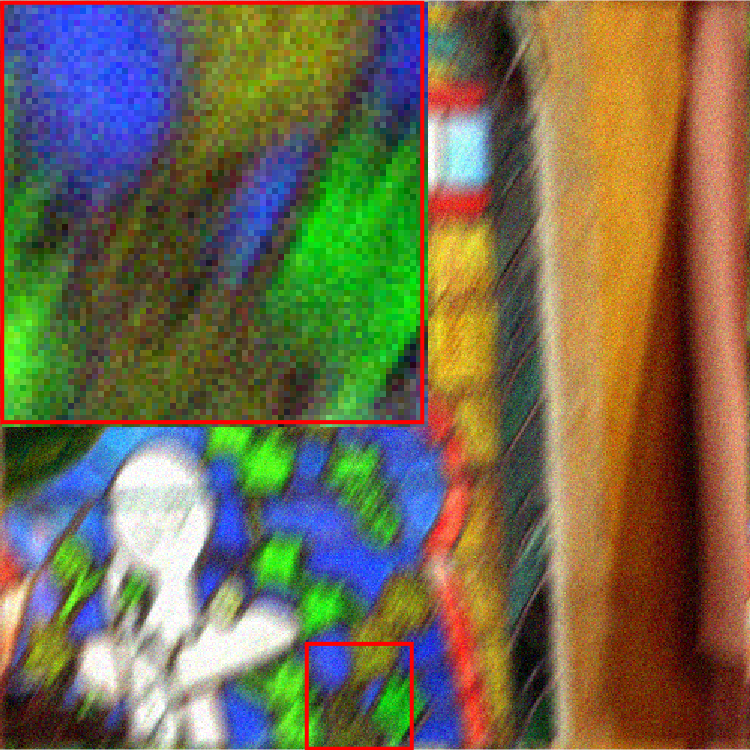}}
\subfigure[(20.89/0.8342)]{\includegraphics[width=0.30\textwidth]{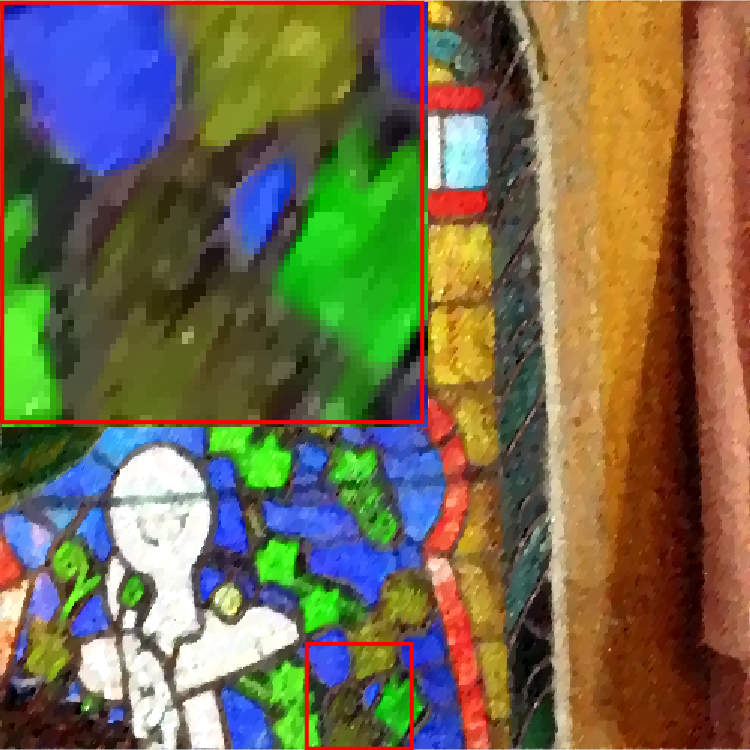}}
\subfigure[(20.66/0.8272)]{\includegraphics[width=0.30\textwidth]{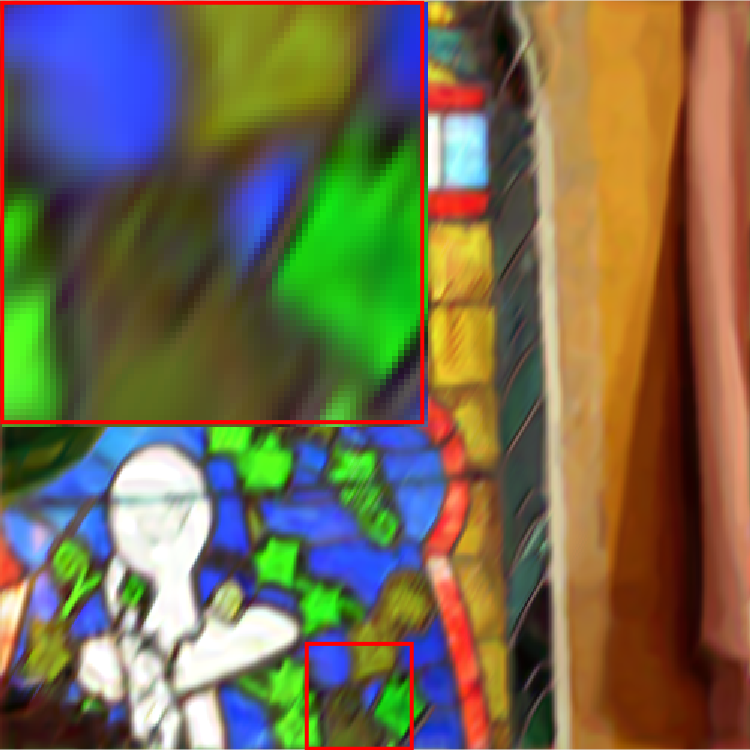}}
\subfigure[(22.58/0.8764)]{\includegraphics[width=0.30\textwidth]{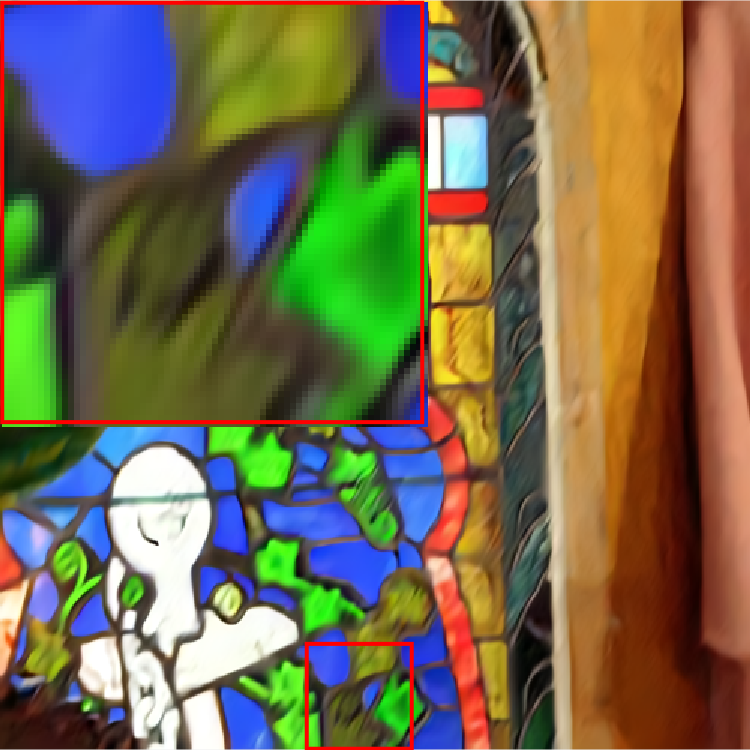}}
\subfigure[(22.56/0.8752)]{\includegraphics[width=0.30\textwidth]{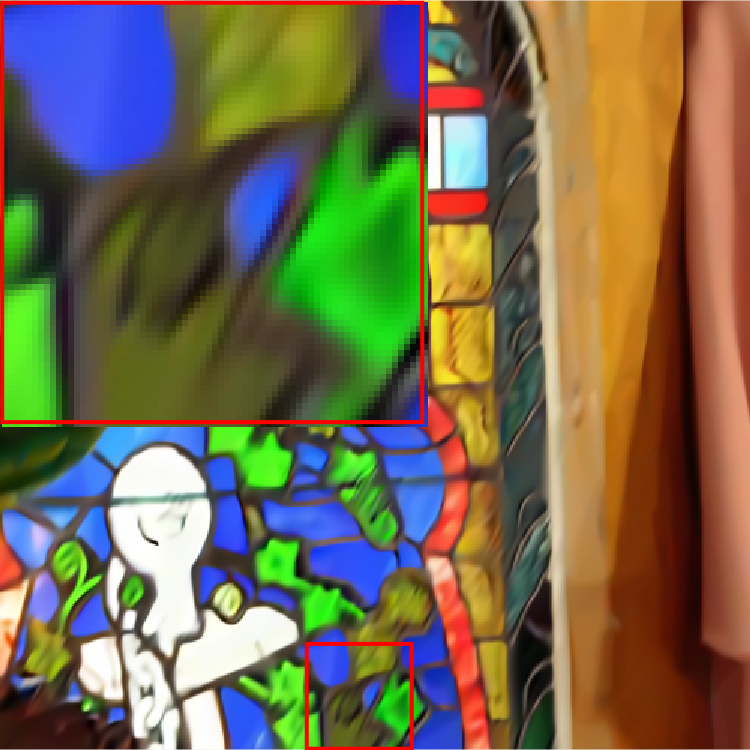}}
\subfigure[(22.56/0.8728)]{\includegraphics[width=0.30\textwidth]{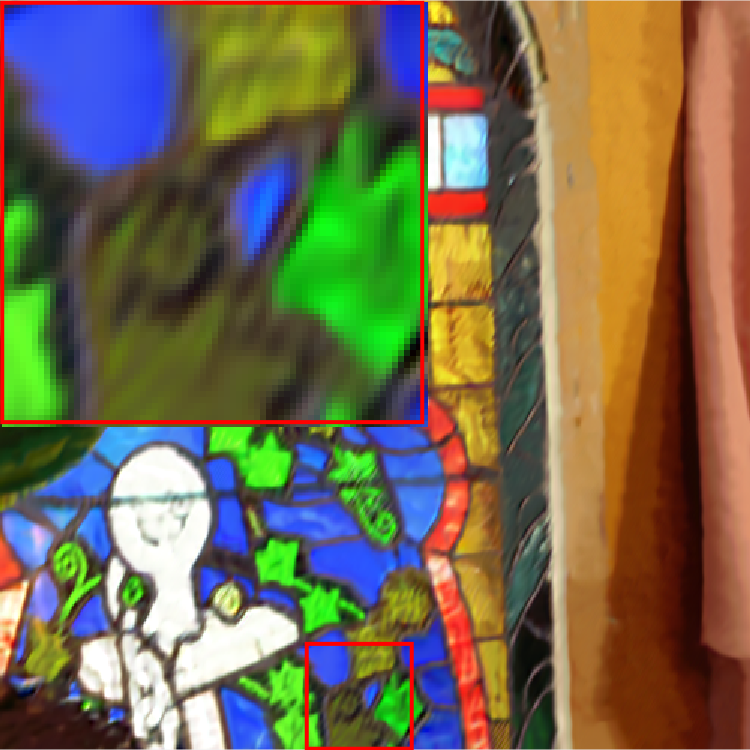}}
\subfigure[(22.92/0.8810)]{\includegraphics[width=0.30\textwidth]{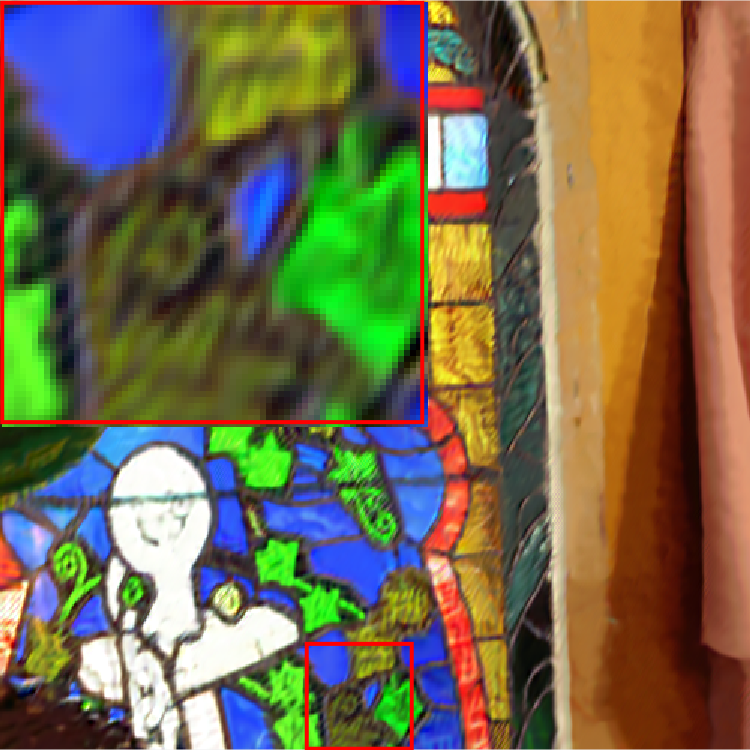}}
\caption{Deblurring performance on ``Img1'' with visual quality and numerical results (PSNR/SSIM). (a) Ground truth; (b) Degraded image with motion kernel (20,60) and Gaussian noise level $\sigma =15$; Restored image by  (c) SV-TV \cite{jia2019color}; (d) IDD-BM3D \cite{danielyan2011bm3d}; (e) IRCNN \cite{zhang2017learning}; (f) MSWNNM \cite{yair2018multi}; (g) QWNNM* \cite{huang2022quaternion}; (h) QWSNM (Ours).}
\end{figure}

\begin{table*}[htbp]
\begin{center}{ \tiny
\resizebox{\textwidth}{80mm}{
\begin{tabular}{|c||c|c|c|c|c|c|}
\hline
\ Methods & SV-TV \cite{jia2019color} & IDD-BM3D \cite{danielyan2011bm3d} & IRCNN \cite{zhang2017learning} & MSWNNM \cite{yair2018multi} & QWNNM* \cite{huang2022quaternion} & QWSNM (Ours) \\
\hline
\multicolumn{7}{|c|}{Set27}\\
\hline
Img1 & 21.50/0.8369 & 21.47/0.8348 & 22.68/0.8683 & 22.66/0.8672 & \underline{22.83/0.8696} & \textbf{23.17/0.8764} \\
Img2 & 25.59/0.8497 & 25.02/0.8460 & \underline{26.53/0.8714} & 26.52/0.8712 & 26.45/0.8712 & \textbf{26.64/0.8724} \\
Img3 & 21.85/0.7563 & 20.92/0.7365 & 23.52/0.8177 & 23.60/\underline{0.8221} & \underline{23.95}/0.8189 & \textbf{24.41/0.8350} \\
Img4 & 26.31/0.9564 & 26.07/0.9546 & 27.42/0.9653 & 27.45/0.9657 & \underline{27.68/0.9671} & \textbf{27.81/0.9680} \\
Img5 & 29.92/0.9710 & 29.89/0.9716 & 31.58/0.9806 & 31.68/0.9808 & \underline{31.85/0.9815} & \textbf{32.01/0.9823} \\
Img6 & 23.00/0.8640 & 22.92/0.8619 & 23.78/0.8791 & 23.73/0.8792 & \underline{24.07/0.8865} & \textbf{24.14/0.8870} \\
Img7 & 23.76/0.8541 & 24.36/0.8699 & 24.99/0.8845 & 24.93/0.8839 & \underline{25.25/0.8914} & \textbf{25.34/0.8931} \\
Img8 & 21.03/0.7982 & 21.25/0.8430 & 22.18/0.8658 & 22.13/0.8398 & \underline{22.53/0.8686} & \textbf{22.81/0.8694} \\
Img9 & 28.71/0.9310 & 28.61/0.9343 & 28.92/0.9385 & 28.91/0.9384 & \underline{28.98/0.9391} & \textbf{29.00/0.9394} \\
Img10 & 19.65/0.8445 & 19.59/0.8439 & 20.30/0.8648 & 20.26/0.8639 & \underline{20.57/0.8723} & \textbf{20.63/0.8743} \\
Img11 & 20.75/0.8124 & 20.59/0.8071 & 21.63/0.8414 & 21.58/0.8386 & \underline{21.96/0.8502} & \textbf{22.06/0.8536} \\
Img12 & 21.77/0.5826 & 21.16/0.5424 & \underline{21.89/0.5992} & 21.83/0.5919 & 21.88/0.5960 & \textbf{21.92/0.5995} \\
Img13 & 23.59/0.7909 & 23.27/0.7856 & 24.43/0.8228 & 24.38/0.8169 & \underline{24.65/0.8267} & \textbf{24.71/0.8324} \\
Img14 & 22.53/0.7253 & 22.09/0.7085 & 23.66/0.7786 & 23.68/0.7807 & \underline{24.02/0.7885} & \textbf{24.30/0.7978} \\
Img15 & 27.86/0.9044 & 27.54/0.9051 & 28.60/0.9152 & 28.58/0.9150 & \underline{28.72/0.9173} & \textbf{28.74/0.9176} \\
Img16 & 26.01/0.8077 & 26.58/0.8237 & 27.46/0.8419 & 27.46/0.8402 & \underline{28.08/0.8604} & \textbf{28.18/0.8606} \\
Img17 & 23.36/0.8375 & 23.36/0.8375 & 24.15/08592 & 24.11/0.8573 & \underline{24.37/0.8589} & \textbf{24.47/0.8621} \\
Img18 & 18.97/0.8045 & 18.78/0.7964 & 19.21/0.8166 & 19.15/0.8136 & \underline{19.42/0.8246} & \textbf{19.48/0.8256} \\
Img19 & 22.06/0.8944 & 20.87/0.8584 & 22.60/0.9038 & 22.72/0.9065 & \underline{23.05/0.9069} & \textbf{23.43/0.9191} \\
Img20 & 26.07/0.9135 & 26.63/0.9347 & 27.83/0.9446 & \underline{27.88/0.9452} & 27.71/0.9431 & \textbf{27.94/0.9475} \\
Img21 & 32.12/0.9536 & 32.52/0.9663 & \underline{33.71/0.9677} & \textbf{33.73/0.9700} & 33.24/0.9651 & 33.30/0.9656 \\
Img22 & 33.60/0.9458 & 32.61/0.9524 & \underline{34.06/0.9607} & \textbf{34.26/0.9619} & 33.78/0.9526 & 34.04/0.9544 \\
Img23 & 23.43/0.8129 & 23.28/0.8104 & 23.88/0.8333 & 23.80/0.8292 & \underline{24.26/0.8466} & \textbf{24.30/0.8482} \\
Img24 & 29.27/0.9563 & 29.14/0.9590 & \textbf{29.69/0.9645} & \underline{29.65/0.9642} & 29.55/0.9619 & 29.64/0.9627 \\
Img25 & 26.00/0.7769 & 25.83/0.7733 & \textbf{26.17/0.7843} & 26.13/0.7810 & 26.13/0.7807 & \underline{26.15/0.7814} \\
Img26 & 17.23/0.9039 & 18.23/0.9241 & 19.31/0.9480 & 19.03/0.9380 & \underline{20.48/0.9544} & \textbf{21.06/0.9610} \\
Img27 & 29.13/0.8344 & 29.12/0.8608 & \textbf{31.14/0.8848} & \underline{31.10/0.8835} & 30.67/0.8739 & 30.84/0.8744 \\
Average & 24.63/0.8487  & 24.51/0.8490 & 25.60/0.8743 & 25.59/0.8684 & \underline{25.76/0.8762} & \textbf{25.94/0.8799} \\
\hline
\multicolumn{7}{|c|}{Set12}\\
\hline
Bird & 23.01/0.7037 & 22.77/0.6913 & 23.33/0.7240 & 23.27/0.7191 & \underline{23.49/0.7316} & \textbf{23.55/0.7361} \\
Plane & 24.24/0.7709 & 23.80/0.7837 & 25.00/0.8220 & 24.98/0.8222 & \underline{25.12/0.8265} & \textbf{25.22/0.8295} \\
Baboon & 22.02/0.6954 & 21.97/0.6790 & \underline{22.38/0.7004} & 22.31/0.6945 & 22.37/0.6984 & \textbf{22.46/0.7023} \\
Bee & 27.03/0.8886 & 28.11/0.9103 & 28.67/0.9129 & 28.75/0.9137 & \underline{29.09/0.9174} & \textbf{29.20/0.9186} \\
Aquatic & 23.26/0.7307 & 22.89/0.7308 & 24.06/0.7804 & 24.03/0.7824 & 24.11/0.7861 & \textbf{24.28/0.7874} \\
Barbara & 22.86/0.7804 & 22.66/0.7731 & 23.60/\underline{0.8070} & 23.60/0.8068 & \underline{23.62}/0.8040 & \textbf{23.79/0.8087}\\
Boat & 22.38/0.7679 & 21.71/0.7555 & \underline{22.98/0.8019} & 22.93/0.7977 & 22.96/0.8006 & \textbf{23.17/0.8042} \\
House & 26.34/0.9085 & 26.64/0.9111 & 28.04/0.9335 & 28.15/0.9342 & \underline{28.61/0.9360} & \textbf{28.95/0.9395} \\
Peppers & 22.57/0.9251 & 23.73/0.9298 & 24.63/0.9413 & \underline{24.65/0.9415} & 24.64/0.9405 & \textbf{24.80/0.9440} \\
Starfish & 23.15/0.9207 & 23.02/0.9251 & 23.86/0.9354 & 23.83/0.9354 & \underline{24.03/0.9365} & \textbf{24.15/0.9367} \\
Lena & 27.55/0.9575 & 27.29/0.9560 & \underline{28.41/0.9643} & 28.39/0.9639 & 28.37/0.9640 & \textbf{28.60/0.9654} \\
Pelican & 24.58/0.8221 & 23.93/0.8122 & 24.85/\textbf{0.8402} & 24.75/\underline{0.8394} & \underline{24.92}/0.8339 & \textbf{25.04}/0.8362 \\
Average & 24.08/0.8226 & 24.04/0.8215 & 24.98/0.8469 & 24.97/0.8459 & \underline{25.11/0.8471} & \textbf{25.27/0.8507} \\
\hline
\end{tabular}}
}
\caption{PSNR (dB) and SSIM values of different restoration models for AB(9,9)/$\sigma=15$.
The best results are highlighted in \textbf{bold} and the second-best results are \underline{underlined}.}
\end{center}
\end{table*}

\begin{table*}[htbp]
\begin{center}{ \tiny
\resizebox{\textwidth}{80mm}{
\begin{tabular}{|c||c|c|c|c|c|c|}
\hline
\ Methods & SV-TV \cite{jia2019color} & IDD-BM3D \cite{danielyan2011bm3d} & IRCNN \cite{zhang2017learning} & MSWNNM \cite{yair2018multi} & QWNNM* \cite{huang2022quaternion} & QWSNM (Ours) \\
\hline
\multicolumn{7}{|c|}{Set27}\\
\hline
Img1 & 23.28/0.8822 & 23.46/0.8822 & 24.20/0.9052 & 24.19/0.9056 & \underline{24.46/0.9075} & \textbf{24.50/0.9089} \\
Img2 & 27.16/0.8803 & 27.02/0.8871 & 28.19/0.8988 & 28.30/0.9032 & \underline{28.32/0.9041} & \textbf{28.38/0.9063} \\
Img3 & 25.12/0.8501 & 24.53/0.8442 & \underline{26.84}/0.8804 & 26.78/0.8667 & 26.80/\underline{0.8930} & \textbf{27.15/0.8950} \\
Img4 & 27.99/0.9690 & 28.20/0.9702 & 29.34/0.9764 & \underline{29.74/0.9787} & 29.68/0.9785 & \textbf{29.79/0.9791} \\
Img5 & 31.61/0.9796 & 32.58/0.9841 & 33.18/0.9861 & 33.52/0.9870 & \underline{33.90/0.9879} & \textbf{33.96/0.9881} \\
Img6 & 25.03/0.9077 & 25.00/0.9047 & 26.17/0.9228 & 26.18/0.9248 & \underline{26.24/0.9250} & \textbf{26.31/0.9261} \\
Img7 & 26.04/0.9074 & 26.73/0.9171 & 27.57/0.9303 & 27.52/0.9311 & \underline{27.71/0.9339} & \textbf{27.76/0.9350} \\
Img8 & 22.71/0.8209 & 23.26/0.8923 & 23.94/0.8764 & 23.93/0.7691 & \underline{24.27/0.8975} & \textbf{24.37/0.8992} \\
Img9 & 29.57/0.9378 & 29.50/0.9455 & 29.69/0.9441 & 29.77/0.9452 & \underline{29.90/0.9493} & \textbf{29.96/0.9500}\\
Img10 & 20.64/0.8724 & 20.70/0.8749 & 21.58/0.8963 & 21.57/0.8960 & \underline{21.72/0.8998} & \textbf{21.77/0.9013} \\
Img11 & 21.78/0.8459 & 21.79/0.8423 & 22.77/0.8746 & 22.78/0.8741 & \underline{22.94/0.8750} & \textbf{23.00/0.8773} \\
Img12 & 22.81/0.6606 & 22.12/0.6088 & \textbf{22.95/0.6760} & \underline{22.93/0.6698} & 22.85/0.6606 & 22.90/0.6653 \\
Img13 & 25.39/0.8491 & 24.90/0.8347 & 26.40/0.8760 & 26.35/0.8750 & \underline{26.52/0.8776} & \textbf{26.56/0.8786} \\
Img14 & 24.69/0.8160 & 24.19/0.7977 & 25.89/0.8449 & 25.83/0.8432 & \underline{26.14/0.8604} & \textbf{26.25/0.8619} \\
Img15 & 29.25/0.9241 & 29.01/0.9225 & 29.85/0.9293 & 29.80/0.9279 & \underline{30.13/0.9373} & \textbf{30.19/0.9378} \\
Img16 & 28.16/0.8641 & 29.37/0.8897 & 30.12/0.8979 & 30.10/0.8965 & \underline{30.76/0.9108} & \textbf{30.80/0.9114} \\
Img17 & 24.51/0.8674 & 24.67/0.8682 & 25.45/0.8806 & 25.40/0.8785 & \underline{25.64/0.8869} & \textbf{25.67/0.8875} \\
Img18 & 20.44/0.8585 & 20.28/0.8523 & 20.90/0.8747 & 20.86/0.8722 & \underline{21.01/0.8776} & \textbf{21.06/0.8791} \\
Img19 & 24.85/0.9386 & 23.93/0.9248 & \textbf{25.70/0.9494} & \underline{25.69/0.9489} & 25.27/0.9444 & 25.44/0.9459 \\
Img20 & 28.10/0.9310 & 29.60/0.9626 & 30.58/0.9597 & 30.56/0.9588 & \underline{30.84/0.9676} & \textbf{30.87/0.9688} \\
Img21 & 33.50/0.9624 & 34.15/0.9743 & 34.37/0.9684 & 34.35/0.9678 & \textbf{35.01/0.9757} & \underline{34.90/0.9743} \\
Img22 & 34.47/0.9306 & 34.81/0.9658 & 35.33/0.9439 & 35.45/0.9500 & \textbf{35.84/0.9635} & \underline{35.76/0.9600} \\
Img23 & 25.55/0.8834 & 25.44/0.8784 & 26.30/0.9022 & 26.38/0.9025 & \underline{26.60/0.9065} & \textbf{26.62/0.9071} \\
Img24 & 30.06/0.9603 & 29.85/0.9655 & 30.41/0.9667 & 30.45/0.9672 & \underline{30.51/0.9698} & \textbf{30.54/0.9701} \\
Img25 & 26.42/0.7889 & 26.25/0.7826 & \underline{26.50/0.7918} & 26.39/0.7912 & 26.48/0.7916 & \textbf{26.52/0.7920} \\
Img26 & 19.56/0.9405 & 21.99/0.9689 & 22.84/0.9740 & 22.69/0.9658 & \underline{23.35/0.9777} & \textbf{23.60/0.9791} \\
Img27 & 31.18/0.8735 & 31.12/0.8925 & 32.77/0.8883 & 32.70/0.8856 & \underline{32.89/0.9084} & \textbf{32.96/0.9098} \\
Average & 26.29/0.8853  & 26.46/0.8903 & 27.40/0.9044 & 27.43/0.9007 & \underline{27.62/0.9099} & \textbf{27.69/0.9109} \\
\hline
\multicolumn{7}{|c|}{Set12}\\
\hline
Bird & 24.52/0.7808 & 24.47/0.7736 & 25.12/0.8080 & 25.14/0.8084 & \underline{25.20/0.8093} & \textbf{25.27/0.8130} \\
Plane & 25.90/0.8014 & 26.04/0.8527 & 27.04/0.8561 & 27.18/0.8744 & \underline{27.19/0.8776} & \textbf{27.30/0.8895} \\
Baboon & 22.97/0.7466 & 22.98/0.7282 & 23.42/\textbf{0.7617} & 23.36/0.7533 & \underline{23.43}/0.7516 & \textbf{23.48}/\underline{0.7572} \\
Bee & 28.66/0.9039 & 30.26/0.9308 & 30.65/0.9293 & 30.98/0.9363 & \underline{31.32/0.9391} & \textbf{31.42/0.9455} \\
Aquatic & 25.09/0.7961 & 25.37/0.8271 & 26.30/0.8486 & 26.42/0.8638 & \underline{26.51/0.8669} & \textbf{26.58/0.8702} \\
Barbara & 24.15/0.8275 & 24.27/0.8301 & 24.91/0.8530 & \underline{24.94/0.8543} & 24.92/0.8517 & \textbf{25.00/0.8585}\\
Boat & 23.68/0.8086 & 23.62/0.8365 & \underline{24.71}/0.8501 & \textbf{24.74}/\underline{0.8534} & 24.55/0.8504 & 24.62/\textbf{0.8585} \\
House & 27.32/0.9242 & 28.61/0.9402 & 29.10/0.9456 & 29.12/0.9460 & \underline{29.82/0.9496} & \textbf{29.90/0.9500} \\
Peppers & 25.55/0.9501 & 25.94/0.9561 & \underline{26.65/0.9622} & 26.64/0.9620 & 26.57/0.9615 & \textbf{26.69/0.9643} \\
Starfish & 25.15/0.9388 & 25.47/0.9487 & 26.60/0.9557 & 26.65/0.9560 & \underline{26.67/0.9596} & \textbf{26.77/0.9654} \\
Lena & 29.06/0.9684 & 29.27/0.9691 & 30.27/0.9754 & 30.24/0.9750 & \underline{30.48/0.9761} & \textbf{30.55/0.9800} \\
Pelican & 25.70/0.8403 & 25.15/0.8412 & \underline{26.18/0.8610} & 26.16/0.8605 & 26.13/0.8602 & \textbf{26.26/0.8701} \\
Average & 25.65/0.8572 & 25.95/0.8687 & 26.75/0.8839 & 26.80/0.8870 & \underline{26.90/0.8878} & \textbf{26.98/0.8931} \\
\hline
\end{tabular}}
}
\caption{PSNR (dB) and SSIM values of different restoration models for GB(25,1.6)/$\sigma=15$. The best results are highlighted in \textbf{bold} and the second-best results are \underline{underlined}.}
\end{center}
\end{table*}

\begin{table*}[htbp]
\begin{center}{ \tiny
\resizebox{\textwidth}{80mm}{
\begin{tabular}{|c||c|c|c|c|c|c|}
\hline
\ Methods & SV-TV \cite{jia2019color} & IDD-BM3D \cite{danielyan2011bm3d} & IRCNN \cite{zhang2017learning} & MSWNNM \cite{yair2018multi} & QWNNM* \cite{huang2022quaternion} & QWSNM (Ours) \\
\hline
\multicolumn{7}{|c|}{Set27}\\
\hline
Img1 & 20.89/0.8342 & 20.66/0.8272 & \underline{22.58/0.8764} & 22.56/0.8752 & 22.55/0.8728 & \textbf{22.92/0.8810} \\
Img2 & 24.96/0.8331 & 24.20/0.8279 & 26.13/0.8668 & 26.13/0.8660 & \underline{26.19/0.8695} & \textbf{26.35/0.8715} \\
Img3 & 22.05/0.7712 & 20.01/0.7163 & 24.05/0.8387 & \underline{24.14/0.8399} & 24.06/0.8359 & \textbf{24.24/0.8412} \\
Img4 & 26.47/0.9591 & 26.03/0.9561 & 27.77/0.9684 & 27.80/0.9687 & \underline{28.22/0.9712} & \textbf{28.28/0.9714} \\
Img5 & 29.47/0.9695 & 29.13/0.9691 & 31.08/0.9788 & 31.11/0.9787 & \underline{31.68/0.9816} & \textbf{31.83/0.9820} \\
Img6 & 22.59/0.8520 & 22.29/0.8458 & 23.62/0.8755 & 23.57/0.8741 & \underline{23.88/0.8815} & \textbf{23.99/0.8833} \\
Img7 & 22.50/0.8231 & 23.05/0.8408 & 24.23/0.8703 & 24.16/0.8690 & \underline{24.56/0.8799} & \textbf{24.71/0.8834} \\
Img8 & 20.40/0.7541 & 20.80/0.8098 & 21.79/0.8535 & 21.74/0.7293 & \underline{22.17/0.9099} & \textbf{22.48/0.9352} \\
Img9 & 28.14/0.9202 & 27.93/0.9251 & \textbf{28.61/0.9360} & 28.55/0.9350 & 28.58/0.9349 & \underline{28.59/0.9351} \\
Img10 & 19.48/0.8410 & 19.23/0.8370 & 20.29/0.8647 & 20.24/0.8632 & \underline{20.58/0.8728} & \textbf{20.65/0.8748} \\
Img11 & 20.30/0.8028 & 19.92/0.7916 & 21.33/0.8362 & 21.27/0.8330 & \underline{21.56/0.8435} & \textbf{21.76/0.8470} \\
Img12 & 21.62/0.5859 & 20.84/0.5268 & \textbf{21.83/0.6056} & 21.75/0.5968 & 21.72/0.5975 & \underline{21.82/0.5984} \\
Img13 & 23.27/0.7840 & 22.56/0.7680 & 24.35/0.8222 & 24.27/0.8160 & \underline{24.42/0.8224} & \textbf{24.59/0.8260} \\
Img14 & 22.04/0.7166 & 21.15/0.6841 & 23.49/0.7790 & 23.49/0.7782 & \underline{23.79/0.7840} & \textbf{24.13/0.7979} \\
Img15 & 27.07/0.8937 & 26.36/0.8883 & 27.98/\underline{0.9097} & 27.93/0.9092 & \underline{28.04}/0.9070 & \textbf{28.19/0.9160} \\
Img16 & 25.57/0.7891 & 25.71/0.8053 & 27.01/0.8327 & 27.01/0.8300 & \underline{27.19/0.8436} & \textbf{27.66/0.8533} \\
Img17 & 22.81/0.8187 & 22.58/0.8149 & 23.92/\underline{0.8478} & 23.88/0.8449 & \underline{24.15}/0.8467 & \textbf{24.32/0.8515} \\
Img18 & 19.29/0.8263 & 18.87/0.8119 & 19.63/0.8408 & 19.56/0.8375 & \underline{19.90/0.8507} & \textbf{19.96/0.8521} \\
Img19 & 20.99/0.8738 & 19.19/0.8249 & 22.24/\underline{0.9006} & \textbf{22.33/0.9021} & 22.06/0.8943 & \underline{22.27}/0.8995 \\
Img20 & 24.82/0.8786 & 25.09/0.9105 & 26.94/0.9377 & \underline{26.94/0.9378} & 26.75/0.9357 & \textbf{27.03/0.9394} \\
Img21 & 31.52/0.9483 & 31.60/0.9591 & \underline{33.34/0.9670} & \textbf{33.35/0.9672} & 32.67/0.9611 & 32.75/0.9613 \\
Img22 & 32.56/0.9247 & 30.88/0.9362 & 33.38/\textbf{0.9594} & 33.47/\underline{0.9580} & \underline{33.75}/0.9549 & \textbf{33.81}/0.9556 \\
Img23 & 23.21/0.8145 & 22.88/0.8051 & 23.72/0.8366 & 23.64/0.8314 & \underline{24.26/0.8542} & \textbf{24.32/0.8549} \\
Img24 & 28.70/0.9483 & 28.57/0.9526 & \textbf{29.31/0.9604} & \underline{29.24/0.9601} & 29.16/0.9578 & 29.17/0.9582 \\
Img25 & 25.84/0.7787 & 25.63/0.7715 & \textbf{26.16/0.7884} & 26.07/0.7820 & 26.05/0.7831 & \underline{26.08/0.7835} \\
Img26 & 17.26/0.9082 & 18.42/0.9311 & 19.92/0.9496 & 19.68/0.9487 & \underline{22.02/0.9705} & \textbf{22.08/0.9708} \\
Img27 & 28.35/0.8276 & 27.05/0.8296 & \textbf{30.54/0.8824} & \underline{30.50/0.8810} & 30.02/0.8649 & 30.32/0.8700 \\
Average & 24.15/0.8399 & 23.73/0.8358  & 25.38/0.8735 & 25.35/0.8685 & \underline{25.56/0.8753} & \textbf{25.72/0.8813} \\
\hline
\multicolumn{7}{|c|}{Set12}\\
\hline
Bird & 22.77/0.6867 & 22.37/0.6725 & 23.34/0.7266 & 23.25/0.7186 & \underline{23.59/0.7417} & \textbf{23.62/0.7425} \\
Plane & 23.60/0.7149 & 23.16/0.7555 & 24.67/0.8134 & 24.66/0.8123 & \underline{24.75/0.8196} & \textbf{24.83/0.8212} \\
Baboon & 21.70/0.6936 & 21.74/0.6730 & 22.16/0.7003 & 22.08/0.6912 & \underline{22.26/0.7018} & \textbf{22.30/0.7024} \\
Bee & 25.48/0.8554 & 26.59/0.8868 & 27.97/0.9064 & 28.02/0.9066 & \underline{28.63/0.9130} & \textbf{28.75/0.9140} \\
Aquatic & 22.77/0.6857 & 22.25/0.6919 & 23.91/0.7730 & 23.87/0.7732 & \underline{24.16/0.7792} & \textbf{24.25/0.7835} \\
Barbara & 22.59/0.7728 & 22.44/0.7701 & 23.44/0.8113 & 23.40/0.8080 & \underline{23.64/0.8168} & \textbf{23.74/0.8199}\\
Boat & 22.02/0.7383 & 21.11/0.7355 & 23.00/0.8068 & 22.93/0.8032 & \underline{22.95/0.8043} & \textbf{23.03/0.8072} \\
House & 25.04/0.8795 & 25.02/0.8845 & 26.33/0.9132 & 26.46/0.9145 & \underline{27.82/0.9275} & \textbf{28.00/0.9300} \\
Peppers & 23.15/0.9174 & 23.31/0.9257 & \underline{24.56/0.9410} & 24.54/0.9405 & 24.42/0.9381 & \textbf{24.61/0.9412} \\
Starfish & 22.55/0.9069 & 22.15/0.9133 & 23.51/0.9339 & 23.47/0.9337 & \underline{23.87/0.9368} & \textbf{23.95/0.9381} \\
Lena & 27.55/0.9564 & 27.05/0.9571 & \underline{28.66/0.9674} & 28.65/0.9670 & 28.63/0.9672 & \textbf{28.73/0.9678} \\
Pelican & 24.36/0.8167 & 23.16/0.8000 & 24.88/0.8438 & 24.79/0.8420 & \underline{25.03/0.8399} & \textbf{25.07/0.8405} \\
Average & 23.63/0.8020 & 23.36/0.8055 & 24.71/0.8448 & 24.68/0.8426 & \underline{24.98/0.8489} & \textbf{25.07/0.8506} \\
\hline
\end{tabular}}
}
\caption{PSNR (dB) and SSIM values of different restoration models for MB(20,60)/$\sigma=15$. The best results are highlighted in \textbf{bold} and the second-best results are \underline{underlined}.}
\end{center}
\end{table*}

\begin{table*}
\begin{center}{\tiny
\resizebox{\textwidth}{60mm}{
\begin{tabular}{|c||c|c|c|c|c|c|c|c|c|c|}
\hline
\ $p$ value & $p=0.15$ & $p=0.25$ & $p=0.35$ & $p=0.45$ & $p=0.55$ & $p=0.65$ & $p=0.75$ & $p=0.85$ & $p=0.95$ & $p=1$ \\
\hline
\multicolumn{11}{|c|}{MB(20,60)/$\sigma=15$}\\
\hline
Bird & 19.28/0.4117 & 19.28/0.4117 & 19.28/0.4117 & 19.28/0.4117 & 19.28/0.4117 & 19.33/0.4144 & 19.51/0.4272 & 20.26/0.4818 & \textbf{23.62/0.7425} & 23.41/0.7270 \\
Plane & 19.36/0.3255 & 19.36/0.3255 & 19.36/0.3255 & 19.36/0.3255 & 19.36/0.3255 & 19.40/0.3278 & 19.59/0.3389 & 20.37/0.3913 & \textbf{24.83/0.8212} & 24.63/0.8125\\
Baboon & 19.06/0.5565 & 19.06/0.556 & 19.06/0.556 & 19.06/0.556 & 19.06/0.556 & 19.09/0.5582 & 19.26/0.5663 & 19.92/05977 & \textbf{22.30/0.7018} & 22.21/0.6928 \\
Bee & 20.39/0.6448 & 20.39/0.6448 & 20.39/0.6448 & 20.39/0.6448 & 20.39/0.6448 & 20.44/0.6469 & 20.70/0.6567 & 21.75/0.6985 & \textbf{28.75/0.9140} & 28.47/0.9117 \\
Aquatic & 18.59/0.3513 & 18.59/0.3513 & 18.59/0.3513 & 18.59/0.3513 & 18.59/0.3513 & 18.62/0.3535 & 18.78/0.3639 & 19.39/0.4101 & \textbf{24.25/0.7835} & 23.86/0.7690 \\
Barbara & 19.16/0.5531 & 19.16/0.5531 & 19.16/0.5531 & 19.16/0.5531 & 19.16/0.5531 & 19.20/0.5555 & 19.38/0.5664 & 20.07/0.6108 & \textbf{23.74/0.8199} & 23.65/0.8164\\
Boat & 18.32/0.4328 & 18.32/0.4238 & 18.32/0.4238 & 18.32/0.4238 & 18.32/0.4238 & 18.35/0.4262 & 18.50/0.4375 & 19.07/0.4874 & \textbf{23.03/0.8072} & 22.50/0.7895 \\
House & 19.65/0.6424 & 19.65/0.6424 & 19.65/0.6424 & 19.65/0.6424 & 19.65/0.6424 & 19.69/0.6452 & 19.91/0.6584 & 20.77/0.7109 & \textbf{28.00/0.9300} & 27.19/0.9215 \\
Peppers & 19.24/0.8107 & 19.24/0.8107 & 19.24/0.8107 & 19.24/0.8107 & 19.24/0.8107 & 19.28/0.8123 & 19.46/0.8194 & 20.17/0.8463 & \textbf{24.61/0.9412} & 24.56/0.9415 \\
Starfish & 19.81/0.7516 & 19.81/0.7516 & 19.81/0.7516 & 19.81/0.7516 & 19.81/0.7516 & 18.94/0.7536 & 19.11/0.7622 & 19.77/0.7965 & \textbf{23.95/0.9381} & 23.71/0.9363 \\
Lena & 20.94/0.8271 & 20.94/0.8271 & 20.94/0.8271 & 20.94/0.8271 & 20.94/0.8271 & 21.00/0.8292 & 21.29/0.8386 & 22.52/0.8745 & \textbf{28.73/0.9678} & 28.63/0.9671 \\
Pelican & 19.56/0.4135 & 19.56/0.4135 & 19.56/0.4135 & 19.56/0.4135 & 19.56/0.4135 & 19.60/0.4161 & 19.80/0.4256 & 20.65/0.4839 & \textbf{25.07/0.8405} & 24.75/0.8332 \\
\hline
\multicolumn{11}{|c|}{AB(9,9)/$\sigma=15$}\\
\hline
Bird & 20.33/0.4712 & 20.33/0.4712 & 20.33/0.4712 & 20.33/0.4712 & 20.34/0.4717 & 20.39/0.4756 & 20.65/0.4946 & 21.74/0.5868 & \textbf{23.55/0.7361} & 23.41/0.7247 \\
Plane & 20.69/0.3765 & 20.69/0.3765 & 20.69/0.3765 & 20.69/0.3765 & 20.70/0.3770 & 20.76/0.3808 & 21.04/0.3997 & 22.36/0.5058 & \textbf{25.22/0.8295} & 25.07/0.8224\\
Baboon & 19.90/0.5806 & 19.90/0.5806 & 19.90/0.5806 & 19.90/0.5806 & 19.91/0.5809 & 19.95/0.5832 & 20.17/0.5936 & 21.06/0.6448 & \textbf{22.46/0.7023} & 22.36/0.6927 \\
Bee & 22.11/0.6941 & 22.11/0.6941 & 22.11/0.6941 & 22.11/0.6941  & 22.11/0.6941 & 22.19/0.6973 & 22.59/0.7116 & 24.24/0.7698  & \textbf{29.20/0.9186} & \textbf{29.20}/0.9126 \\
Aquatic & 19.97/0.4135 & 19.97/0.4135 & 19.97/0.4135 & 19.97/0.4135 & 19.99/0.4142 & 20.04/0.4178 & 20.27/0.4350 & 21.45/0.5222 & \textbf{24.28/0.7874} & 24.15/0.7849 \\
Barbara & 20.08/0.5912 & 20.08/0.5912 & 20.08/0.5912 & 20.08/0.5912 & 20.09/0.5914 & 20.14/0.5945 & 20.37/0.6092 & 21.40/0.6745 & \textbf{23.79/0.8087} & 23.53/0.8053\\
Boat & 19.55/0.4805 & 19.55/0.4805 & 19.55/0.4805 & 19.55/0.4805 & 19.56/0.4817 & 19.61/0.4852 & 19.82/0.5027 & 20.94/0.5942 & \textbf{23.17/0.8042} & 22.92/0.7976 \\
House & 21.51/0.7215 & 21.51/0.7215 & 21.51/0.7215 & 21.51/0.7215 & 21.52/0.7230 & 21.60/0.7267 & 21.95/0.7437 & 23.58/0.8137 & \textbf{28.95/0.9395} & 28.74/0.9369 \\
Peppers & 20.48/0.8458 & 20.48/0.8458 & 20.48/0.8458 & 20.48/0.8458 & 20.49/0.8467 & 20.55/0.8486 & 20.80/0.8569 & 21.96/0.8893 & \textbf{24.80/0.9440} & 24.72/0.9433 \\
Starfish & 20.35/0.7915 & 20.35/0.7915 & 20.35/0.7915 & 20.35/0.7915 & 20.36/0.7930 & 20.41/0.7954 & 21.66/0.8066 & 22.76/0.8485 & \textbf{24.15/0.9367} & 24.06/0.9373 \\
Lena & 22.19/0.9601 & 22.19/0.8601 & 22.19/0.860 & 22.19/0.860 & 22.20/0.8604 & 22.29/0.8628 & 22.70/0.8738 & 24.30/0.9104 & \textbf{28.60/0.9654} & 28.49/0.9646 \\
Pelican & 20.85/0.4638 & 20.85/0.4638 & 20.85/0.4638 & 20.85/0.4638 & 20.85/0.4638 & 20.91/0.4683 & 21.21/0.4874 & 22.44/0.5740 & \textbf{25.04/0.8362} & 24.95/0.8215 \\
\hline
\multicolumn{11}{|c|}{GB(25,1.6)/$\sigma=15$}\\
\hline
Bird & 21.14/0.5421 & 21.14/0.5421 & 21.14/0.5421 & 21.14/0.5421 & 21.14/0.5421 & 21.21/0.5460 & 21.88/0.4406 & 23.05/0.6712 & \textbf{25.27/0.8130} & 24.98/0.7967 \\
Plane & 21.62/0.4285 & 21.62/0.4285 & 21.62/0.4285 & 21.62/0.4285 & 21.62/0.4285 & 21.69/0.4320 & 22.04/0.4494 & 23.89/0.5663 & \textbf{27.30/0.8895} & 26.86/0.8692\\
Baboon & 20.54/0.6377 & 20.54/0.6377 & 20.54/0.6377 & 20.54/0.6377 & 20.54/0.6377 & 20.59/0.6400 & 20.84/0.6508 & 22.02/0.7096 & \textbf{23.48/0.7572} & 23.23/0.7376 \\
Bee & 22.45/0.7063 & 22.45/0.7063 & 22.45/0.7063 & 22.45/0.7063  & 22.45/0.7063 & 22.54/0.7090 & 22.97/0.7224 & 25.26/0.7944  & \textbf{31.27/0.9386} & 31.17/0.9375 \\
Aquatic & 21.34/0.4930 & 21.34/0.4930 & 21.34/0.4930 & 21.34/0.4930 & 21.34/0.4930 & 21.40/0.4965 & 21.71/0.5139 & 23.36/0.6185 & \textbf{26.58/0.8702} & 26.32/0.8597 \\
Barbara & 20.92/0.6405 & 20.92/0.6405 & 20.92/0.6405 & 20.92/0.6405 & 20.92/0.6405 & 20.97/0.6435 & 21.25/0.6581 & 22.61/0.7336 & \textbf{25.00/0.8585} & 24.81/0.8476\\
Boat & 20.60/0.5385 & 20.60/0.5385 & 20.60/0.5385 & 20.60/0.5385 & 20.60/0.5385 & 20.66/0.5419 & 20.92/0.5584 & 22.32/0.6579 & \textbf{24.62/0.8524} & 24.20/0.8401 \\
House & 22.01/0.7448 & 22.01/0.7448 & 22.01/0.7448 & 22.01/0.7448 & 22.01/0.7448 & 22.08/0.7483 & 22.47/0.7649 & 24.54/0.8443 & \textbf{29.90/0.9500} & 29.61/0.9475 \\
Peppers & 21.49/0.8720 & 21.49/0.8720 & 21.49/0.8720 & 21.49/0.8720 & 21.49/0.8720 & 21.55/0.8738 & 21.88/0.8820 & 23.51/0.9179 & \textbf{26.63/0.9643} & 26.40/0.9604 \\
Starfish & 21.49/0.8123 & 21.49/0.8123 & 21.49/0.8123 & 21.49/0.8123 & 21.49/0.8123 & 21.55/0.8144 & 21.87/0.8250 & 23.53/0.8763 & \textbf{26.77/0.9654} & 26.46/0.9581 \\
Lena & 22.46/0.8675 & 22.46/0.8675 & 22.46/0.8675 & 22.46/0.8675 & 22.46/0.8675 & 22.52/0.8683 & 22.95/0.8789 & 25.21/0.9247 & \textbf{30.55/0.9800} & 30.27/0.9750 \\
Pelican & 21.29/0.4810 & 21.29/0.4810 & 21.29/0.4810 & 21.29/0.4810 & 21.29/0.4810 & 21.35/0.4820 & 21.68/0.4994 & 23.36/0.6044 & \textbf{26.26/0.8701} & 26.05/0.8615 \\
\hline
\end{tabular}}
}
\caption{PSNR (dB) and SSIM values of different $p$ values for color image deblurring on the Set12 dataset. The best results are marked in \textbf{bold}.}
\end{center}
\end{table*}

\begin{figure}
\centering
\subfigure[$\sigma$=15]{\includegraphics[width=0.495\textwidth]{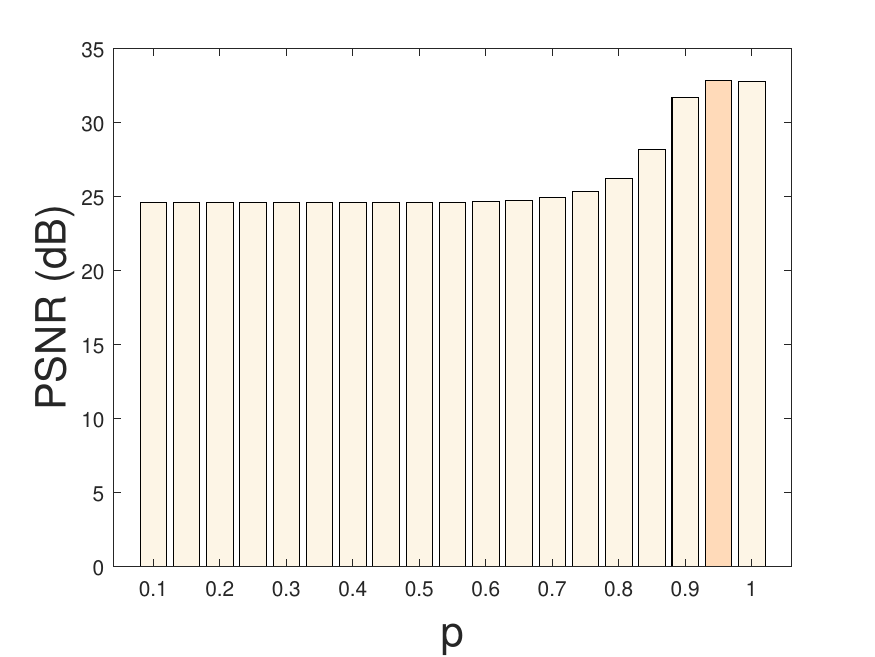}}
\subfigure[$\sigma$=25]{\includegraphics[width=0.495\textwidth]{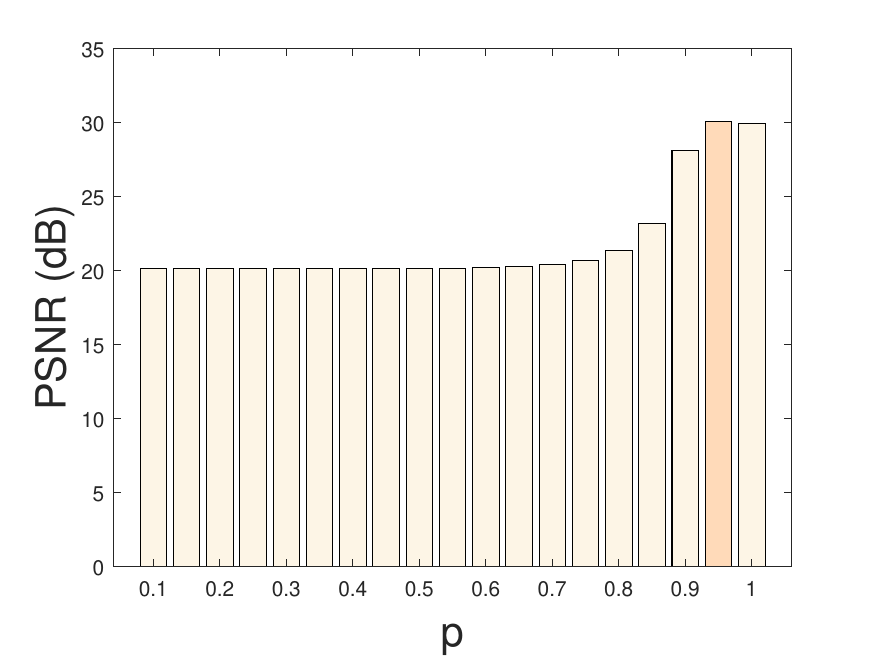}}
\subfigure[$\sigma$=35]{\includegraphics[width=0.495\textwidth]{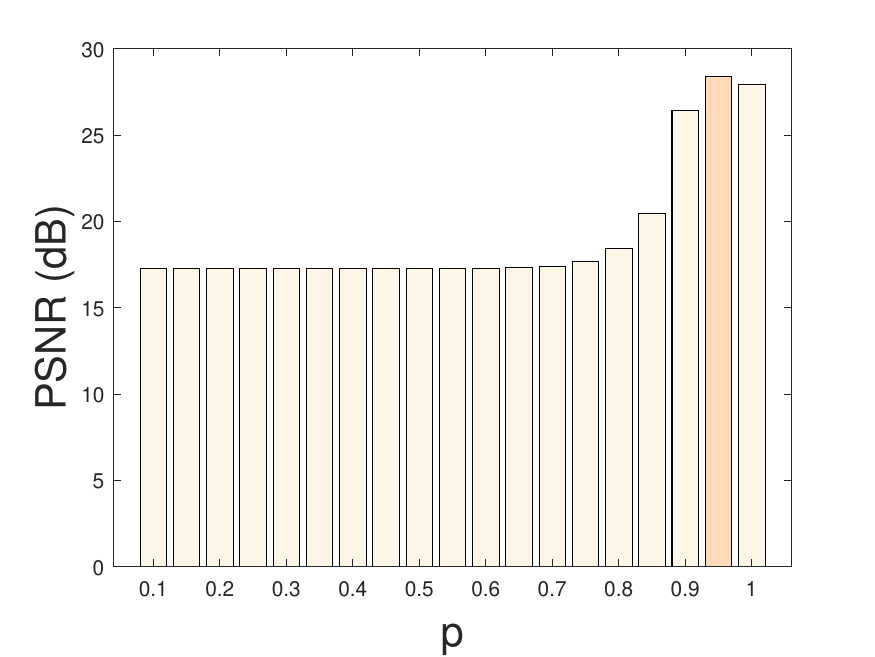}}
\subfigure[$\sigma$=45]{\includegraphics[width=0.495\textwidth]{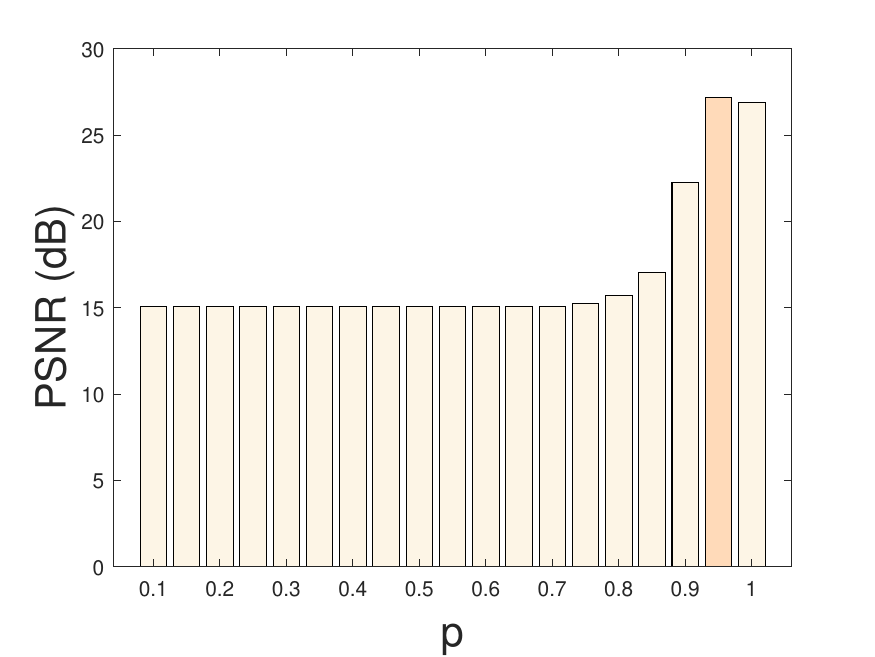}}
\caption{The influence of changing $p$ values on denoised results under different noise levels $\sigma$. Test image: ``Bird''.}
\end{figure}

\begin{figure}
\centering
\subfigure[]{\includegraphics[width=0.495\textwidth]{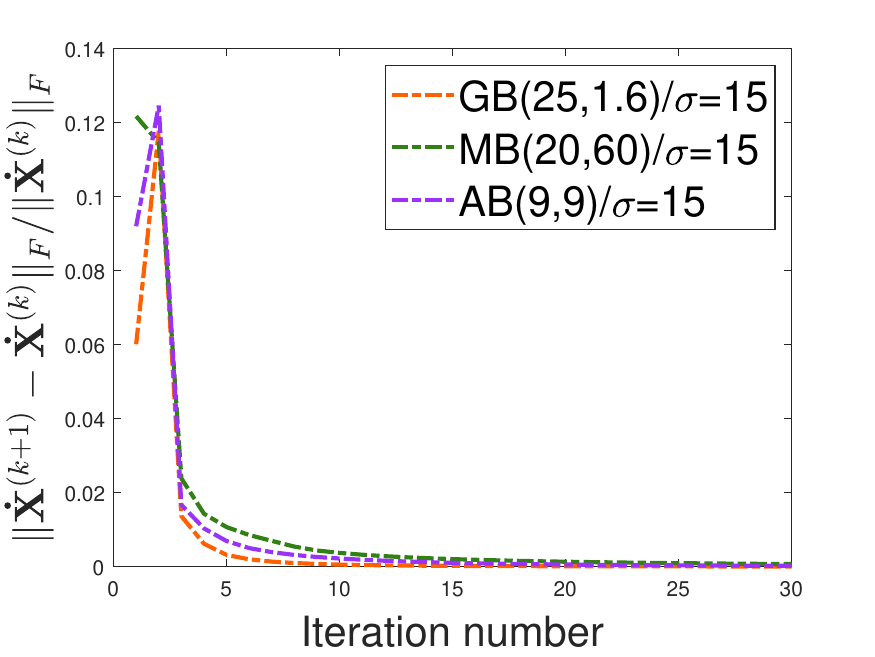}}
\subfigure[]{\includegraphics[width=0.495\textwidth]{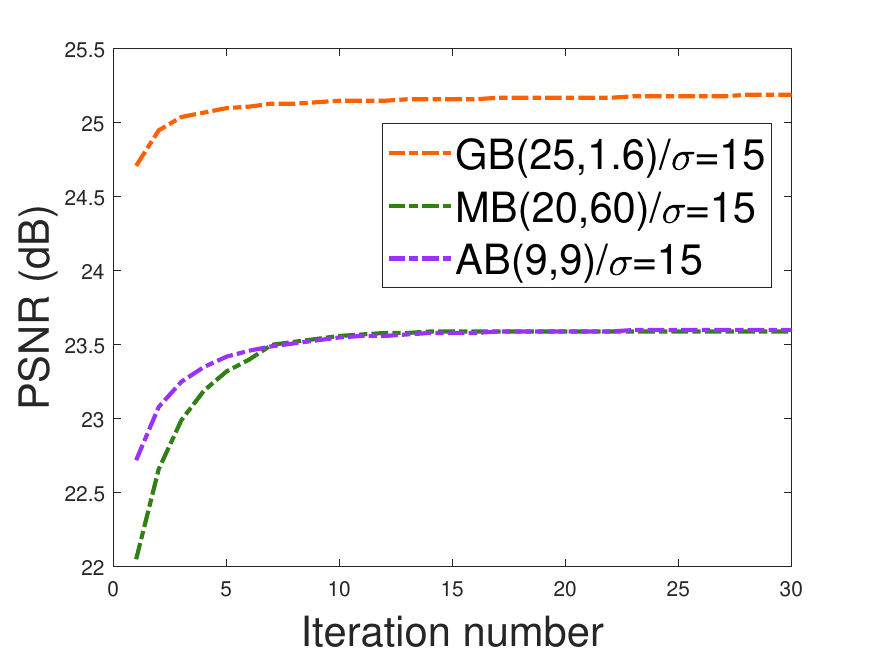}}
\caption{The evolution curves of (a) relative error and (b) PSNR value between consecutive iterations in the case of deblurring. Test image: ``Bird''.}
\end{figure}

\subsection{Analysis of Power $p$}
In this part, the effects of power $p$ in the weighted Schatten $p$-norm are studied.
Table 6 reports the PSNR and SSIM values for deblurring with the power $p$ ranging from 0.15 to 1 in the Set12 dataset. From the Table 6, it can be observed that the recovery performance of our method is relatively sensitive to the parameter $p$. When $p$ is 0.95, the PSNR and SSIM of QWSNM on all test images consistently achieve the highest values. Therefore, $p$ value of 0.95 is a reasonable choice for color image deblurring. To illustrate the influence of $p$ values under different noise levels for color image denoising, we just utilize the test image ``Bird'' as an example and the results are presented in Fig. 9. In each subfigure, the vertical coordinate represents the PSNR value under a certain noise level, while the horizontal coordinate denotes the values of $p$ ranging from 0.1 to 1 with interval 0.05. It can be seen that when $p$ is 0.95, our algorithm consistently achieve the highest PSNR values under different noise levels.
Therefore, we also select $p=0.95$ for color image denoising.

\subsection{Convergence Study}
Fig. 10 illustrates the empirical convergence of our proposed algorithm in the cases of three blur kernels on the test image ``Bird''. One can clearly observe that with the growth of iteration number, all the relative error and PSNR curves decrease and increase monotonically and ultimately become flat and stable, numerically exhibiting the reliable convergence behavior of proposed QWSNM algorithm.
It is noted that other scenarios also have the similar conclusions.

\section{Conclusion}
In this paper, we proposed a quaternion-based QWSNM model for color image restoration, which combines the advantages of the WSNM regularizer and quaternion representation.
To be concrete, the WSNM regularizer can better adaptively shrink the rank components over the WNNM counterpart, and the quaternion representation of color images is capable of fully preserving the correlation of three color channels.
In particular, we modified the QADMM algorithm through a continuation strategy to iteratively solve the resulting optimization problem. Moreover,  rigorous convergence analysis was provided to verify the sound properties of our proposed algorithm. Experimental results on color image denoising under different noise levels and color image deblurring under three types of blur kernels demonstrated the effectiveness and robustness of QWSNM.
Compared with the related competing methods, the visual results showed that our proposed QWSNM is capable of better preserving the color structure and avoiding color distortion.

However, despite the good recovery performance, our proposed algorithm has the efficiency limitation. In fact, the QSVD in the optimization process is quite time-consuming,  so how to improve the algorithmic efficiency remains a significant issue. Very recently, Gai et al. \cite{gai2023theory} introduced the theory of reduced biquaternion sparse representation, which benefited from the commutative of the reduced biquaternion algebra.
In the future research, we would pay more attention to the reduced biquaternion representation and extend existing algorithms to the reduced biquaternion domain for color image restoration. In parallel, we plan to extend the proposed QWNNM algorithm to other relevant tasks, such as color image inpainting, color image super-resolution and compressive sensing, etc.

\section*{Acknowledgements}
The research was supported by the NSFC under Grants 12001005, 12271083.


\section*{Conflict of Interest}
The authors declare that they have no known competing financial interests or personal relationships that could have appeared to influence the work reported in this paper.

\section*{Code Availability}
Code generated or used during the study are available from the corresponding author on
reasonable request.

\bibliography{qwsnm}
\bibliographystyle{spmpsci}

\end{document}